%% file: paper.tex
\documentclass[journal, transmag]{IEEEtran}
\input{notation}
\input{packages}

\newtheorem{definition}{Definition}
\begin{document}
\title{Riemannian Optimization for Distance-Geometric Inverse Kinematics}
\author{Filip Mari\'c$^{a,b,\dagger }$, Matthew Giamou$^{a,\dagger}$, Adam W. Hall$^{a,c}$,\\ Soroush Khoubyarian$^a$, Ivan Petrovi\'c$^{a,b}$, and Jonathan Kelly$^{a}$
\thanks{$^\dagger$Denotes equal contribution.}
\thanks{$^a$Filip Mari\'c, Matthew Giamou, Adam W. Hall, Soroush Khoubyarian, Ivan Petrovi\'c and Jonathan Kelly are with the Space and Terrestrial Autonomous Robotic Systems Laboratory, University of Toronto, Institute for Aerospace Studies, Toronto, Canada. \{\texttt{<first name>.<last name>@robotics.utias.utoronto.ca}\}}
\thanks{$^b$Filip Mari\'c and Ivan Petrovi\'c are with the Laboratory for Autonomous Systems and Mobile Robotics, University of Zagreb, Faculty of Electrical Engineering and Computing, Zagreb, Croatia. \{\texttt{<first name>.<last name>@fer.hr}\}}
\thanks{$^c$Adam W. Hall is jointly with the Dynamic Systems Laboratory, University of Toronto, Institute for Aerospace Studies, Toronto, Canada.\}}
}

\newenvironment{rcases}
  {\left.\begin{aligned}}
  {\end{aligned}\right\rbrace}
\newcommand{\mpg}[1]{\textcolor{green}{#1}}
\newcommand{\fm}[1]{\textcolor{blue}{#1}}
\newcommand{\jk}[1]{\textcolor{red}{\textbf{{JK:} #1}}}

\markboth{IEEE Transactions on Robotics}%
{Shell \MakeLowercase{\textit{et al.}}: Bare Demo of IEEEtran.cls for IEEE Journals}

\maketitle

\begin{abstract}
  Solving the inverse kinematics problem is a fundamental challenge in motion planning, control, and calibration for articulated robots.
  Kinematic models for these robots are typically parametrized by joint angles, generating a complicated mapping between the robot configuration and the end-effector pose.
  Alternatively, the kinematic model and task constraints can be represented using invariant distances between points attached to the robot.
  In this paper, we formalize the equivalence of distance-based inverse kinematics and the distance geometry problem for a large class of articulated robots and task constraints.
  Unlike previous approaches, we use the connection between distance geometry and low-rank matrix completion to find inverse kinematics solutions by completing a partial Euclidean distance matrix through local optimization.
  Furthermore, we parametrize the space of Euclidean distance matrices with the Riemannian manifold of fixed-rank Gram matrices, allowing us to leverage a variety of mature Riemannian optimization methods.
  Finally, we show that bound smoothing can be used to generate informed initializations without significant computational overhead, improving convergence.
  We demonstrate that our inverse kinematics solver achieves higher success rates than traditional techniques, and substantially outperforms them on problems that involve many workspace constraints.
\end{abstract}


\IEEEpeerreviewmaketitle

\input{sections/introduction.tex}

\input{sections/related_work.tex}

\input{sections/background.tex}

\input{sections/dg_ik.tex}

\input{sections/algorithm.tex}

\input{sections/experiments.tex}

\input{sections/conclusion.tex}

\input{sections/future_work.tex}


%


\section*{Acknowledgment}
This work was supported in part by the European Regional Development Fund under the grant KK.01.1.1.01.0009 (DATACROSS) and by the Natural Sciences and Engineering Research Council of Canada (NSERC).
Jonathan Kelly gratefully acknowledges support from the Canada Research Chairs program.
Matthew Giamou is a Vector Institute Postgraduate Affiliate and RBC Fellow.
Jonathan Kelly is a Vector Institute Faculty Affiliate.

\ifCLASSOPTIONcaptionsoff
	\newpage
\fi



\bibliographystyle{IEEEtran}
\bibliography{robotics_abbrv, references}
%

%

\begin{IEEEbiography}	[{\includegraphics[width=1in,height=1.25in,clip,keepaspectratio]{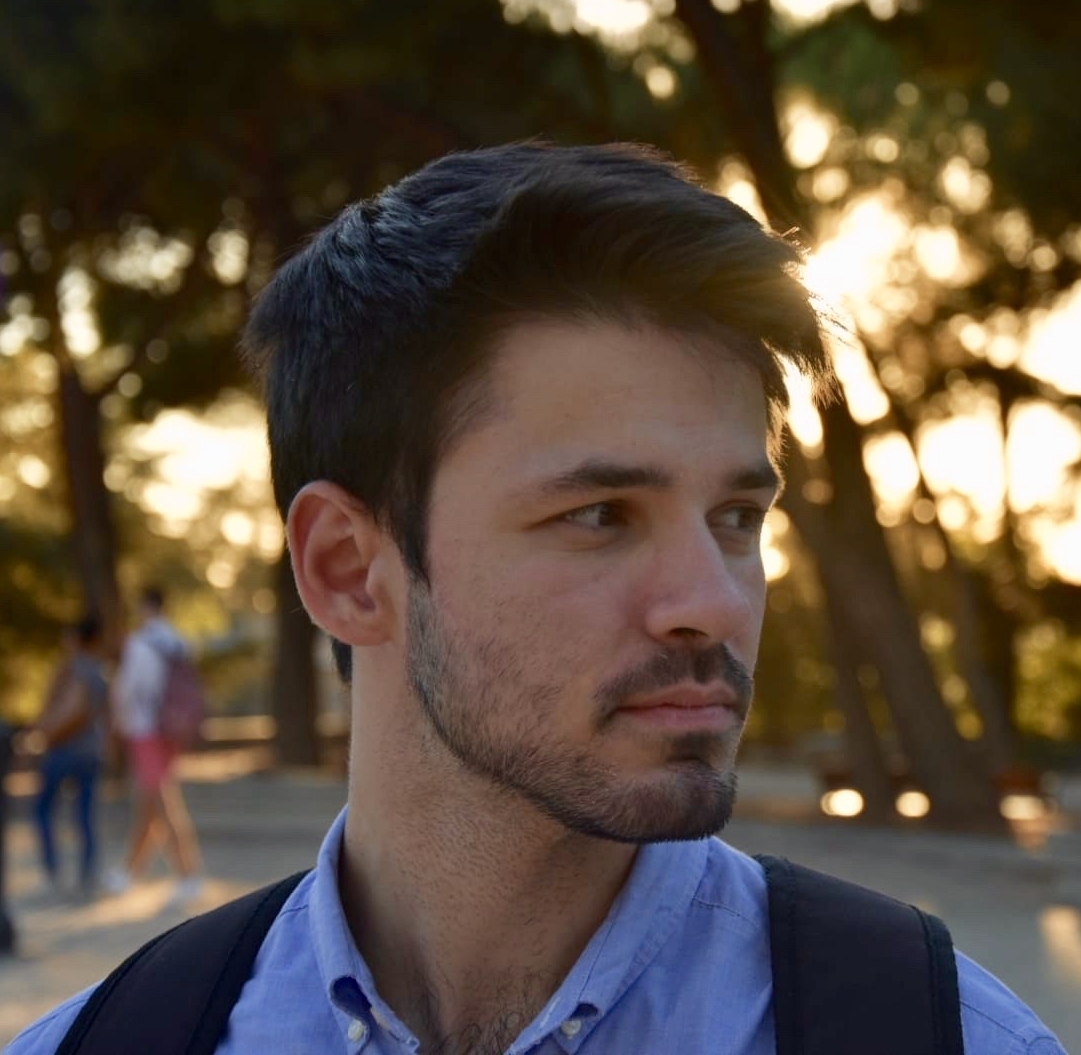}}]{Filip Mari\'c}
received his B.Sc. and M.Sc. Degrees in electrical engineering and information technology from the University of Zagreb, Faculty of Electrical Engineering and Computing in 2015 and 2017, respectively. He is currently pursuing a Ph.D. degree at the Space and Terrestrial Autonomous Robotic Systems (STARS) laboratory at the University of Toronto, jointly with the Laboratory for Autonomous Systems and Mobile Robotics (LAMOR) at the University of Zagreb. His research is focused on applying concepts from differential geometry to kinematics and motion planning in robotics.\end{IEEEbiography}

\begin{IEEEbiography}[{\includegraphics[width=1in,height=1.25in,clip,keepaspectratio]{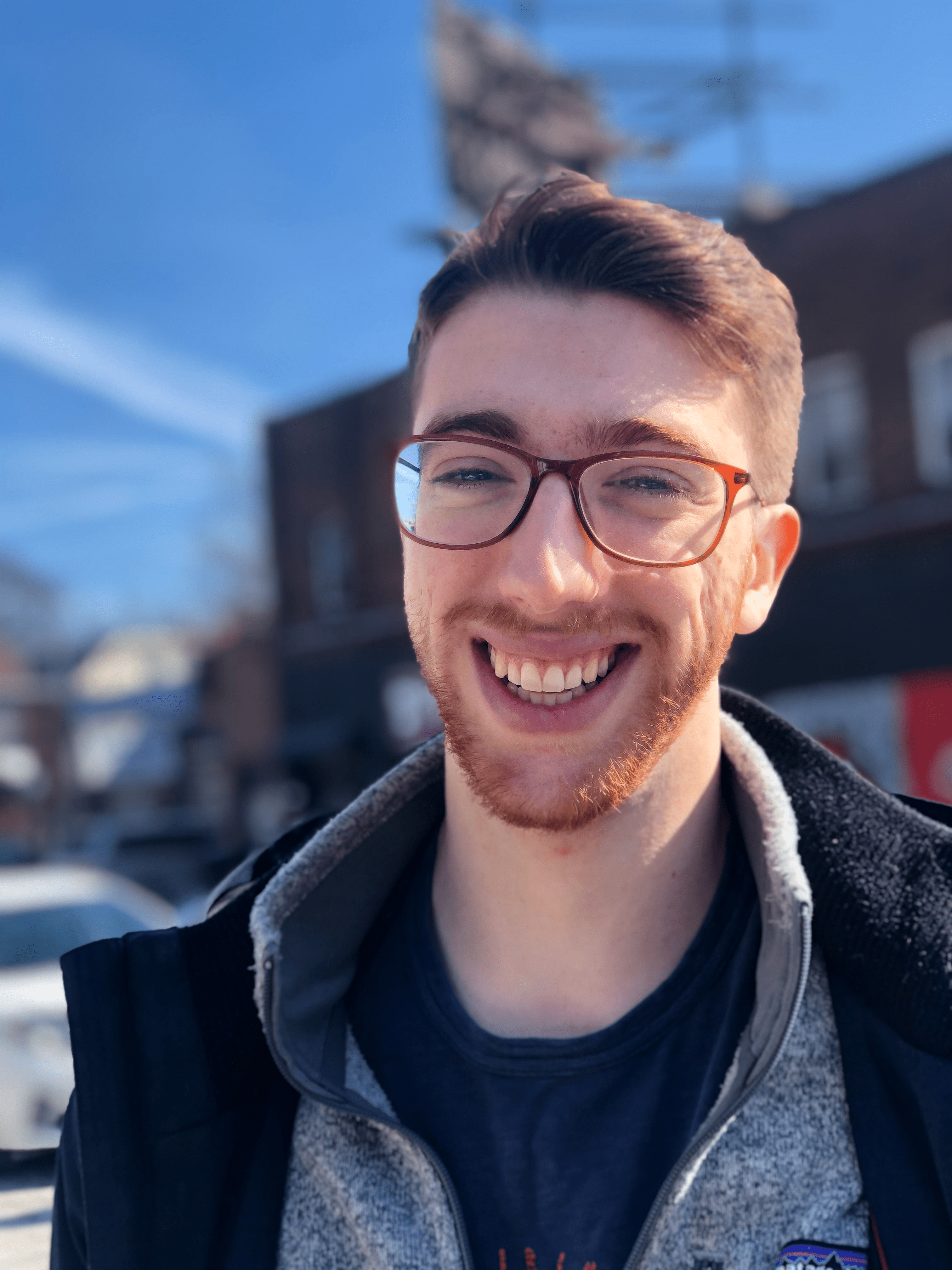}}]{Matthew Giamou}
	received his B.A.Sc. from the University of Toronto in 2015 and his M.Sc. from MIT In 2017. 
	He is currently a Ph.D. candidate at the Space and Terrestrial Autonomous Robotic Systems (STARS) laboratory at the University of Toronto, where he is investigating convex relaxations for various challenging state estimation and planning problems.
	His research interests also include multi-robot systems and the integration of learned and classical models for robot perception.	
\end{IEEEbiography}

\begin{IEEEbiography}[{\includegraphics[width=1in,height=1.25in,clip,keepaspectratio]{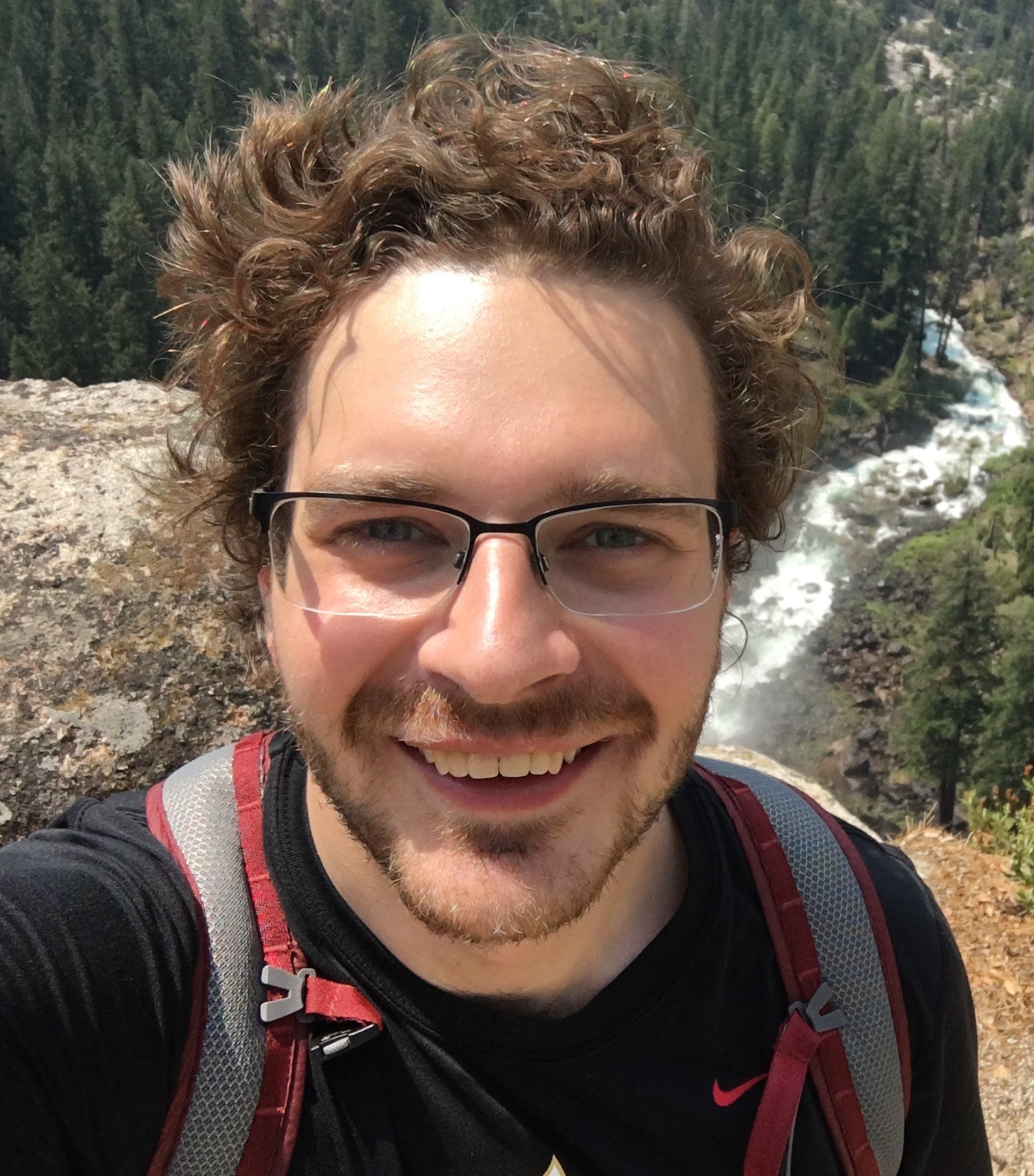}}]{Adam W. Hall}
	received a B.Sc.Eng. in engineering physics and a B.Sc. in chemistry from Queen's University in 2014. 
	Upon graduating, he worked in the space robotics industry before starting his Ph.D. at the University of Toronto.
	Currently, as a member of both the Space and Terrestrial Autonomous Robotic Systems (STARS) laboratory and the Dynamic Systems Laboratory (DSL), Adam is exploring how to apply classical notions of safety to learning-based control problems.
\end{IEEEbiography}

\begin{IEEEbiography}[{\includegraphics[width=1in,height=1.25in,clip,keepaspectratio]{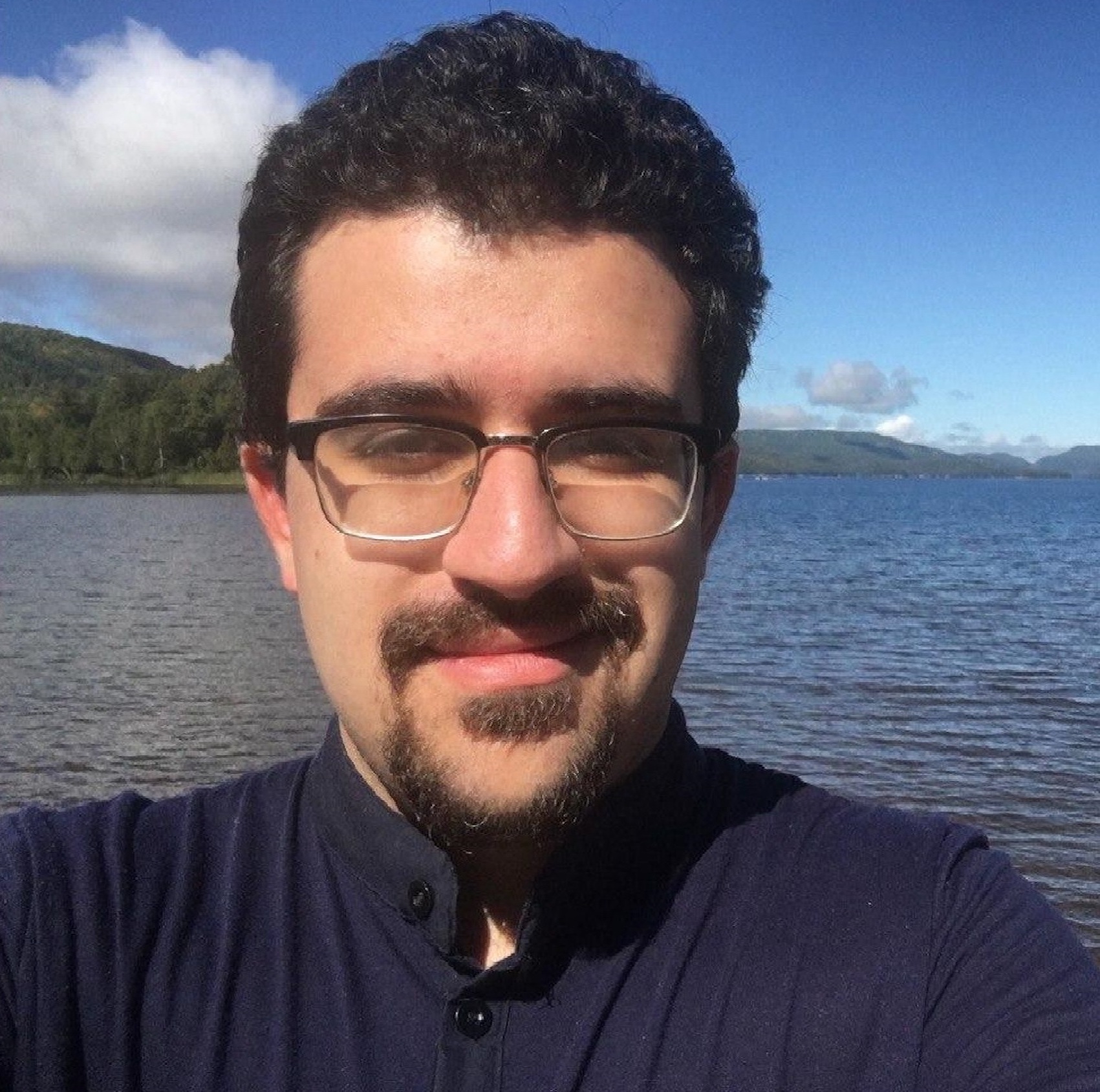}}]{Soroush Khoubyarian}
	is a third year Engineering Science student at the University of Toronto. He was a student at the Space and Terrestrial Autonomous Robotic Systems (STARS) laboratory at the University of Toronto Institute for Aerospace Studies in the summers of 2019 and 2020, where he worked on convex optimization methods for sensor calibration and manipulation. He specializes in Engineering Physics and is expected to receive his B.A.Sc in 2023.
\end{IEEEbiography}

\begin{IEEEbiography}[{\includegraphics[width=1in,height=1.25in,clip,keepaspectratio]{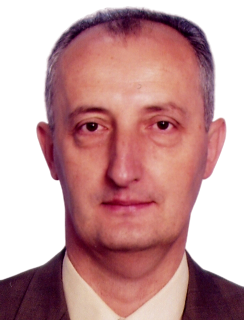}}]{Ivan Petrovi\'c} received the Master of Science and the Ph.D. degrees from FER Zagreb, Zagreb, Croatia, in 1990 and 1998, respectively. He is a Professor and the Head of the Laboratory for Autonomous Systems and Mobile Robotics, Faculty of Electrical Engineering and Computing, University of Zagreb, Zagreb, Croatia. He has published about 60 journal and 200 conference papers. His current research interest includes advanced control and estimation techniques and their application in autonomous systems and robotics. Prof. Petrović is a Full Member of the Croatian Academy of Engineering, and the Chair of the IFAC Technical Committee on Robotics.
\end{IEEEbiography}

\begin{IEEEbiography}[{\includegraphics[width=1in,height=1.25in,clip,keepaspectratio]{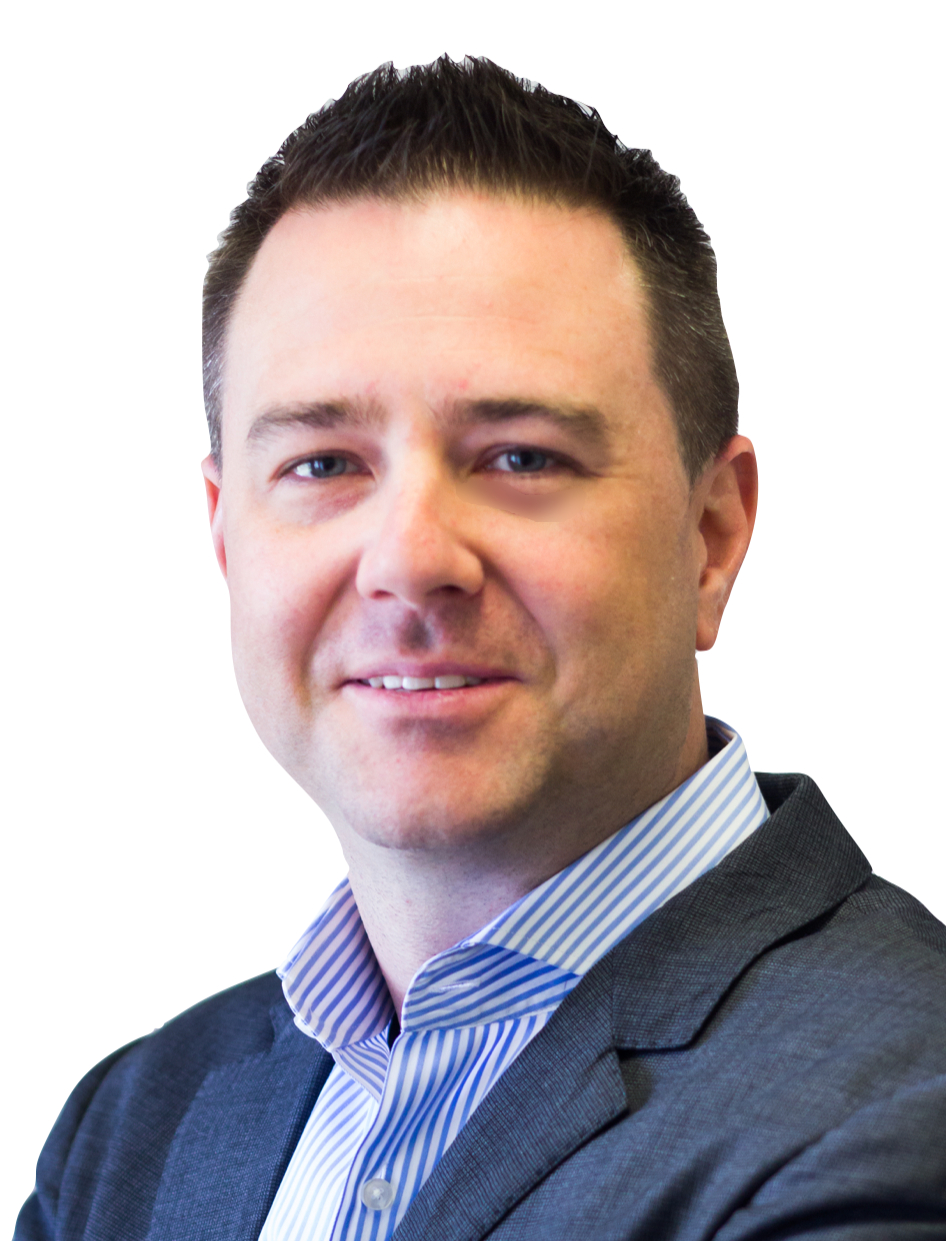}}]{Jonathan Kelly} received his Ph.D. degree from the University of Southern California, Los Angeles, USA, in 2011. From 2011 to 2013 he was a postdoctoral fellow in the Computer Science and Artificial Intelligence Laboratory at the Massachusetts Institute of Technology, Cambridge, USA. He is currently Dean's Catalyst Professor and Director of the Space and Terrestrial Autonomous Robotic Systems Laboratory, University of Toronto Institute for Aerospace Studies, Toronto, Canada. Prof. Kelly holds the Tier II Canada Research Chair in Collaborative Robotics. His research interests include perception, planning, and learning for interactive robotic systems.
\end{IEEEbiography}







\end{document}

%% file: notation.tex
\usepackage{xparse}

\NewDocumentCommand\bbm{}{ \begin{bmatrix} }
\NewDocumentCommand\ebm{}{ \end{bmatrix} }
\NewDocumentCommand\Vector{m}{ \boldsymbol{\mathbf{#1}} }

\NewDocumentCommand\Matrix{m}{ \boldsymbol{\mathbf{#1}} }

\NewDocumentCommand\Norm{m}{\left\Vert#1\right\Vert }

\NewDocumentCommand\Real{}{ \mathbb{R} }

\NewDocumentCommand\LieGroupO{m}{ \mathrm{O}(#1) }

\NewDocumentCommand\T{}{\mathsf{T}}

\NewDocumentCommand\diag{m}{\text{diag}\left(#1\right)}

\newcommand*\cpp{C\kern-0.1ex\raisebox{0.3ex}{\scalebox{0.9}{+\kern-0.2ex+}}}

%% file: packages.tex
\usepackage{times}

\usepackage{multicol}

\usepackage{xparse}
\usepackage{amsmath}
\usepackage{amssymb}
\usepackage{amsfonts}
\usepackage{dsfont}
\usepackage{graphicx}
\usepackage{url}
\usepackage{booktabs}
\usepackage[para]{threeparttable}
\usepackage{multirow}
\usepackage{makecell}

\usepackage{amsthm}
\usepackage[ruled,vlined]{algorithm2e}
\usepackage[colorlinks,bookmarksopen,bookmarksnumbered,citecolor=red,urlcolor=red]{hyperref}
\usepackage[capitalize]{cleveref}
\usepackage{xcolor}
\usepackage{framed}
\usepackage{caption}
\usepackage{subcaption}
\usepackage{pgfplots}
\usepackage{adjustbox}
\DeclareUnicodeCharacter{2212}{−}
\usepgfplotslibrary{groupplots,dateplot}
\usetikzlibrary{patterns,shapes.arrows}
\pgfplotsset{compat=newest}
\usepackage{flushend}

\SetKwInOut{Parameters}{Parameters}

\SetKwComment{Comment}{$\triangleright$\ }{}
\SetCommentSty{mycommfont}

\usepackage{setspace}
\usepackage[normalem]{ulem}

\colorlet{shadecolor}{gray!10}

\newtheorem{problem}{Problem}

\newtheorem{proposition}{Proposition}

%% file: sections/introduction.tex
\section{Introduction}

Articulated robots consist of actuated revolute joints connected by rigid links.
A significant portion of the difficulty associated with performing a specific task involves finding joint angles that achieve a desired end-effector pose.
%
%
Identifying a set of joint angles or a \textit{configuration} that reaches the desired goal pose(s) of one or more end-effectors is known as the \textit{inverse kinematics} (IK) problem~\cite{siciliano2010robotics}.
In general, this problem cannot be solved analytically and admits an infinite number of solutions for robots with redundant degrees of freedom (DOF).
Therefore, most approaches resort to numerical methods that solve constrained local optimization problems over joint angles.
This leads to constraints on end-effector and link poses that manifest as nonlinear expressions involving joint angles and kinematic parameters such as link lengths.
Without an informed or lucky initial guess, local search algorithms may converge to local minima, leading to inefficient performance or failure to find a sufficiently accurate solution.

\begin{figure}
  \centering
  \includegraphics[width=0.9\columnwidth]{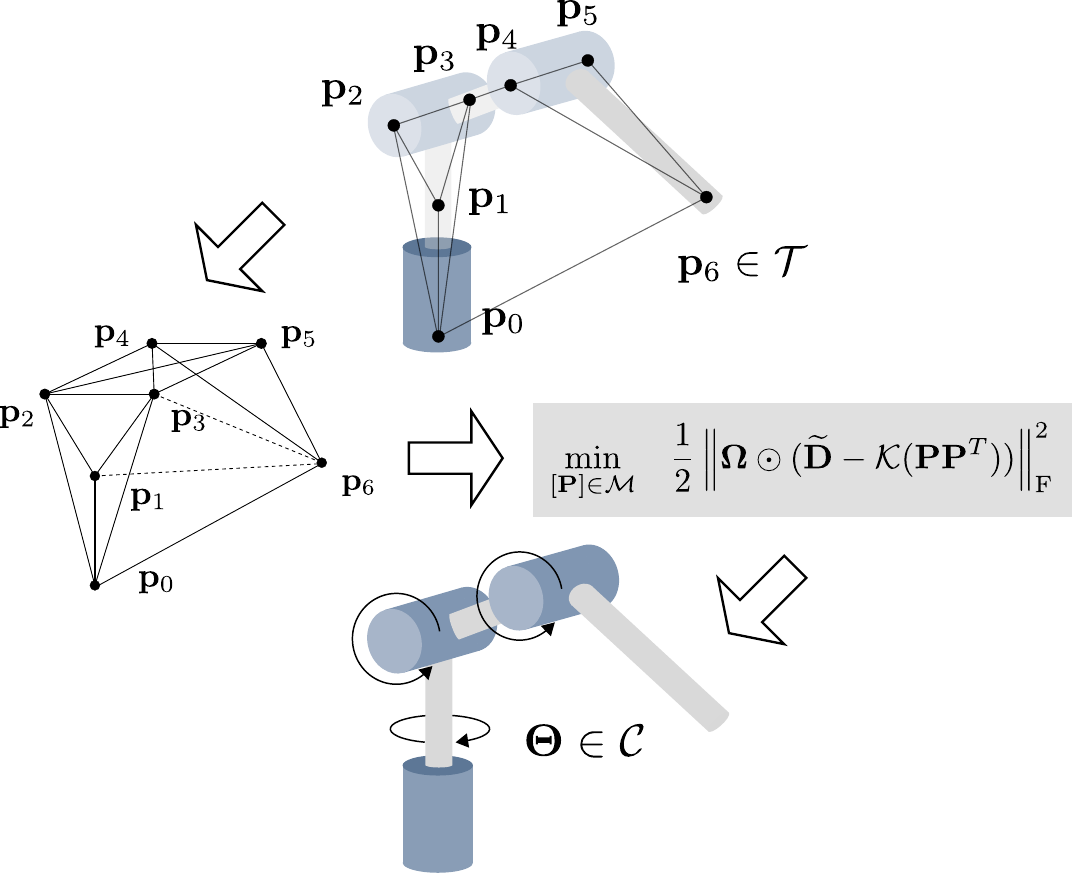}
  \caption{Overview of the proposed algorithm. A goal end-effector position $\Vector{p}_{6}$ is defined for a 3-DOF robotic manipulator (top); the inverse kinematics problem is to find the corresponding joint angles $\Vector{\Theta}$. Our method uses the matrix $\widetilde{\mathbf{D}}$ of distances between points $\Matrix{P}$ common to all feasible IK solutions to define an incomplete graph whose edges are weighted by known distances. Then, we apply Euclidean distance matrix completion with the known distance selection matrix $\Matrix{\Omega}$ to recover the weights corresponding to the unknown edges, solving the IK problem.}\label{fig:system_overview}
\vspace{-0.65\baselineskip}
\end{figure}

Many problems can be expressed using points and their relative distances~\cite{dokmanic_euclidean_2015}.
In the case of IK, a geometrically-intuitive alternative to joint angle parameters is obtained by considering points attached to the body of the robot~\cite{dejalonTwentyfiveYearsNatural2007}.
The positions of these points on the robot and their relative distances can be used to describe the kinematic model and end-effector goal poses within an IK formulation~\cite{blanchini_convex_2017, le2019kinematics}.
This approach unifies the domain and codomain (i.e., the configuration and task spaces) of the kinematic model, eliminating trigonometric constraints that appear when using joint angles.
In contrast to an angle-based parametrization, the search space is not restricted to feasible configurations, but to arbitrary point ``conformations'' in Euclidean space.
The convergence of such approaches is hindered by a large and highly redundant search space, in part because the point set can be rotated and translated without affecting the distance constraints.
A more compact representation is obtained by modelling the robot using \emph{distances} as variables~\cite{porta_inverse_2005, han_inverse_2006}.
While this avoids the redundancy induced by the invariance of the points to rigid transformations, the search space also includes points in higher dimensions, as there is no guarantee that a given set of distances can be realized as a collection of points in two- or three-dimensional Euclidean space.

In this paper, we explore IK from a \emph{distance geometry} perspective, revealing an equivalence between IK and the general distance geometry problem (DGP)~\cite{libertiEuclideanDistanceGeometry2014}.
By formalizing this equivalence for a large class of robots comprised of planar (i.e., two-dimensional), spherical, and revolute joints, we are able to connect IK to a rich literature of DGP solutions based on low-rank matrix completion~\cite{nguyenLowRankMatrixCompletion2019}.
We find solutions using the method introduced by Mishra et. al.~\cite{Mishra_2011}, where a Euclidean distance matrix~\cite{dokmanic_euclidean_2015} is parametrized with the manifold of fixed-rank Gram matrices~\cite{Journ_e_2010}, maintaining the advantages of a relaxed search space with fixed dimensionality and reduced redundancy.
Additionally, we show that bound smoothing~\cite{havelDistanceGeometryTheory2002} can be used to generate informed initializations, further improving convergence.
Our main contributions are as follows:
\begin{enumerate}
    \item we present a distance-geometric formulation of IK for a variety of robot types and prove its equivalence to traditional angle-based IK;
    \item we extend this formulation by systematically incorporating common configuration and workspace constraints;
    \item we show that our IK formulation can be solved using generic Riemannian optimization methods for low-rank matrix completion; and
    \item we demonstrate that informed initializations can be generated using the bound smoothing method from distance geometry.
\end{enumerate}
Finally, we provide a free and open-source Python implementation of our algorithms and simulation experiments, which empirically show the effectiveness of our approach on a variety of robot models.

\Cref{sec:related} briefly reviews existing IK literature and the application of distance geometry to practical problems.
\Cref{sec:background} covers the relevant background material on distance geometry, elucidating the connection between Euclidean distance matrices, low-rank matrix completion, and Riemannian optimization.
\Cref{sec:ikdb} begins with an introduction to the IK problem and its relevant terminology.
Sections \ref{sec:kinematic_models} to \ref{sec:solution_recovery} introduce distance-geometric IK formulations for a variety of robots and constraints (e.g., end-effector poses and joint limits), while \Cref{sec:eqdgp} formally proves their equivalence to the DGP.
\Cref{sec:algorithm} gives a detailed description of our IK algorithm and the bound smoothing procedure that can be used to find an informed initialization.
\Cref{sec:experiments} contains extensive experimental results for several commercial manipulators and a variety of complex high-DOF mechanisms.
Finally, \Cref{sec:conclusion} provides a summary of our findings and a discussion of limitations and potential future work.
Compared to typical IK formulations based on joint angles, experimental results indicate that our method often achieves higher success rates and faster convergence, and outperforms benchmark algorithms when many workspace constraints are present.
%


%% file: sections/related_work.tex
\section{Related Work}
\label{sec:related}
We begin with a discussion of existing theory and algorithms for characterizing and solving inverse kinematics problems. 
Since our main contribution is a principled application of distance geometry to IK, we proceed 
with a review of the distance geometry literature, with a particular focus on its (previously) limited intersection with IK.

\subsection{Inverse Kinematics}
Due to its widespread use in areas such as robotics \cite{angeles2013computational} and computer graphics \cite{aristidouInverseKinematicsTechniques2018}, inverse kinematics remains an active research area with an abundance of relevant literature that we can only briefly summarize here.
IK is typically formulated with a set of trigonometric equations and inequalities that constrain the kinematic model configuration to achieve a desired end-effector position or pose~\cite{siciliano2010robotics}.
It is well established that kinematic chains with up to six DOF admit a finite number of configurations that satisfy these constraints for feasible end-effector poses~\cite{lee1988new}.
Moreover, the entire solution set can be determined analytically~\cite{manocha1994efficient, husty2007new}, and even generated in an automated manner using the \textit{IKFast} algorithm~\cite{diankov2010automated}.
Unfortunately, in many cases of interest there exist redundant DOF or multiple end-effectors, rendering an analytical approach infeasible.

Heuristic methods can be used to efficiently solve IK problems for a limited class of robots.
The \textit{cyclic coordinate descent} (CCD) algorithm~\cite{kenwright2012inverse} iteratively adjusts joint angles using simple geometric expressions, resulting in very low computational overhead and making it useful in real-time applications.
Similarly, the \emph{triangulation} algorithm from~\cite{muller2007triangualation} incurs an even lower computational overhead than CCD, providing global convergence guarantees within a predetermined number of calculations for robots comprised of joints with unconstrained axes of rotation.
While highly efficient, both methods require additional modifications and engineering in order to be adapted to robots with multiple end-effectors or with joint limits.
The algorithms developed by Han and Rudolph apply unique parametrizations to IK problems over spherical linkages~\cite{han_inverse_2006, han_unified_2007}.
These parametrizations reveal a piecewise-convex structure that simplifies the solution, but the approach does not apply to generic revolute manipulators and cannot readily incorporate simple constraints like joint limits.
Recently, Aristidou et al.\ introduced FABRIK~\cite{Aristidou_2011, aristidou_extending_2016}, a heuristic solution to the IK problem which uses iterative forward-backward passes over joint positions to quickly produce high-quality solutions for a variety of robot models.
Because of its speed and general applicability, we implement FABRIK as a benchmark for comparison with our algorithm in \Cref{sec:experiments}.
IK is often formulated as a local optimization problem over joint angles and solved using unconstrained or bounded nonlinear programming methods such as L-BFGS-B \cite{zhu_algorithm_1997} or SQP~\cite{schulman2014motion}.
These methods have robust theoretical underpinnings~\cite{boyd2004convex} and can approximately support a wide range of constraints through the addition of penalties to the cost function~\cite{beeson2015trac}.
However, the highly nonlinear nature of the problem makes them susceptible to local minima, often requiring multiple initial guesses before returning a global minimum, if at all.
In robotics, many such methods belong to the family of closed-loop IK (CLIK) techniques~\cite{siciliano2010robotics}, where the kinematic Jacobian (pseudo)inverse is used to apply differential kinematics in a closed-loop fashion, emulating a feedback control problem~\cite{sciavicco1986coordinate}.
Major advantages of CLIK methods include their ease of implementation and a variety of extensions providing numerical robustness~\cite{buss2005selectively} and efficient incorporation of secondary objectives through redundancy resolution~\cite{nakamura1987task}.
However, alongside the convergence issues commonly encountered with first-order local optimization, CLIK methods often have problems with singularities~\cite{spong2020robot} due to their reliance on the kinematic Jacobian.
In \Cref{sec:experiments}, we compare our algorithm to the IK solver recently introduced by Erleben et. al.~\cite{erleben_solving_2019}, who show that second-order methods with exact Hessian matrices have improved convergence properties.

Recently, several optimization approaches have been introduced that forgo the standard joint angle parametrization in favour of models based on Cartesian coordinates (also known as \emph{natural coordinates}~\cite{dejalonTwentyfiveYearsNatural2007}).
The authors of~\cite{dai_global_2019} use a piecewise-convex relaxation of the $\mathrm{SO(3)}$ group together with a set of points on the robot to formulate the constrained IK problem as a mixed-integer linear program (MILP).
Their formulation can detect infeasible problems and provide approximate solutions to feasible problems, at the cost of a computationally intensive solution method.
Yenamandra et al.~\cite{yenamandra_convex_2019} use a similar relaxation to formulate IK as a semidefinite program.
Blanchini et al.~\cite{blanchini_inverse_2015, blanchini_convex_2017} treat points on a rigid manipulator as virtual masses in a potential field, leading to ``minimum energy'' solutions to convex formulations of planar and spherical inverse kinematics.
Naour et al.~\cite{le2019kinematics} formulate IK as a nonlinear program over inter-point distances, showing that solutions can be recovered for unconstrained articulated bodies.
While our kinematic model is based on inter-point distances, our approach differs from previous work by encompassing a larger class of robots and allowing for constraints such as symmetric joint limits and spherical obstacle avoidance.
Moreover, we provide a comparison with both heuristic and nonlinear optimization approaches, demonstrating that our proposed solution method provides a benchmark for IK problems.

\subsection{Distance Geometry}
The theory of distance geometry plays an important role in the development of computational methods for analyzing problems defined using inter-point distances~\cite{libertiEuclideanDistanceGeometry2014}.
This elegant theoretical framework is often applied to solve a diverse set of problems spanning computational chemistry~\cite{havelDistanceGeometryTheory2002}, signal processing~\cite{dingSensorNetworkLocalization2010}, and acoustics~\cite{dokmanic_euclidean_2015}.
Liberti et al.~\cite{libertiEuclideanDistanceGeometry2014} present a detailed taxonomy of distance geometry problems, which can be collectively described as completing a partially-connected graph of inter-point distances.
When a high degree of connectivity is present (i.e., most distances are known), the classical multidimensional scaling (MDS) algorithm~\cite{cox2008multidimensional} is often used.
The EMBED algorithm models the problem of molecular conformation using distances, providing bounds on unknown distances using \emph{bound smoothing}~\cite{havelDistanceGeometryTheory2002} and iteratively finding solutions for smaller problem instances.
Larger problem instances that satisfy some additional assumptions can be solved with a branch-and-prune strategy~\cite{liberti2008branch}.
Convex relaxations of the DGP have been coupled with semidefinite programming methods in both chemistry and sensor network localization~\cite{biswas2006semidefinite, leung2010sdp}.

Known distances can be arranged in an incomplete Euclidean distance matrix (EDM)~\cite{dokmanic_euclidean_2015} of rank at most $K + 2$, where $K$ is the dimension of the Euclidean space.
In many applications, the DGP can be solved by determining the unknown EDM entries using a low-rank \emph{matrix completion} approach~\cite{nguyenLowRankMatrixCompletion2019}.
Recently, numerical optimization methods based on results from Riemannian geometry have been utilized to efficiently perform EDM completion~\cite{vandereyckenLowrankMatrixCompletion2012}.
Mishra et al.\ solve the EDM completion problem using a Riemannian trust-region method by parametrizing the EDM with elements of a quotient manifold invariant to orthogonal transformations~\cite{Mishra_2011}.
Nguyen et al. present a similar approach in \cite{nguyenLocalizationIoTNetworks2019} that finds solutions to the problem of sensor network localization using a Riemannian conjugate gradient method on a manifold representation of EDMs.

Many problems pertaining to active structures such as robots also admit purely distance-based descriptions~\cite{portaDistanceGeometryActive2018}.
Porta et al.\ relate the IK problem to EDM completion~\cite{porta_inverse_2005} for several common classes of manipulators and leverage an algebraic approach to find configurations reaching a desired end-effector pose.
For general systems of kinematic and geometric constraints, Porta et al.\ apply a complete but computationally expensive branch-and-prune solver that iteratively eliminates regions of the solution space using geometric techniques~\cite{porta_branch-and-prune_2005}.
Our recent work~\cite{maricInverseKinematicsSerial2020} applies a convex \textit{sum-of-squares} (SOS) relaxation~\cite{parrilo_semidefinite_2003, lasserre_global_2001} to a distance-geometric formulation of inverse kinematics, exploiting theoretical properties that guarantee a globally optimal solution for many problem instances.
Moreover, we previously showed that kinematic constraints induce an inherently sparse structure that can be used to significantly reduce the computational burden usually associated with the SOS approach (and other convex relaxation-based methods).

In this paper, we use the Riemannian optimization method in~\cite{Mishra_2011} to solve distance-geometric formulations of the IK problem.
We model the motion constraints of individual joint-link pairs by defining distances between key points in the structure of the robot in a manner similar to~\cite{porta_inverse_2005}, which allows us to set constraints on end-effector poses and joint angles by limiting the associated distances to a predetermined range.
In contrast to~\cite{porta_inverse_2005}, the numerical optimization approach makes our algorithm suitable for a more general class of robots with redundant DOF.
Crucially, we use the DGP as a mathematical basis for our approach and derive simple and inclusive criteria for the compatibility of a robot mechanism with our method.
Finally, we show that bound smoothing~\cite{havelDistanceGeometryTheory2002} can be used as an effective initialization method that significantly improves convergence.

%% file: sections/background.tex
\section{Background}\label{sec:background}
In this section, we state the core distance geometry problems that serve as a theoretical foundation for the proposed inverse kinematics formulation.
By analyzing well-known identities for representing collections of points in the form of  Euclidean distance matrices~\cite{dokmanic_euclidean_2015}, we arrive at a DGP formulation that is based on low-rank matrix completion~\cite{Hendrickson_1992}.
Finally, we summarize the optimization-based EDM completion approach introduced in~\cite{Mishra_2011}, which we use to solve the problems presented herein.

\subsection{Distance Geometry}\label{sec:dgp}
The fundamental problem of distance geometry, as stated in~\cite{libertiEuclideanDistanceGeometry2014}, is as follows:
\begin{problem}[Distance Geometry Problem]\label{prob:DGP}
Given an integer $K > 0$, a set of vertices $V$, and a simple undirected graph $G = (V,E)$ whose edges $\{u, v\} \in E$ (where $u, v \in V$) are assigned non-negative weights 
\begin{equation*}
	\{u,v\} \mapsto d_{u,v} \in \Real_{+} \, ,
\end{equation*}
find a function $ p :V \rightarrow \mathbb{R}^{K}$ such that the Euclidean distances between pairs match the assigned weights
\begin{equation}
	\forall\, \{u, v\} \in E\, , \quad \|p(u) - p(v) \| = d_{u,v}.
\end{equation}
\end{problem}
\noindent The function $p:V \rightarrow \mathbb{R}^{K}$ is also known as a \textit{realization} of the graph $G$.
Any realization $p$ of $G$ maps all the vertices in $V$ to a collection of points $\mathbf{P} \in \mathbb{R}^{|V| \times K}$, where each row is the position $\Vector{p}_{u} = p(u) \in \mathbb{R}^{K}$ of the point corresponding to vertex $u \in V$.

In some cases, we may wish to reduce the number of possible realizations by constraining a subset of inter-point distances (i.e., edge weights) to some interval.
Consequently, we can extend \Cref{prob:DGP} such that the edges in $E$ are weighted by positive intervals, resulting in the more general \textit{interval distance geometry problem}~\cite{libertiEuclideanDistanceGeometry2014}:
\begin{problem}[Interval Distance Geometry Problem]\label{prob:iDGP}
Given an integer $K > 0$ and a simple undirected graph $G = (V,E)$ whose edges $\{u, v\} \in E$ are weighted by intervals
\begin{equation*}
	\{u,v\} \mapsto [d_{u,v}^-, d_{u,v}^+] \subseteq \Real_+,
\end{equation*}
find a realization in $\Real^{K}$ such that Euclidean distances between pairs belong to the edge intervals
\begin{equation}
	\forall\, \{u, v\} \in E\, , \quad \|p(u) - p(v) \| \in [d_{u,v}^-, d_{u,v}^+].
\end{equation}
\end{problem}
\noindent Note that for all $e = \{u,v\} \in E$, the notation for \Cref{prob:iDGP} supports unconstrained or missing distances ($d_e^- = 0$, $d_e^+ \rightarrow \infty$), as well as the equality constraints found in \Cref{prob:DGP} ($d_e^- = d_e^+$).

In this paper, we refer to both \Cref{prob:DGP} and~\ref{prob:iDGP} as the DGP, where the specific formulation can be inferred from the presence or absence of distance intervals on the edge weights of $G$.
If a realization that satisfies an instance of the DGP exists, the corresponding collection of points $\Matrix{P}$ can be arbitrarily translated, rotated, and reflected such that the distance constraints still hold~\cite{dokmanic_euclidean_2015}.
This defines the equivalence class $[\Matrix{P}]$ of \cref{eq:equivalence}.
Additionally, there may be multiple equivalence classes $[\Matrix{P}]$ that correspond to distinct solutions to \Cref{prob:DGP} or~\ref{prob:iDGP}.
We will make use of this distinction in \Cref{sec:eqdgp}.

\subsection{Euclidean Distance Matrices}\label{subsec:EDM}
Consider a realization of a graph $G$, obtained by solving \Cref{prob:DGP}.
By arranging the resulting points in a matrix ${\mathbf{P} = {[\mathbf{p}_{0}, \mathbf{p}_{1},\dots ,\mathbf{p}_{N-1}]}^\T \in \mathbb{R}^{N \times K}}$, all inter-point distances $d_{u,v}$ can be determined via the Euclidean norm:
\begin{equation*}
	d_{u,v}= \|\mathbf{p}_{u} - \mathbf{p}_{v}\|\,.
\end{equation*}
The product $\mathbf{X} \triangleq \Matrix{P}\Matrix{P}^\T$ is known as the \textit{Gram matrix}, and belongs to $\mathcal{S}_{+}^{N}$, the set of $N\times N$ symmetric positive semidefinite matrices.
Elements of the Gram matrix can conveniently express squared inter-point distances
\begin{equation}\label{eq:gram_dist}
	d_{u,v}^{2} = \Matrix{X}_{u,u} -2\Matrix{X}_{u,v} + \Matrix{X}_{v,v}\,,
\end{equation}
and the full set of squared inter-point distances in \cref{eq:gram_dist} can be efficiently calculated using the matrix identity
\begin{equation}\label{eq:EDM}
	\Matrix{D} = \diag{\Matrix{X}}\Matrix{1}^{\T} + \Matrix{1}\diag{\Matrix{X}}^{\T} - 2\Matrix{X},
\end{equation}
where $\diag{\Matrix{X}}$ is the vector formed by the main diagonal of the Gram matrix~\cite{dokmanic_euclidean_2015}, and $\Matrix{1}$ is a column vector of ones.
The resulting matrix $\Matrix{D}$ is known as a \textit{Euclidean Distance Matrix} (EDM).
We use $\mathcal{K}\left(\Matrix{X}\right)$ to denote the linear operator mapping $\Matrix{X}$ to $\Matrix{D}$ as defined by \cref{eq:EDM}.

Consider the problem of recovering the original collection of points $\Matrix{P}$ from squared inter-point distances in the matrix $\Matrix{D}$.
Necessary and sufficient conditions for a matrix to be an EDM can be found in~\cite{Sippl_1986}.
If $\mathbf{D}$ is an EDM, a Gram matrix that satisfies \cref{eq:EDM} can be recovered by taking
\begin{equation}\label{eq:EDM2Gram}
	\Matrix{X} = -\frac{1}{2}\Matrix{J}\Matrix{D}\Matrix{J},
\end{equation}
where $\mathbf{J} = \mathbf{I} - \frac{1}{N}\mathbf{1}\mathbf{1}^{\T}$ is the so-called geometric centering matrix~\cite{dokmanic_euclidean_2015}.
Once $\Matrix{X}$ has been recovered, a collection of points $\hat{\Matrix{P}} \in \mathbb{R}^{N \times K}$ can be obtained through the eigenvalue decomposition $\Matrix{X} = \Matrix{U}\Matrix{\Lambda}\Matrix{U}^{\T}$ by taking the first $K$ eigenvalues $\lambda_i$\footnote{Assuming the recovered realization is $K$-dimensional, only the first $K$ eigenvalues are nonzero.}:
\begin{equation}\label{eq:SVD}
	\hat{\mathbf{P}}^{\T} = \left[\diag{\sqrt{\lambda_{0}},\sqrt{\lambda_{1}}, \dots, \sqrt{\lambda_{K-1}}}, \mathbf{0}_{K\times N-K}\right]\mathbf{U}^{\T}.
\end{equation}
While the squared distances of points in $\hat{\mathbf{P}}$ recovered using this procedure match those defined in $\mathbf{D}$ exactly, they are in general not equal to the original $\Matrix{P}$.
This is due to the fact that inter-point distances are preserved under rigid transformations.
In order to recover the absolute positions of the points, at least $K$ points, known as \textit{anchors}, need to have their positions defined a priori.
These anchors are used to formulate the orthogonal Procrustes problem~\cite{dokmanic_euclidean_2015}, whose solution is the rotation (or reflection) $\mathbf{R} \in \LieGroupO{K}$ and translation $\mathbf{t} \in \mathbb{R}^{K}$ that transform the positions of anchors in $\hat{\Matrix{P}}$ to their predefined positions.
This transformation can then be applied to all the points in $\hat{\Matrix{P}}$ to yield the desired set of points $\Matrix{P}$.

\subsection{Euclidean Distance Matrix Completion}
\label{subsec:EDMCP}

As discussed in Section~\ref{sec:dgp}, a collection of points can be described using a graph $G = (V,E)$ that is weighted by inter-point distances.
If all inter-point distances are known, the graph is \textit{complete}, meaning that all of its edges and corresponding weights are prescribed.
This graph can be compactly represented by the EDM whose elements are
\begin{equation}
	\forall\, \{u,v\} \in E, \quad \Matrix{D}_{u,v} \triangleq d_{u,v}^{2}\, ,
\end{equation}
and a realization of the graph can be obtained simply by taking the collection of points recovered via \cref{eq:EDM2Gram} and \cref{eq:SVD}.
It follows that the recovered collection of points is in fact a solution of the DGP defined by Problem~\ref{prob:DGP}.
Conversely, many DGP instances are represented by graphs that only have a subset of edges defined a priori, resulting in an EDM with missing elements.

The problem of finding the missing elements in a partially defined EDM is known as the \textit{EDM completion problem}~\cite{dokmanic_euclidean_2015}, which is strongly NP-hard in general~\cite{libertiEuclideanDistanceGeometry2014}.
By defining the symmetric binary matrix $\Matrix{\Omega}$ with elements
\begin{equation} \label{eq:omega}
	\Matrix{\Omega}_{u,v} \triangleq
	\begin{cases}
		\begin{split}
			&1 \enspace \mathrm{if} \enspace \{u, v\} \in E, \\
			&0 \enspace \mathrm{otherwise},
		\end{split}
	\end{cases}
\end{equation}
we arrive at a common statement of the EDM completion problem as low-rank matrix completion:
\begin{align}\label{eq:EDM_completion}
	\min_{\Matrix{X}\in \mathcal{X}}  f(\Matrix{X}) \triangleq \frac{1}{2}\left\lVert \Matrix{\Omega}\odot(\widetilde{\Matrix{D}}-\mathcal{K}\left(\Matrix{X}\right))\right\rVert_{\mathrm{F}}^{2}\,,
\end{align}
where $\odot$ is the Hadamard (element-wise) matrix product, the subscript F denotes the Frobenius norm, and $\widetilde{\Matrix{D}}$ is the incomplete distance matrix.
Since the workspace dimension $K$ of the robot is known, the Gram matrix $\Matrix{X}$ defined in \cref{eq:gram_dist} is constrained to the manifold
\begin{equation}\label{eq:SPD}
	\mathcal{X} = \left\{\Matrix{P}\Matrix{P}^{\T} : \Matrix{P}\in \mathbb{R}_{*}^{N\times K} \right\}\, ,
\end{equation}
which is exactly the manifold of rank-$K$ positive-semidefinite matrices.
The NP-hardness of this problem originates from the non-convex constraint on the rank of $\Matrix{X}$, which can be relaxed in order to obtain a solution using Euclidean local search or  \textit{semidefinite programming}~\cite{alfakih1999solving}.
From~\cref{eq:SPD}, we see that relaxing the rank constraint expands the search space to collections of points with dimension greater than $K$.
This is fundamentally incompatible with physical problems, for which the dimension of points is fixed.
In fact, many interior point methods that solve the resulting convex semidefinite program tend to return a max-rank (and therefore potentially non-physical) solution~\cite{so_theory_2007}.

We can avoid explicit rank constraints in \cref{eq:EDM_completion} by using the Burer-Monteiro factorization~\cite{Burer_2004} to define the cost function directly in terms of the points $\Matrix{P} \in \mathbb{R}^{N \times K}$.
This results in a non-convex optimization problem
\begin{align}\label{eq:EDM_completion_iDGP2}
	f^{*} = \min_{\Matrix{P}\in \mathbb{R}^{N \times K}} & f(\Matrix{P}\Matrix{P}^\T),
\end{align}
which reduces the number of variables without changing the global minimum~\cite{fangEuclideanDistanceMatrix2012}.
Following the derivation in~\cite{Mishra_2011}, the gradient of \cref{eq:EDM_completion} with respect to $\Matrix{P}$ is defined as
\begin{equation}
	\nabla_{\Matrix{P}} f = 4\big(\Matrix{S} - \text{diag}\left(\Matrix{S}\Matrix{1}\right)\big) \, \Matrix{P},
\end{equation}
where $\Matrix{S} = \Matrix{\Omega}\odot(\widetilde{\Matrix{D}}-\mathcal{K}\left(\Matrix{X}\right))$ and $\text{diag}\left(\Matrix{S}\Matrix{1}\right)$ is a diagonal matrix formed by the vector $\Matrix{S}\Matrix{1}$.

Second-order optimization methods benefit from an exact analytical expression of the Hessian $\nabla_{\Matrix{P}}( \nabla_{\Matrix{P}}f) = \nabla_{\Matrix{P}}^{2}f$.
In~\cite{chuLEASTSQUARESEUCLIDEAN}, an analytic expression for the full Hessian of \cref{eq:EDM_completion_iDGP2} is obtained in an element-wise fashion.
However, this expensive computation can be avoided by observing that many optimization methods only require the Hessian-vector product~\cite{Pearlmutter_1994} and by making use of the identity
\begin{equation*}
	D_{\Matrix{Z}}(\nabla_{\Matrix{P}}f) \triangleq \nabla_{\Matrix{P}}^{2}f\cdot\Matrix{Z}\, ,
\end{equation*}
where $D_{\Matrix{Z}}(\nabla_{\Matrix{P}}f) = \frac{d \nabla f\left({(\Matrix{P} + t\Matrix{Z})(\Matrix{P} + t\Matrix{Z})}^{\T}\right)}{dt}$ is the Frech\'et (directional) derivative of the gradient in the direction $\Matrix{Z}$.
We then obtain
\begin{equation}\label{eq:hessian}
	\begin{aligned}
		\nabla_{\Matrix{P}}^{2}f\cdot \Matrix{Z} = & \,4\,D_{\Matrix{Z}}\big(\Matrix{S} - \text{diag}\left(\Matrix{S}\Matrix{1}\right)\big)\,\Matrix{P} \\
		+                                          & \,4\,\big(\Matrix{S} - \text{diag}\left(\Matrix{S}\Matrix{1}\right)\big)\,\Matrix{Z}\, .
	\end{aligned}
\end{equation}
Unlike gradient descent, second-order optimization methods feature a superlinear convergence rate, which is useful when highly accurate solutions of \cref{eq:EDM_completion_iDGP2} are required.

\subsection{Optimization on the Manifold}
\label{subsec:Riemannian}
As stated in Section~\ref{subsec:EDM}, inter-point distances are invariant to rigid transformations of the underlying point set.
It follows that the problem in \cref{eq:EDM_completion_iDGP2} is invariant to right-multiplication of the variable $\Matrix{P}$ with orthogonal matrices $\Matrix{Q} \in \LieGroupO{K}$.
This results in non-isolated minima, which have been shown to cause step evaluation issues for second-order methods~\cite{absil2009optimization, Journ_e_2010} when close to a solution, rendering the classical result of superlinear convergence void.
This issue is circumvented in~\cite{Journ_e_2010} by considering the set of all equivalence classes of the form
\begin{equation}\label{eq:equivalence}
	[\Matrix{P}] = \left\{\Matrix{P}\Matrix{Q} \vert \,\Matrix{Q} \in \mathbb{R}^{K \times K}\, , \Matrix{Q}^{\T}\Matrix{Q}= \Matrix{I}\right\}.
\end{equation}
Elements of this set constitute a manifold $\mathcal{M}$ which is the quotient of the set of full-rank $N \times K$ matrices by the orthogonal group $\LieGroupO{K}$:
\begin{equation}\label{eq:manifold}
	\mathcal{M} \triangleq \mathbb{R}_{*}^{N \times K}/\, \LieGroupO{K}.
\end{equation}
It follows that the overall search space is reduced by reformulating \cref{eq:EDM_completion_iDGP2} on the manifold $\mathcal{M}$ as
\begin{align}\label{eq:EDM_completion_3}
	\phi^{*} = \min_{\left[\Matrix{P}\right]\in \mathcal{M}} & \phi(\left[\Matrix{P}\right]),
\end{align}
where $\phi(\left[\Matrix{P}\right]) = f(\Matrix{P}\Matrix{P}^{T})$.
Further, it can be shown that the quotient manifold $\mathcal{M}$ has the structure of a Riemannian manifold~\cite{absil2009optimization}.
\begin{definition}[Riemannian Manifold]
	A Riemannian manifold $\mathcal{M}$ is a smooth manifold equipped with a positive-definite inner product $g_{\Matrix{P}} : T_{\Matrix{P}}\mathcal{M} \times  T_{\Matrix{P}}\mathcal{M} \to \mathbb{R}$ on the tangent space $T_{\Matrix{P}}\mathcal{M}$ at each point $\Matrix{P} \in \mathcal{M}$ that varies smoothly from point to point.
\end{definition}
\noindent Next, we provide an overview of how the EDM completion problem in \cref{eq:EDM_completion_3} can be adapted to the Riemannian setting, as described by Mishra et. al.~\cite{Mishra_2011}.
We refer the reader to~\cite{absil2009optimization} for a detailed treatment of quotient manifolds and their geometry.

Formally, the tangent space $T_{\Matrix{P}}\mathcal{M}$ of a point $\Matrix{P}$ in $\mathcal{M}$ is the space of all tangent vectors $\gamma'(0)$ to curves $\gamma: \mathbb{R} \rightarrow \mathcal{M}$, where $\gamma(0) = \Matrix{P}$.
The tangent space $T_{\Matrix{P}}\mathcal{M}$ is endowed with the inner product, also known as the Riemannian metric
\begin{equation}\label{eq:riemmanian_metric}
	g_{\Matrix{P}}(\Matrix{Z}_{1}, \Matrix{Z}_{2}) = \operatorname{Tr}(\Matrix{Z}^{\T}_{1}\Matrix{Z}_{2}), \quad \Matrix{Z}_{1},\Matrix{Z}_{2} \in T_{\Matrix{P}}\mathcal{M},
\end{equation}
which is the usual metric on $\mathbb{R}^{N\times K}$.
%
We can divide the tangent space into two orthogonal subspaces~\cite{Journ_e_2010}
\begin{equation*}
	T_{\Matrix{P}}\mathcal{M} = \mathcal{V}_{P}\mathcal{M} \oplus \mathcal{H}_{P}\mathcal{M}\, ,
\end{equation*}
where the tangent space to the equivalence classes in \cref{eq:equivalence} is known as the \textit{vertical subspace}
\begin{equation}
	\mathcal{V}_{P}\mathcal{M} = \left\{\Matrix{P}\Matrix{Q} \vert \,\Matrix{Q} \in \mathbb{R}^{K \times K}\, , \Matrix{Q}^{\T} + \Matrix{Q} = 0\right\}\, ,
\end{equation}
and the orthogonal complement of $\mathcal{V}_{P}\mathcal{M}$ in $T_{P}\mathcal{M}$ is known as the \textit{horizontal subspace}
\begin{equation}\label{eq:horizontal}
	\mathcal{H}_{\Matrix{P}}\mathcal{M} = \left\{\Matrix{Z} \in  T_{\Matrix{P}}\mathbb{R}_{*}^{N \times K} \vert \, \Matrix{Z}^{\T}\Matrix{P} = \Matrix{P}^{\T}\Matrix{Z}\right\}.
\end{equation}
Given a tangent vector $\Matrix{Z} \in T_{\Matrix{P}}\mathcal{M}$ at a point $\Matrix{P} \in \mathcal{M}$, we can recover the horizontal component from \cref{eq:horizontal} using the \textit{horizontal projection} operator.
\begin{definition}[Horizontal projection]
	The horizontal projection $P_{\mathcal{H}_{\Matrix{P}}\mathcal{M}}: T_{\Matrix{P}}\mathcal{M} \rightarrow \mathcal{H}_{\Matrix{P}}\mathcal{M}$, that recovers the horizontal lift $\Matrix{Z}_{\mathcal{H}}$ of the tangent vector $\Matrix{Z} \in T_{\Matrix{P}}\mathcal{M}$ corresponding to the horizontal subspace in \cref{eq:horizontal} is defined as
	\begin{equation}\label{eq:projection}
		P_{\mathcal{H}_{\Matrix{P}}\mathcal{M}}(\Matrix{Z}) = \Matrix{Z} - \Matrix{P}\Matrix{C}\, ,
	\end{equation}
	where $\Matrix{C}$ is a skew-symmetric matrix solving the Sylvester equation:
	\begin{equation*}
		\Matrix{C}\Matrix{P}^{\T}\Matrix{P} + \Matrix{P}^{\T}\Matrix{P}\Matrix{C} = \Matrix{P}^{\T}\Matrix{Z} - \Matrix{Z}^{\T}\Matrix{P}.
	\end{equation*}
\end{definition}
\noindent Using the projection operator $P_{\mathcal{H}_{\Matrix{P}}\mathcal{M}}$, derivatives of the function $\phi$ (defined on the manifold) are computed from the derivatives of the function $f$ (defined in Euclidean space)~\cite{absil2009optimization} by producing the horizontal lift of the Euclidean gradient of $f$ at point $\Matrix{P}$:
\begin{equation}\label{eq:manifold_grad}
	\text{grad}_{\Matrix{P}} \phi = P_{\mathcal{H}_{\Matrix{P}}\mathcal{M}}(\nabla_{\Matrix{P}} f).
\end{equation}
Similarly, by projecting the directional derivative of the gradient defined in \cref{eq:hessian}, we compute the Hessian-vector product of $\phi$ from that of $f$ as
\begin{equation}\label{eq:manifold_grad}
	\text{Hess}_{\Matrix{P}} \phi \left[ \Matrix{Z}\right] = P_{\mathcal{H}_{\Matrix{P}}\mathcal{M}}(D_{\Matrix{Z}}(\text{grad}_{\Matrix{P}}\phi)).
\end{equation}
Once the geometrically-correct derivatives have been produced, the step size is calculated and the point is moved along a descent direction on the manifold.
To ensure the resulting point remains on the manifold, we use the \textit{retraction} operator.
\begin{definition}[Retraction]
	In order to apply a direction of movement in $\mathcal{H}_{\Matrix{P}}\mathcal{M}$ while staying on the manifold $\mathcal{M}$, we use the \textit{retraction} operator, which is defined as
	\begin{equation}\label{eq:retr}
		R_{\Matrix{P}}(\Matrix{W}) = \Matrix{P} + \Matrix{W}.
	\end{equation}
\end{definition}
\noindent The projection and retraction operators allow for the adaptation of classic local optimization algorithms to the Riemannian setting~\cite{boumal2015lowrank, weiGuaranteesRiemannianOptimization2016}.
%


%% file: sections/dg_ik.tex
\section{Distance-Geometric Inverse Kinematics}
\label{sec:ikdb}
Articulated bodies, such as the robotic manipulator shown in Fig.~\ref{fig:system_overview}, are composed of revolute joints connected by rigid links.
Joint angles can be arranged in a vector $\boldsymbol{\Theta} \in \mathcal{C}$, where $\mathcal{C} \subseteq \Real^n$ is known as the \textit{configuration space}.
Analogously, the poses of one or multiple end-effectors constitute the \textit{task space} $\mathcal{T}$.
The mapping $F: \mathcal{C} \rightarrow \mathcal{T}$, relating joint variables to task space coordinates (e.g., end-effector poses) is known as the \textit{forward kinematics} of a robot.
This leads to the accompanying notion of \textit{inverse kinematics}, which is simply the inverse mapping $F^{-1}: \mathcal{T} \rightarrow \mathcal{C}$.
Since the mapping $F^{-1}$ is generally not injective (i.e., there are multiple solutions for a single target pose) and cannot be determined analytically, numerical methods are instead used to find a single solution.
This leads us to the definition of the central problem in this paper:
\begin{problem}[Inverse Kinematics]\label{prob:IK}
  Given a task space goal $\Vector{w} \in \mathcal{T}$, find a configuration $\boldsymbol{\Theta} \in \mathcal{C}$ such that the forward kinematic mapping $F(\boldsymbol{\Theta}) = \Vector{w}$ holds.
\end{problem}
Note that the IK problem extends to cases where various constraints on the position of the robot are present.
For example, it is often necessary to avoid collision with obstacles in the environment, as well as to respect limits on joint angles.
Therefore, while the core problem is relatively simple, it can be extended to include challenging instances that commonly occur in practice.
Unlike most approaches that attempt to directly solve \Cref{prob:IK} in terms of joint variables $\Vector{\Theta}$, we adopt an alternative formulation based on inter-point distances~\cite{maricInverseKinematicsSerial2020, porta_inverse_2005, Le_Naour_2019}.
Our approach, which allows us to trivially recover $\Vector{\Theta}$ from our distance-based solution, is developed in detail in Sections \ref{sec:kinematic_models} to \ref{sec:solution_recovery} and summarized in \cref{fig:system_overview}.
In \cref{sec:eqdgp}, we prove that our formulation of \cref{prob:IK} is equivalent to the distance geometry problem~\cite{libertiEuclideanDistanceGeometry2014} for a broad class of manipulators.
Finally, in \Cref{subsec:special_cases} we show how robots with planar and spherical joints constitute special cases for which our formulation can be trivially simplified.

\input{sections/revolute_joints.tex}

\input{sections/dgp_equivalence.tex}

\input{sections/planar_and_spherical_joints.tex}


%% file: sections/revolute_joints.tex
\subsection{Kinematic Model}\label{sec:kinematic_models}
\begin{figure}
	\centering
	\includegraphics[width=0.8\columnwidth]{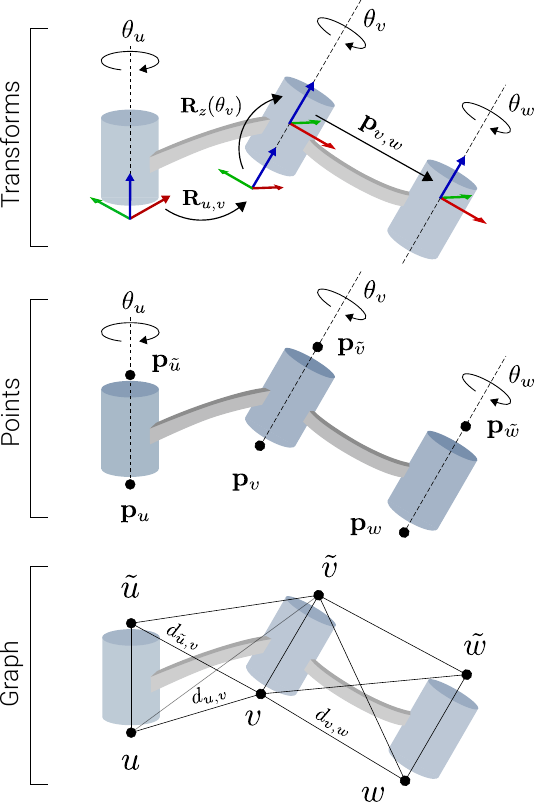}
	\caption{
		Visualization of the point placement used to describe a generic linkage of revolute joints and the corresponding graph representation.
		The top graphic shows the transformations used to obtain the poses of the joint coordinate frames.
		The middle graphic shows how pairs of points indexed by $(u, \tilde{u})$, $(v, \tilde{v})$, and $(w, \tilde{w})$ are placed along the rotation axis of their respective joints.
		The bottom graphic shows how the corresponding vertex representations form a graph whose edges are weighted by the known inter-point distances defined by link geometry in \cref{eq:dh_links}.
	}\label{fig:linkage_model}
\end{figure}

\begin{figure*}
	\centering
	\begin{subfigure}{0.245\textwidth}
		\centering
		\includegraphics[scale = 0.65]{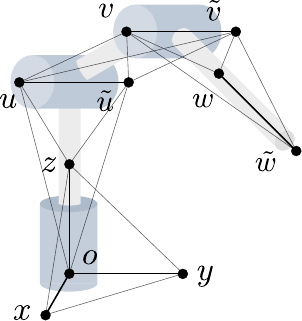}
		\caption{}\label{fig:revolute_link_constr}
	\end{subfigure}
	\begin{subfigure}{0.245\textwidth}
		\centering
		\includegraphics[scale = 0.65]{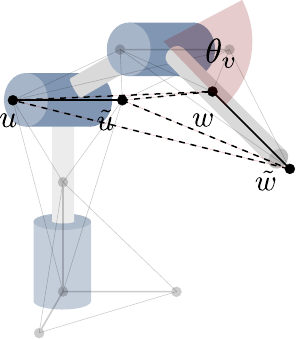}
		\caption{}\label{fig:revolute_angle_constr}
	\end{subfigure}
	\begin{subfigure}{0.245\textwidth}
		\centering
		\includegraphics[scale = 0.65]{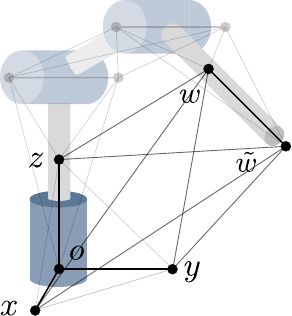}
		\caption{}\label{fig:revolute_goal_constr}
	\end{subfigure}
	\begin{subfigure}{0.245\textwidth}
		\centering
		\includegraphics[scale = 0.65]{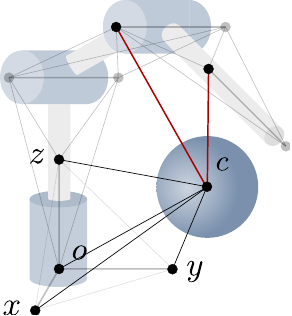}
		\caption{}\label{fig:revolute_chain_a}
	\end{subfigure}
	\caption{Visualization of the procedure in Section~\ref{sec:ikdb} for a 3-DOF revolute manipulator with joint limits. a) The solid black lines represent fixed distances between neighbouring joints and between base vertices. b) The dashed lines correspond to the distances constrained to some interval $[d^{-},\ d^{+}]$ determined by symmetric limits on the joint angle $\theta_{v}$. c) The solid black lines represent the distances fixed by setting a desired end-effector pose. d) Spherical obstacles are represented as vertices in the graph whose position is fixed by defining the distances to the base nodes. The lines drawn in red represent some distances whose lower bounds can be set in order to achieve obstacle avoidance.}
\label{fig:revolute_chain_big}
\end{figure*}
Articulated robots are comprised of a series of single-axis revolute joints oriented to provide a useful range of poses in their task spaces.
Our goal in this section is to construct a graph representation of such mechanisms that is compatible with the DGP formulation of \Cref{prob:iDGP}.
We achieve this by rigidly attaching a pair of points to the rotation axis of each joint in a manner similar to \cite{porta_inverse_2005}.
As the example in~\cref{fig:linkage_model} illustrates, the distances between points corresponding to neighbouring joints are invariant to changes in their angles during movement.
These distances are key to describing the degrees of freedom of the robot.

Consider the points attached to the rotation axes of neighbouring joints, shown in \cref{fig:linkage_model} and labelled $u$ and $v$.
We denote the positions of these points and the orientations of their respective coordinate frames as $\Vector{p}_{u}$, $\Vector{p}_{v}$, $\Matrix{R}_{u}$, and $\Matrix{R}_{v}$, respectively.
Further, $\Vector{p}_{u,v}$ and $\Matrix{R}_{u,v}$ denote the position and orientation of the fixed coordinate frame at $v$ in the rotating coordinate frame at $u$. The matrix $\Matrix{R}_{z}(\theta_u)$ rotates $\Vector{p}_v$ and its child joints about the axis $\hat{\Vector{z}}$ by the joint angle $\theta_{u}$.
Given the joint angles $\Vector{\Theta}$, these positions and orientations can be computed recursively as:
\begin{equation}\label{eq:root_points}
	\begin{split}
		\Matrix{R}_{v} &= \Matrix{R}_{u}\,\Matrix{R}_{z}(\theta_u)\,\Matrix{R}_{u,v}\, ,\\
		\Vector{p}_{v} &= \Vector{p}_{u} + \Matrix{R}_{u}\,\Matrix{R}_{z}(\theta_u) \,\Vector{p}_{u,v}\\
	\end{split}
	\,\forall\, (u,v):v \succ u,
\end{equation}
where $\succ$ indicates that the joint indexed by $u$ is the parent of the joint indexed by $v$ in the directed graph describing the robot structure.\footnote{We assume that the graph is a tree whose root is the fixed robot base and whose leaves are the end-effectors; we leave extensions of our formulation to parallel manipulators, which contain loops, for future work.}
For neighbouring joints, $\Vector{p}_{u,v}$ and $\Matrix{R}_{u,v}$ are determined by the robot model parametrization (e.g., the DH convention~\cite{hartenberg1955kinematic} or Lie groups~\cite{lynch2017modern}).
The second pair of points, labelled by $\tilde{u}$ and $\tilde{v}$, are obtained by translation along axis of rotation of each joint:
\begin{equation}\label{eq:aux_points}
	\begin{split}
		\Vector{p}_{\tilde{u}} &= \Vector{p}_u + \Matrix{R}_{u}\hat{\Vector{z}}\, ,\\
		\Vector{p}_{\tilde{v}} &= \Vector{p}_v + \Matrix{R}_{v}\hat{\Vector{z}}\, .
	\end{split}
\end{equation}
Together, these four points describe the relative position and orientation of the joints' rotation axes.
For a given robot, we index the points obtained using~\cref{eq:root_points} and~\cref{eq:aux_points} with the set of vertices $V_{s} \subset V$ of an incomplete graph $G = (V,E)$, where the set of edges $E$ is weighted by inter-point distances that are known a priori.
These distances describe the overall geometry and degrees of freedom of the robot, and they are invariant to the feasible motions of the robot (i.e., they remain constant in spite of changes to the configuration $\Matrix{\Theta}$):
\begin{equation}\label{eq:dh_links}
	\begin{split}
		d_{u, \tilde{u}} &= d_{v, \tilde{v}} = 1, \\
		d_{u, v} &= \lVert \Vector{p}_{u,v} \rVert, \\
		d_{u, \tilde{v}} &= \lVert \Vector{p}_{u,v} + \Matrix{R}_{u,v}\hat{\Vector{z}}\rVert,\\
		d_{\tilde{u}, v} &= \lVert \Vector{p}_{u,v} - \hat{\Vector{z}}\rVert,\\
		d_{\tilde{u}, \tilde{v}} &= \lVert \Vector{p}_{u,v} - \hat{\Vector{z}} + \Matrix{R}_{u,v}\hat{\Vector{z}}\rVert\, .
	\end{split}
	\,\forall\, (u,v) \in V_{s} \, .
\end{equation}
Depending on the specific link geometry, some identities in \cref{eq:dh_links} may vanish, allowing us to merge identical points and reduce the overall size of the graph.

\subsection{Constraints}
In addition to describing the rigid structure of a robot using the distances in \cref{eq:dh_links}, we also need to introduce distance constraints that encode features of specific IK problem instances.
To that end, this section describes how additional vertices and distances may be used to constrain the end-effector(s) of a robot, implement obstacle avoidance, and enforce joint limits.

\subsubsection{Base structure}
In order to uniquely specify points with known positions (i.e., end-effectors) in terms of distances, we define the ``base vertices" of a robot as  $V_b = \{o, x, y, z\} \subset V$, where $o$ is the root or base joint.
The elements of $V_b$ are used to form a coordinate frame with vertex $o$ at the origin, as shown in \cref{fig:revolute_link_constr}.
We achieve this by defining the following edge weights:
\begin{equation} \label{eq:base_structure_constraints}
	\begin{split}
		d_{o, x} &= d_{o, y} = d_{o, z} = 1, \\
		d_{x, y} &= d_{x, z} = d_{y, z} =\sqrt{2}.
	\end{split}
\end{equation}
This base structure is used in the remainder of the section to specify, in terms of distance constraints, end-effector poses and joint limit constraints for links starting at the root.
Note that the vertices $x$ and $y$ may be dropped in cases where the root joint angle is not limited to some interval, as this makes the solution set depend only on the distances from the goal points to the root $o$ and to each other.
Moreover, we may trivially reduce the graph size by using the points $\Vector{p}_{o}$ and $\Vector{p}_{\tilde{o}}$ attached to the base joint axis instead of the vertices $o$ and $z$~\cite{porta_inverse_2005}.

\subsubsection{End-effector pose} \label{subsubsec:end_effector_constraints}
We consider two types of task space goals for an end-effector:
\begin{enumerate}
	\item a 3-DOF position goal ($\mathcal{T}_p = \Real^3$), and
	\item a 5-DOF ``direction goal" defined by the position of two distinct points ($\mathcal{T}_d = \mathcal{T}_p \times \mathcal{T}_p$).\footnote{A full 6-DOF pose goal is not supported, as purely distance-geometric constraints cannot prevent reflections of the tool frame: this is equivalent to the assumption that the final joint has no angle limits when that axis is aligned with the final joint (e.g., for a common spherical wrist).}
\end{enumerate}
In both cases, $\Vector{w} \in \mathcal{T}$ is encoded as a set of points fixed to the end-effector.
These points are indexed by vertices $V_{e} \subset V$, and their positions are completely determined by the goal $\Vector{w}$.
Thus, an end-effector goal can be enforced by weighting the relevant edges with distances
\begin{equation}\label{eq:ee}
	d_{u,v} = \lVert \Vector{p}_{u} - \Vector{p}_{v}\rVert \, , \, u \in V_{e}\, , v \in V_{b}\, .
\end{equation}

\subsubsection{Joint limits}
Depending on the kinematic structure of the robot, symmetric joint limits can be represented by using distance intervals.
In \cref{fig:revolute_angle_constr}, we see that a given joint angle can be represented (up to sign) using up to four distinct distances:
\begin{equation*}
	\begin{aligned}
		d^{-}_{u,w} & = \sqrt{d^{2}_{u,v} + d^{2}_{v,w} - 2d_{u,v}d_{v,w}\cos(\theta_{v}^{\mathrm{lim}})}\, , \\
		d^{+}_{u,w} & = d_{u,v}+d_{v,w}\, .
	\end{aligned}
\end{equation*}
It follows that joint values can be restricted to symmetric intervals by constraining the distances between vertices assigned to the parent and child of a particular joint.
Generally, limiting one of the aforementioned distances to an interval results in a distinct set of joint angle limits, depending on the particular distance chosen.
However, it is important to note that the nature of the distance-based representation may make it difficult to implement arbitrary joint limits, as undesired symmetries may occur in certain ranges.

\subsubsection{Obstacle avoidance}
We extend our model to incorporate spherical obstacles whose centers are indexed with the set of vertices $V_{o} \subset V$.
The radius of each obstacle is given by the function $\rho: V_o \rightarrow \Real_+$.
Much like the elementary basis vectors in $V_b$, we can fix each center $\Vector{p}_c$, $\forall\, c \in V_o$, in the global reference frame and augment $G$ to include the constant inter-point distances for edges in $V_o \times V_o$ and $V_o \times V_b$:
\begin{equation}
	\lVert \Vector{p}_i - \Vector{p}_j \rVert = d_{i, j} \ \forall\, (i, j) \in V_o \times V_o \cup V_o \times V_b.
\end{equation}
Finally, points attached to the joints of a robot (i.e., those indexed by $V_s$) can be constrained to lie outside of each obstacle:
\begin{equation}
	\lVert \Vector{p}_u - \Vector{p}_c \rVert \geq \rho(c) \ \forall (u, c) \in V_s \times V_o.
\end{equation}
The radii given by $\rho$ can be inflated to account for the shape and size of the robot's joints or as a conservative safety measure.
For robots with long links, auxiliary points indexed by $V_{s'}$ can be easily added between points in $V_s$ for higher precision collision avoidance.

\subsection{Solution Recovery}
\label{sec:solution_recovery}
The remaining edge weights can be determined by completing the resulting partial EDM and a canonical realization $\Matrix{P}^{*}$ can be recovered.
This result is achieved by identifying the points indexed by $V_b$ with the origin and $K$ elementary basis vectors in $\Real^K$, which act as anchors in the solution of the Procrustes procedure~\cite{dokmanic_euclidean_2015} discussed in \Cref{subsec:EDM}.
For $K=3$, since the anchors form a right-handed frame that fully specifies a 6-DOF pose, $\Matrix{P}^*$ is unique.
In \Cref{prop:equivalence}, we prove that $\Matrix{P}^{*}$ will correspond to a unique feasible configuration $\boldsymbol{\Theta} \in \mathcal{C}$ if the successive joint axes of our robot are \textit{coplanar}.
Once $\Matrix{P}^{*}$ is obtained, we can iteratively recover all the joint values by solving
\begin{equation}\label{eq:rev_angle_rec}
	\begin{split}
		\theta_{u} = \min_{\theta} &\lVert \Matrix{R}_{u}\Matrix{R}_{z}(\theta)\Vector{p}_{u,v} - (\Vector{p}_{v} - \Vector{p}_{u}) \rVert^{2} \\
		+& \lVert \Matrix{R}_{u}\Matrix{R}_{z}(\theta)\Vector{p}_{u,v} + \Matrix{R}_v\hat{\Vector{z}} - (\Vector{p}_{\tilde{v}} - \Vector{p}_{u}) \rVert^{2}\, .
	\end{split}
\end{equation}
This problem can be reduced to finding the roots of a quartic polynomial and therefore admits a fast closed-form solution.
Alternatively, $\Matrix{\Theta}$ can be recovered from \cref{eq:root_points} and \cref{eq:aux_points} with inverse trigonometric functions.
Notably, the optimization formulation of \cref{eq:rev_angle_rec} is robust to numerical errors in the calculation of $\Matrix{P}^{*}$ by \Cref{alg:rtr}.


%% file: sections/dgp_equivalence.tex
\subsection{Equivalence to Distance Geometry}
\label{sec:eqdgp}

In this section, we prove that the robot models (i.e., the proposed graph descriptions) discussed thus far allow us to solve inverse kinematics (\Cref{prob:IK}) by means of the DGP (\Cref{prob:DGP}).
The main result is as follows:

\begin{proposition}[IK $\equiv$ DGP] \label{prop:equivalence}
	Suppose that the kinematic model of \Cref{sec:kinematic_models} describes a robot whose successive joint axes are \emph{coplanar}.
	Then, the solutions to \Cref{prob:DGP} correspond one-to-one with the solutions to \Cref{prob:IK}.
	More precisely, if $\Vector{p}_u, \Vector{p}_{\tilde{u}}, \Vector{p}_v$, and $\Vector{p}_{\tilde{v}}$ are coplanar for all $v \succ u$, then for any end-effector target $\Vector{w} \in \mathcal{T}$, we have a corresponding DGP encoded in $G$ and there exists a bijection
	\begin{equation}
		Q: \mathcal{C}_{\Vector{w}} \rightarrow \mathcal{G},
	\end{equation}
	where $\mathcal{C}_{\Vector{w}} \subset \mathcal{C}$ is the set of configurations achieving $\Vector{w}$, and $\mathcal{G}$ is the space of all realizations (up to an arbitrary Euclidean transformation) of $G$.
\end{proposition}

\begin{proof}
	We begin by recalling that the presence of base structure constraints (\cref{eq:base_structure_constraints}) in our model allows us to identify ``equivalent" realizations of a completion $G$ with a canonical point assignment $\Matrix{P}^*$.
	We will assume $K=3$ for the entire proof: $K=2$ is a special case which can be simplified by noting that all joints share the same axis of rotation and by removing the ``auxiliary" points defined by \cref{eq:aux_points}.

	For a given robot and $\Vector{w} \in \mathcal{T}$, the preceding sections described a set of DGP constraints which we write as the incomplete graph $G = (V,E)$.
	We will now proceed to prove that there is a bijection between $\mathcal{C}_{\Vector{w}}$ and the equivalence classes representing distinct solutions to $G$ (each equivalence class is represented by its canonical solution $\Matrix{P}^*$).
	It suffices to construct a map $Q$ that is injective and surjective, where $Q$ is simply the iterative procedure described by \cref{eq:root_points} and \cref{eq:aux_points}.

	\textit{Injective}: We will use a proof by contradiction. Suppose $\exists\,\Matrix{\Theta}_1 \neq \Matrix{\Theta}_2 \in \mathcal{C}_{\Vector{w}}$ such that $Q(\Matrix{\Theta}_1) = \Matrix{P}^* = Q(\Matrix{\Theta}_2)$.
	Let $u$ be the vertex label for a joint whose corresponding angle $\Matrix{\Theta}_1^u = \theta_{u_1} \neq \theta_{u_2} = \Matrix{\Theta}_2^u$, but has $\theta_{s_1} = \theta_{s_2}$ for all ancestors $s$ of $u$.
	At least one such $u$ is guaranteed to exist because \cref{eq:root_points} and \cref{eq:aux_points} tell us that the axis points $\Vector{p}_v$ and $\Vector{p}_{\tilde{v}}$ of joint $v \succ u$ only depend on $\theta_u$ and the positions of all its ancestor joints' points.
	This gives us:
	\begin{align}
		\Vector{p}_u + \Matrix{R}_u\Matrix{R}_z(\theta_{u_1})\Vector{p}_{u,v} & = \Vector{p}_u + \Matrix{R}_u\Matrix{R}_z(\theta_{u_2})\Vector{p}_{u,v}, \\ \notag
		\implies \Matrix{R}_z(\theta_{u_1})^\T \Matrix{R}_z(\theta_{u_2})     & = \Matrix{I},                                                            \\
		\implies \theta_{u_1}                                                 & = \theta_{u_2}, \notag
	\end{align}
	where we have assumed without loss of generality that $\Vector{p}_{u,v} \neq c\,\hat{\Vector{z}}$ for some $c \in \Real$.\footnote{In the special case where $\Vector{p}_{u,v} = c\,\hat{\Vector{z}}$, injectivity can be proved with $\Vector{p}_{\tilde{v}}$ instead of $\Vector{p}_{v}$. If $\Vector{p}_{\tilde{v}}$ is collinear with $\Vector{p}_{v}$, $\Vector{p}_{u}$, and $\Vector{p}_{\tilde{u}}$, then joint $v$ is a rotation around the same axis as joint $u$ and they can be effectively combined into a single joint.}
	This contradicts our premise that $\theta_{u_1} \neq \theta_{u_2}$ and proves that the mapping is injective.

	\textit{Surjective}: We will show that for each $\Matrix{P}^* \in \mathcal{G}$ there exists $\Matrix{\Theta} \in \mathcal{C}_{\Vector{w}}$ such that $Q(\Matrix{\Theta}) = \Matrix{P}^* \in \mathcal{G}$.
	By the definition of $\mathcal{G}$, $\lVert \Matrix{P}^*_v - \Matrix{P}^*_u \rVert = \lVert \Vector{p}_{u, v} \rVert$ and $\lVert \Matrix{P}^*_v - \Matrix{P}^*_{\tilde{u}} \rVert = \lVert \Vector{p}_{u, v} - \hat{\Vector{z}} \rVert$ for all $v \succ u$, therefore we can always find $\theta_u$ such that
	\begin{equation}
		\Matrix{P}^*_v = \Matrix{P}^*_u + \Matrix{R}_{u}\,\Matrix{R}_{z}(\theta_u) \,\Vector{p}_{u,v},
	\end{equation}
	as required by \cref{eq:root_points}.
	Since we have assumed that the points $\Vector{p}_u, \Vector{p}_{\tilde{u}}, \Vector{p}_v$, and $\Vector{p}_{\tilde{v}}$ are coplanar, the position of $\Matrix{P}^*_{\tilde{v}}$ is uniquely determined by the other three points and therefore \emph{must} take the form specified by $Q$ and given by $\theta_u$ in \cref{eq:aux_points}:
	\begin{equation}
		\Matrix{P}^*_{\tilde{v}} = \Matrix{P}^*_v + \Matrix{R}_{v}\hat{\Vector{z}} = \Matrix{P}^*_u + \Matrix{R}_{u}\,\Matrix{R}_{z}(\theta_u) \,\Vector{p}_{u,v} + \Matrix{R}_{v}\hat{\Vector{z}}.
	\end{equation}
	If the points were \emph{not} coplanar, the six distance constraints in \cref{eq:dh_links} would only specify their relative positions up to a reflection ambiguity (i.e., there would be two feasible tetrahedra with opposite ``chirality" or ``handedness").
	This situation is illustrated in \cref{fig:proof}.

	Finally, we address the interval constraints in $G$ which correspond to joint angle limits of $\mathcal{C}$ and obstacle avoidance in $\mathcal{T}$.
	We have established the desired bijection $Q: \mathcal{C}_{\Vector{w}} \rightarrow \mathcal{G}$ for IK problems defined only by equality constraints.
	Including interval constraints simply limits the space of DGP solutions to $\mathcal{G}' \subset \mathcal{G}$.
	Since the inverse of a bijection is a bijection, and a subset of the domain of a bijection induces a bijection, we have the desired $Q': \mathcal{C}_{\Vector{w}} \rightarrow \mathcal{G}'$ and the proof is complete.
\end{proof}

\Cref{prop:equivalence} establishes that the IK problem for a large class of robots can be formulated as a DGP.
Our experiments in \Cref{sec:experiments} demonstrate that the requirement of coplanar neighbouring axes is satisfied by many popular commercial manipulators.
Additionally, it is worth noting that planar and spherical manipulators satisfy this requirement by definition.

\begin{figure}
	\centering
	\includegraphics[width=0.5\columnwidth]{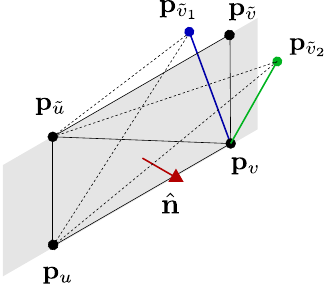}
	\caption{
	Illustration of the \emph{chirality} or \emph{handedness} issue that arises for non-coplanar pairs of joint axes $v \succ u$.
	If $\Vector{p}_u, \Vector{p}_{\tilde{u}}, \Vector{p}_v$, and $\Vector{p}_{\tilde{v}}$ are coplanar, our distance-geometric approach cannot introduce spurious solutions based on reflections of the true robot geometry.
	If $\Vector{p}_{\tilde{v}}$ were replaced with $\Vector{p}_{\tilde{v_1}}$ as shown, the spurious point $\Vector{p}_{\tilde{v_2}}$ would also satisfy the distance constraints in \cref{eq:root_points} because it is a reflection of $\Vector{p}_{\tilde{v_1}}$ across the plane containing $\Vector{p}_u, \Vector{p}_{\tilde{u}}$ and $\Vector{p}_v$.
	}\label{fig:proof}
\end{figure}


%% file: sections/planar_and_spherical_joints.tex
\subsection{Special Cases}\label{subsec:special_cases}
\begin{figure}[b]
  \centering
	\begin{subfigure}{0.2\textwidth}
    \centering
    \includegraphics[scale = 0.6]{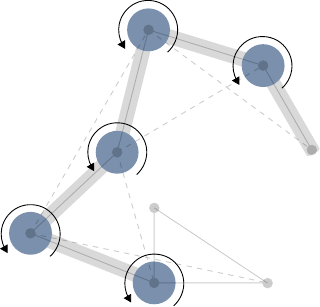}
    \caption{}\label{fig:planar_chain}
  \end{subfigure}
	\begin{subfigure}{0.2\textwidth}
    \centering
    \includegraphics[scale = 0.6]{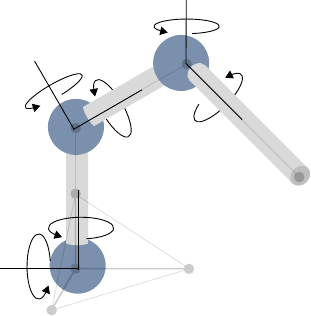}
    \caption{}\label{fig:spherical_chain}
  \end{subfigure}
  \caption{Visualization of planar and spherical mechanisms. These are two common examples of models for which our IK formulation uses a reduced number of variables.}\label{fig:planar_and_spherical_chain}
\end{figure}
In many cases, the generic DGP formulation of the IK problem presented in this section will result in kinematically redundant points in $V_s$.
One example is points $\Vector{p}_u$, $\Vector{p}_v$ that can be selected to coincide (i.e., $\Vector{p}_u =\Vector{p}_v$ for all values of $\Vector{\Theta}$) because they lie on joint axes $u \succ v$ that intersect.
While finding a generic strategy for generating DGP representations with minimal graph sizes is beyond the scope of this paper, we can identify two cases of practical importance where this reduction is both trivial and significant.

By reducing the point dimensionality to $K=2$, we can represent planar mechanisms using a single point per joint, as shown in~\cref{fig:planar_chain}.
Without loss of generality, this significantly reduces both the number of points and distances used to describe the IK problem, resulting in lower overall computation times.
For a detailed derivation of this simplified planar problem formulation, see \cite{maricInverseKinematicsSerial2020}.

A similar simplification can be obtained for a class of robots with joints allowing full or partial spherical motion, such as those of the human shoulder and hip.
In addition to humanoids, these spherical joints can be found in highly redundant snake-like robots used in applications that include pipe inspection and surgery~\cite{ananthanarayanan_real-time_2015}.
As shown in~\cref{fig:spherical_chain}, the joints can be represented as two orthogonal revolute joints in the same position, which again allows us to drop the points used to define the rotation axes, while still allowing us to limit the elevation angle and thereby restricting link motions to a spherical cone.
The remaining constraints can be trivially derived from those presented for the general case.


%% file: sections/algorithm.tex
\section{Algorithm}
\label{sec:algorithm}

To summarize so far, in~\cref{sec:kinematic_models} we derived distance geometric IK formulations for a variety of different robot types.
In~\cref{subsec:EDMCP} and~\cref{subsec:Riemannian} we showed how a Riemannian optimization approach to low-rank matrix completion can be used to efficiently solve instances of the DGP.
In this section, we propose an extension to the optimization problem in \Cref{subsec:Riemannian}, making it compatible with the distance-geometric IK formulation.

Modifying the matrix $\Matrix{\Omega}$ from \cref{eq:omega} to select only those EDM elements corresponding to known invariant distances in $\widetilde{\Matrix{D}}$, we introduce the matrices $\Matrix{\Psi}_{-}$ and $\Matrix{\Psi}_{+}$ that select inter-point distances with lower or upper bounds defined by joint limits or obstacle avoidance constraints (i.e., the known elements of $\widetilde{\Matrix{D}}_{-}$ and $\widetilde{\Matrix{D}}_{+}$).
We can now formulate the DGP as the Riemannian optimization problem
\begin{equation}\label{eq:EDM_completion_iDGP}
	\begin{split}
		\min_{\left[\Matrix{P}\right]\in \mathcal{M}}  &\phi(\left[\Matrix{P}\right]) \triangleq \frac{1}{2}\left\lVert \Matrix{\Omega}\odot(\widetilde{\Matrix{D}}-\mathcal{K}\left(\Matrix{P}\Matrix{P}^{T}\right))\right\rVert_{\mathrm{F}}^{2}\\ + &\frac{1}{2}\left\lVert \text{max}\left\{\Matrix{\Psi}_{\pm}\odot(\pm \widetilde{\Matrix{D}}_{\pm} \mp \mathcal{K}(\Matrix{P}\Matrix{P}^{T})), 0 \right\}\right\rVert_{\mathrm{F}}^{2}\, .
	\end{split}
\end{equation}
The first term in this cost serves to enforce equality constraints on the distance matrix $\mathcal{K}(\Matrix{X})$, representing the constant set of distances defined by robot and task geometry.
The second term enforces joint limits and obstacle avoidance constraints by setting either a lower or upper bound on a set of distances related to the joint rotation angle or distances from the obstacle center, with the element-wise $\max$ operator producing a non-zero value when these bounds are violated.
Assuming that the joint limits are symmetric, the lower or upper bound on the distances will be implicitly enforced by the link length constraints, meaning that only one bound needs to be explicitly included in \cref{eq:EDM_completion_iDGP}.

\subsection{The Riemannian Trust-Region Algorithm}\label{sec:trust_region}
The quality of the configuration recovered by the reconstruction step outlined in \cref{sec:solution_recovery} depends on highly accurate solutions to~\cref{eq:EDM_completion_iDGP}.
As mentioned in \cref{subsec:EDMCP}, the superlinear convergence guarantee of second-order optimization methods helps to quickly obtain accurate solutions.
To maintain this guarantee for the non-isolated minima of~\cref{eq:EDM_completion_iDGP}, we use the second-order Riemannian trust region (RTR) algorithm introduced in \cite{genrtr}.
Briefly, the trust region algorithm focuses on sequentially solving the problem
\begin{equation}\label{eq:rtr_subproblem}
	\begin{split}
		\min_{\Matrix{Z} \in \mathcal{H}_{\Matrix{P}}\mathcal{M}} &m_{\Matrix{P}}(\Matrix{Z})\, ,\\
		\text{s.t.} \quad &g_{\Matrix{P}}(\Matrix{Z}, \Matrix{Z}) \leq \Delta^{2} \, ,
	\end{split}
\end{equation}
where
\begin{equation}\label{eq:linearlized_cost}
	m_{\Matrix{P}}(\Matrix{Z}) \triangleq \phi(\Matrix{P}) + g_{\Matrix{P}}(\Matrix{Z}, \text{grad}_{\Matrix{P}}\phi) + \frac{1}{2}g_{\Matrix{P}}(\Matrix{Z}, \text{Hess}_{\Matrix{P}} \phi \left[ \Matrix{Z} \right]).
\end{equation}
In other words, a quadratic approximation of the model constructed at point $\Matrix{P}$ informs a search for the optimal descent direction $\Matrix{Z}$ within a trust region of radius $\Delta$.
Accepting or rejecting a candidate descent direction and updating the trust region radius is based on the quotient
\begin{equation}\label{eq:rtr_ratio}
	\rho = \frac{\phi(\Matrix{P}) - \phi(R_{\Matrix{P}}(\Matrix{Z}))}{m_{\Matrix{P}}(0) - m_{\Matrix{P}}(\Matrix{Z})}.
\end{equation}
A basic variant of this procedure is formalized in \Cref{alg:rtr}, where the subproblem in \cref{eq:rtr_subproblem} is approximately solved using the truncated conjugate gradient method introduced in \cite{genrtr}.
The numerical cost per iteration of this approach was shown in~\cite{Mishra_2011} to be $\mathcal{O}(dK + NK + NK^{2} + K^{3})$, where $d$ is the number of known entries in the EDM.
Since $K \in [2,3]$ for IK problems, the complexity of \Cref{alg:rtr} is linear with respect to the number of points and known distances.
Similarly to conventional trust region methods, the dominant computational bottleneck of RTR is the Hessian calculation, making the additional overhead added by the horizontal projection in~\cref{eq:projection} comparatively insignificant.
\begin{algorithm}[t]
	\SetAlgoLined
	\KwIn{Initial point $\Matrix{P}_{0}$}
	\KwData{Cost function $\phi$}
	\Parameters{$\bar{\Delta} > 0$, $\Delta_{0} \in [0, \bar{\Delta}]$, $\rho' \in [0, \frac{1}{4})$}
	\KwResult{Solution $\Matrix{P}_{N}$}
	\For {$k = 0,1,\,\dots\, N$}{
		$\Matrix{Z}_{k} \leftarrow$~\cref{eq:rtr_subproblem} and~\cref{eq:linearlized_cost}\Comment*[f]{Compute step}\\
		$\rho \leftarrow$~\cref{eq:rtr_ratio}\\
		\uIf(\Comment*[f]{Update TR radius}){$\rho < \frac{1}{4}$}{
			$\Delta_{k+1} \leftarrow \frac{1}{4}\Delta_{k}$
		}
		\uElseIf{$\rho > \frac{3}{4}$ and $\lVert\Matrix{Z}_{k}\rVert = \Delta_{k}$}{
			$\Delta_{k+1} \leftarrow \min(2\Delta_{k}, \bar{\Delta})$
		}
		\Else{$\Delta_{k+1} \leftarrow \Delta_{k}$}
		\uIf(\Comment*[f]{Accept or reject step}){$\rho > \rho'$}
		{$\Matrix{P}_{k+1} \leftarrow R_{\Matrix{P}}(\Matrix{Z}_{k})$ \Comment*[f]{Retraction (\cref{eq:retr})}}
		\Else{$\Matrix{P}_{k+1} \leftarrow \Matrix{P}_{k}$}
	}
	\caption{Riemannian Trust-Region (RTR)}\label{alg:rtr}
\end{algorithm}

\subsection{Bound Smoothing}
\label{sec:bound_smoothing}
It is well established that the convergence of local optimization methods depends on the choice of starting point.
The distance-based parametrization of the IK problem has the unique advantage of admitting informed initializations generated via a procedure known as \textit{bound smoothing}~\cite{havelDistanceGeometryTheory2002}.
First, we take the known set of distance bounds generated by the robot structure and problem constraints as described in Section~\ref{sec:ikdb} and form the graph $G$.
Next, a bipartite graph is formed with two copies of $G$, with the vertices of the two graphs connected by edges weighted by their negative respective distance.
Finally, the resulting all-pairs shortest path problem is solved using the Floyd-Warshall algorithm in $\mathcal{O}(|V|^{3})$ time.
For the IK problems analyzed in this paper, the computation time of this search is on the order of 1 ms for a simple Python implementation on a laptop computer.
The lower bounds on the distances between vertices can now be obtained by taking the shortest path between their representations in different subgraphs, with the upper bounds being the shortest path within one subgraph.
An initial guess for the distance matrix, known as a pre-EDM, can then be generated by sampling individual elements within these bounds.
Note that this procedure can be applied iteratively to produce even better approximations.
Moreover, the computation time can be further reduced through the use of parallelization.

The full algorithm is described in~\cref{alg:rik}.
We assume that an incomplete graph $G$ describing the IK problem and an initializing conformation $\Matrix{P}_{0}$ are provided as input.
Alternatively, we may also generate an initialization using the bound smoothing procedure described in the previous section.
Next, we solve the local optimization problem using~\cref{alg:rtr} and transform the resulting configuration back to the canonical coordinate system.
Finally, we recover the joint angle variables using~\cref{eq:rev_angle_rec}.

\begin{algorithm}[t]
	\SetAlgoLined
	\KwIn{Incomplete graph $G$, Initial point $\Matrix{P}_{0}$}
	\KwResult{Configuration $\Matrix{\Theta}$}
	Define $\phi$ from $G$ via~\cref{eq:EDM_completion_iDGP}\\
	\If{\textbf{not} $\Matrix{P}_{0}$}
	{
		Get $\Matrix{P}_{0}$ using bound smoothing (\cref{sec:bound_smoothing})
	}
	$\Matrix{P}^{*} \leftarrow \text{RTR}(\Matrix{P}_{0}; \phi)$\\
	$\Matrix{P} \leftarrow \text{OrthogonalProcrustes}(\Matrix{P}^{*})$\\
	Obtain $\Matrix{\Theta}$ from $\Matrix{P}$ using~\cref{eq:rev_angle_rec}
	\caption{Inverse Kinematics}\label{alg:rik}
\end{algorithm}


%% file: sections/experiments.tex
\section{Experiments}
\label{sec:experiments}

In this section, we present an analysis of the performance of our method relative to multiple benchmark algorithms in a series of simulation studies.
A variety of 2D and 3D kinematic models, including commercial manipulators and hyper-redundant kinematic chains and trees, were tested with and without joint angle limits and spherical obstacles.
All Python code used in our experiments is freely available in our Git repository.\footnote{\texttt{\url{https://github.com/utiasSTARS/GraphIK}}}

\subsection{Experimental Methodology}
\label{sec:methodology}
The results for each experiment were obtained with the following procedure:

\begin{enumerate}
	\item generate a kinematic model (e.g., a planar robot with links arranged in a perfect binary tree  with randomly generated symmetric joint angle limits);
	\item randomly sample an angle configuration $\Matrix{\Theta}_g \in \mathcal{C}$ for this model from a uniform distribution over the joint angle limits;
	\item determine the target position(s) or pose(s) $\Vector{w} \in \mathcal{T}$ of the end-effector(s) using $\Matrix{\Theta}_g$ and the model's forward kinematics (i.e., $\Vector{w} = F(\Matrix{\Theta}_g)$);
	\item run each IK algorithm on the problem instance defined by the kinematic model and the goal $\Vector{w}$ using $\Matrix{\Theta}_0 = \Vector{0}$ as the initial configuration;
	\item record statistics for each algorithm, including the number of iterations required until convergence, runtime, end-effector error(s), and joint angle limit violations;
	\item repeat steps 2-5 above $N$ times and summarize the statistics.
\end{enumerate}
In all tables and figures, the success rate of an algorithm for a particular experiment was determined as the portion of runs where the solution satisfied all of the following criteria:
\begin{itemize}
	\item joint angle limits were obeyed to within a tolerance of 1\% of the bound magnitude;
	\item obstacle avoidance constraints were obeyed to within a tolerance of 0.01 m;
	\item the sum of the position errors of the end-effectors was less than 0.01 m; and
	\item the sum of the rotation errors of the end-effectors was less than 0.01 rad.
\end{itemize}
The success rates of experiments reported in the tables and waterfall curves correspond to 95$\%$ Jeffreys confidence intervals~\cite{tonycai_one-sided_2005}.
The statistics on solution error, runtime, and the number of iterations required for convergence that appear in various tables and figures are computed using the entire set of $N$ runs (not just the successful portions).
We denote joint angle-limited variants of experiments with a $+$ symbol (e.g., results labelled ``6-DOF$+$'' use the same robot as those labelled ``6-DOF'' but additionally enforce randomly generated joint angle limits).

In all relevant figures, the results from our algorithm are labelled \texttt{RTR} for ``Riemannian Trust Region.''
When the bound smoothing procedure of \Cref{sec:bound_smoothing} is used, \texttt{-B} is appended to the label (i.e., \texttt{RTR-B}).
For each experiment, we compare our approach with a variety of benchmark algorithms from the optimization and IK literature.
While we report runtime statistics for many of our experiments, we stress that our selection of baseline algorithms is designed to illustrate a variety of techniques and highlight the unique advantages of our novel distance-geometric formulation.
Therefore, we did not choose particularly fast or industrially-proven implementations and opted for Python variants of all algorithms.
Finally, all experiments were performed on a laptop computer with a 2.9 GHz Dual-Core Intel Core i5 processor.

\subsection{Benchmark Algorithms} \label{subsec:benchmarks}
As discussed in \cref{sec:background}, generalized approaches to solving the IK problem most often resort to numerical methods that search for a joint configuration $\Matrix{\Theta} \in \mathcal{C}$ satisfying the defined constraints.
Among the large variety of such approaches, local nonlinear programming is perhaps the most common and versatile.
Therefore, we primarily compare our algorithm to the formulation
\begin{align} \label{prob:slsqp}
	\min_{\Vector{\Theta} \in \mathcal{C}} & \quad \Norm{\Vector{e}\left(F(\Matrix{\Theta}), \Vector{w}\right)}^2                                                                     \\
	\text{s.t.}                     & \quad \Norm{\Vector{p}_i(\Vector{\Theta}) - \Vector{c}_j}^2 \geq l^2_j \enspace \forall\, i \in V_{s}, \enspace \forall\, j \in V_{o},\notag \\
	                                       & \quad \Vector{\Theta}_{\mathrm{min}} \le \Vector{\Theta} \le \Vector{\Theta}_{\mathrm{max}},\notag
\end{align}
where $F(\Matrix{\Theta})$ is the forward kinematic mapping and $\Vector{w}$ is the goal, as defined in \Cref{prob:IK}.
The vector-valued function $\Vector{e}$ in the objective represents an appropriate error for the task space.
The inequality constraints serve to enforce obstacle avoidance and joint limits.
Note that while the error $\Vector{e}$ can also be driven to zero with an equality constraint, we have empirically found that incorporating the error in an objective function results in higher success rates---this phenomenon is also reported in~\cite{beeson2015trac}.

\subsection{Hyper-redundant and Tree-like Robots} \label{sec:planar_results}
\begin{table*}
	\centering
	\resizebox{\textwidth}{!}{\input{tables/planar_chain}}
	\caption{Results for planar chain manipulators over 1,000 random experiments with pose goals. The $+$ indicates joint angle limits.}
	\label{tab:planar_chain_results}
\end{table*}

\begin{table*}
	\centering
	\resizebox{\textwidth}{!}{\input{tables/planar_tree}}
	\caption{Results for planar tree manipulators over 1,000 random experiments with pose goals. The $+$ indicates joint angle limits.}
	\label{tab:planar_tree_results}
\end{table*}

We begin by analyzing the performance of our algorithm for hyper-redundant and tree-like planar robots.
This approach helps to avoid introducing confounding factors in the analysis, as the choice of any particular revolute manipulator opaquely affects the difficulty of IK problems.
More importantly, these mechanisms allow us to systematically scale the size and number of constraints of the IK problem by adding joints and introducing multiple end-effectors, while minimizing the number of redundant points and fixed distances as noted in~\Cref{subsec:special_cases}.
Since the full pose of a planar end-effector is determined by its position and the position of its parent joint, the error $\Vector{e}$ in the cost function of the joint angle-based approach in \cref{prob:slsqp} can simply be defined as the difference between the end-effectors' and their parent joints' positions and their goal positions in $\Vector{w}$.
We solve~\cref{prob:slsqp} using second-order trust region methods with similar convergence guarantees to those provided by our algorithm, referring to the unconstrained method as \texttt{trust-exact} and the joint angle-constrained method as \texttt{trust-constr}.
Our open source code provides an interface for testing this problem formulation with other solvers provided by \texttt{scipy.optimize}.
In addition to the formulation in \cref{prob:slsqp}, we implemented the \texttt{FABRIK} \cite{Aristidou_2011} heuristic IK solver in Python as an additional benchmark for our experiments.
In contrast to generic nonlinear solvers, this is an IK-specific heuristic method tailored to planar and spherical robots.

For the experiments in this section, all algorithms were allowed a total of 2,000 iterations and used a numerical tolerance of $10^{-9}$ for all stopping criteria.
Where applicable, the magnitude of the gradient of the cost function or Lagrangian was used as the stopping criterion.
Otherwise, the magnitude of the cost function or norm of the variable change in one iteration was used.
All other parameters were assigned their default values as provided by the \texttt{pymanopt}~\cite{townsend2016pymanopt} and \texttt{scipy.optimize} libraries~\cite{virtanen2020scipy}.
Since the implementations of these benchmark algorithms were not extensively tuned for performance, we place greater emphasis on the number of iterations taken by each algorithm as a more meaningful statistic than runtime in the results to follow.
\texttt{FABRIK} was allowed a maximum of 2,000 full forward and backward iterations per problem instance.

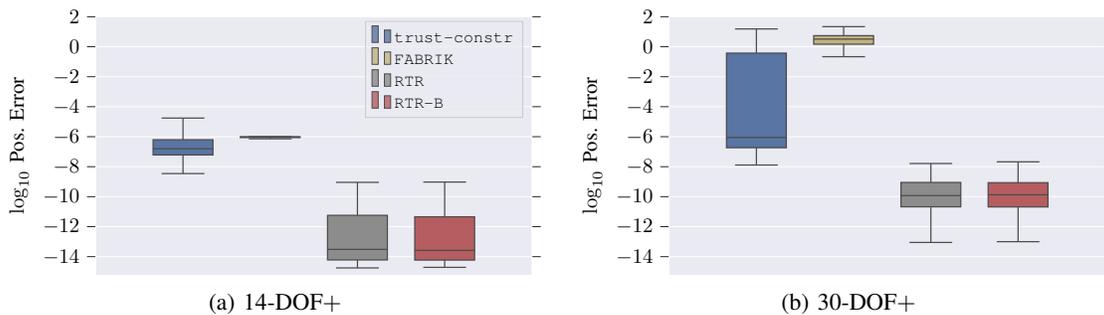
\begin{figure*}
	\centering
	\begin{subfigure}[t]{0.4\textwidth}
		\centering
		\begin{adjustbox}{width=\linewidth}
			\input{figs/results/boxplot_planar_tree_dof_14_bounded_True.tex}
		\end{adjustbox}
		\caption{14-DOF$+$}
		\label{fig:boxplot_planar_tree_14DOF}
	\end{subfigure}
	~
	\begin{subfigure}[t]{0.4\textwidth}
		\centering
		\begin{adjustbox}{width=\linewidth}
			\input{figs/results/boxplot_planar_tree_dof_30_bounded_True.tex}
		\end{adjustbox}
		\caption{30-DOF$+$}
		\label{fig:boxplot_planar_tree_30DOF}
	\end{subfigure}
	\caption{Box-and-whiskers plots summarizing end-effector position error over 1,000 experiments with planar binary tree robots with joint angle limits.}
	\label{fig:boxplots_planar}
\end{figure*}
\begin{figure*}
	\centering
	\begin{subfigure}[t]{0.4\textwidth}
		\centering
		\begin{adjustbox}{width=\linewidth}
			\input{figs/results/waterfall_planar_tree_dof_14_bounded_True.tex}
		\end{adjustbox}
		\caption{14-DOF$+$}
		\label{fig:waterfall_planar_tree_14DOF}
	\end{subfigure}
	~
	\begin{subfigure}[t]{0.4\textwidth}
		\centering
		\begin{adjustbox}{width=\linewidth}
			\input{figs/results/waterfall_planar_tree_dof_30_bounded_True.tex}
		\end{adjustbox}
		\caption{30-DOF$+$}
		\label{fig:waterfall_planar_tree_30DOF}
	\end{subfigure}
	\caption{Waterfall curves of success rate versus position error tolerance for 1,000 experiments with planar tree robots with joint angle limits. The shaded regions are 95$\%$ Jeffreys confidence intervals centered on the solid lines. The rotation error tolerance is fixed at 0.01 rad.}
	\label{fig:waterfall_planar}
\end{figure*}
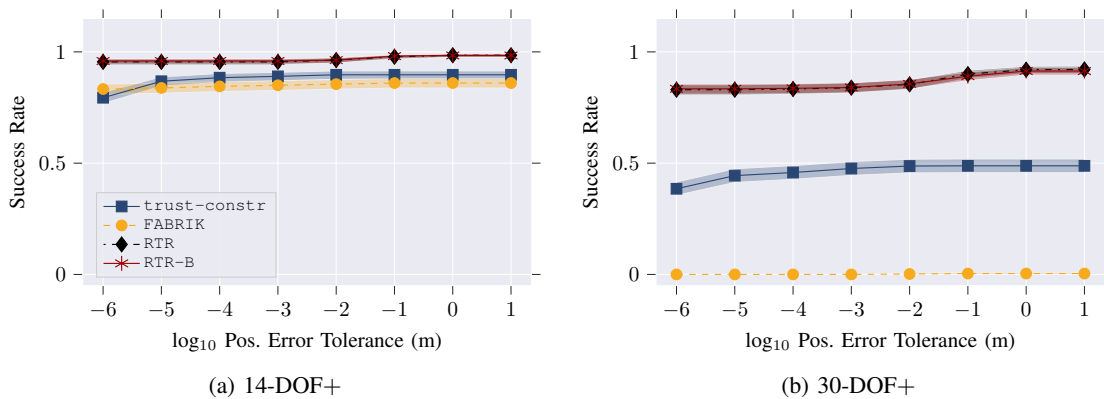

\Cref{tab:planar_chain_results} summarizes our results for 6- and 10-DOF planar manipulators of the type shown in~\cref{fig:planar_chain}, with and without joint angle limits.
Unsurprisingly, all four algorithms achieve a 100\% success rate when no joint angle limits are present.
When joint angle limits are introduced, \texttt{FABRIK} performs far worse, while the \texttt{RTR} algorithm performs similarly to \texttt{trust-constr} in terms of success rate and iterations required.
Initializing the problem using bound smoothing reduces the number of iterations needed  while increasing the success rate of our \texttt{RTR-B} algorithm compared to a naive initialization.
Curiously, both \texttt{trust-constr} and \texttt{FABRIK}'s performance improve as the number of DOF increases.
We suspect that the additional DOF give the heuristic approach of \texttt{FABRIK} more time to ``steer'' away from a difficult initial configuration induced by pose goals.  
The \texttt{trust-constr} algorithm may benefit from more DOF that can be used to ``escape'' from local minima or extremely flat regions of the cost function landscape.
While we only report on experiments involving \emph{pose} goals in this paper, we found that when joint angle limits were introduced, \texttt{FABRIK} was much better suited to \emph{position} goals (i.e., without a specified orientation) for the robot end-effector(s).

\Cref{tab:planar_tree_results} contains results for experiments involving binary tree-like robots, such as the 6-DOF example with a height of two shown in~\cref{fig:convergence}.
The 14- and 30-DOF results correspond to robots with a perfect binary tree kinematic structure of height three and four, respectively.
While not as practical as chain-like robots or the revolute manipulators of \Cref{subsec:rev_exp_1}, these experiments serve to showcase the performance of IK methods on highly kinematically-constrained mechanisms, while also scaling naturally to a higher number of DOF.
In all cases, the \texttt{RTR} and \texttt{RTR-B} algorithms outperform the three benchmark approaches in terms of success rate.
When no joint limits are present, the overall difference in success rates is relatively small, with \texttt{RTR} and \texttt{trust-exact} having a similar number of outer iterations.
When joint limits are introduced, the experimental procedure in~\Cref{sec:methodology} generates problems within a tighter range around a naive initialization, resulting in higher success rates for all approaches.
Due to a high number of constraints, the number of outer iterations for \texttt{trust-constr} increases more than ten-fold, significantly degrading performance.
In contrast, this effect does not occur for \texttt{RTR} and \texttt{RTR-B}, where the number of iterations remains considerably lower while a similar increase in success rate is observed.
The box-and-whiskers plots in \cref{fig:boxplots_planar} summarize the position error statistics for each algorithm over \emph{all} runs in the constrained case, including those runs that did \emph{not} qualify as a success.
The waterfall curves in \cref{fig:waterfall_planar} display the success rate as a function of an increasing position error tolerance, demonstrating that the higher success rate of our algorithm is maintained for different accuracy requirements.

Mean runtimes of both \texttt{RTR} and \texttt{RTR-B} across all planar experiments remain below 0.1 s, with the 30-DOF tree-like robot unsurprisingly resulting in the most computationally-intensive problems.
The runtime for \texttt{FABRIK} across all planar experiments was 3.4 s, while the local solvers had a mean runtime of 10.2 s.
Both of these sets of runtime statistics are influenced by the large proportion of unsuccessful runs that required all 2,000 allowed iterations before terminating.
However, since we cannot guarantee that the local algorithms were provided with equally well-tuned implementations of core subroutines (e.g., Hessian computations for the second-order trust region solvers), we urge our readers to treat the statistics on iterations reported in \Cref{tab:planar_chain_results} and \Cref{tab:planar_tree_results} as more qualitative indicators of performance.

\begin{figure*}[t]
	\centering
	\includegraphics[width=\textwidth]{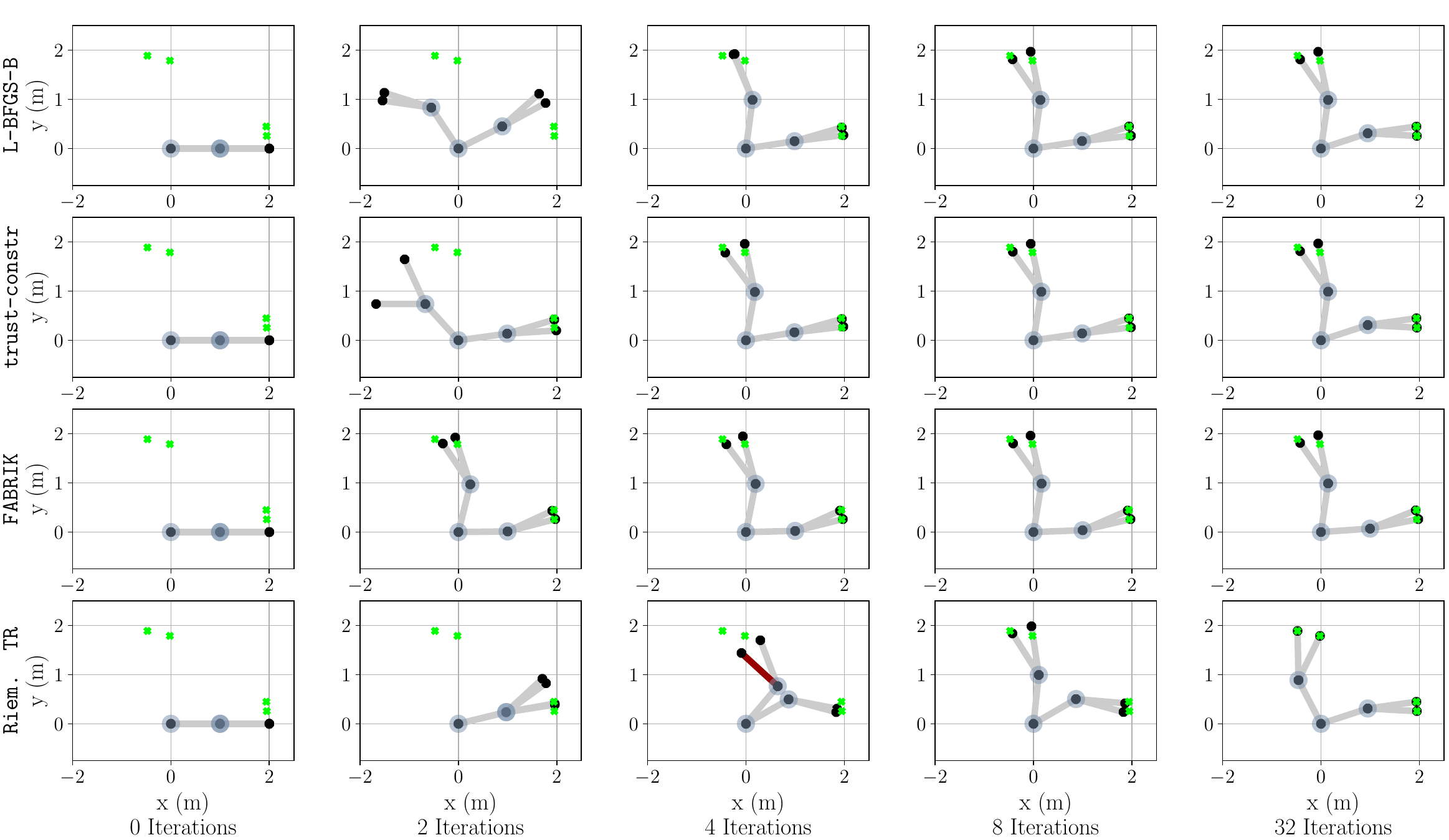}
	\caption{Convergence of various algorithms on a 6-DOF binary tree robots with joint angle limits. \texttt{FABRIK} and the angle-parametrized \texttt{L-BFGS-B} and \texttt{trust-constr} all converge to a local minimum, whereas \texttt{Riem.\ TR} is able to converge to the the global minimum from the same initial condition ($\Vector{\Theta} = \Vector{0}$).
		Note that \texttt{FABRIK} quickly converges but is unable to accurately reach any of the four end-effector targets.
		The red link in the \texttt{RTR} solution after four iterations indicates that the link's parent joint is violating its joint angle limits.
	}
	\label{fig:convergence}
\end{figure*}
\begin{figure*}
	\centering
	\begin{subfigure}[t]{0.49\textwidth}
		\centering
		\includegraphics[width=\textwidth]{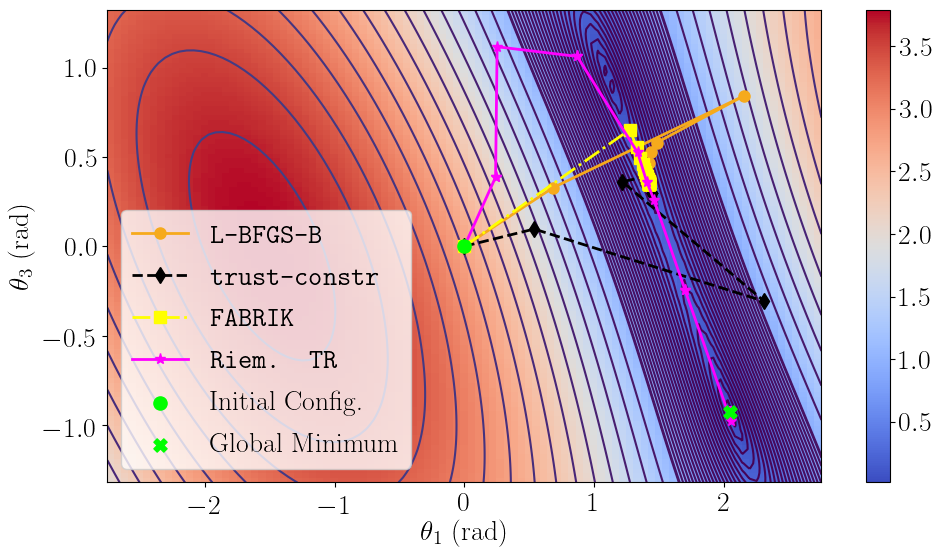}
		\caption{L$_2$ error for the end-effector actuated by $\theta_1$ and $\theta_3$.}
		\label{fig:heatmap_1_3}
	\end{subfigure}
	~
	\begin{subfigure}[t]{0.49\textwidth}
		\centering
		\includegraphics[width=\textwidth]{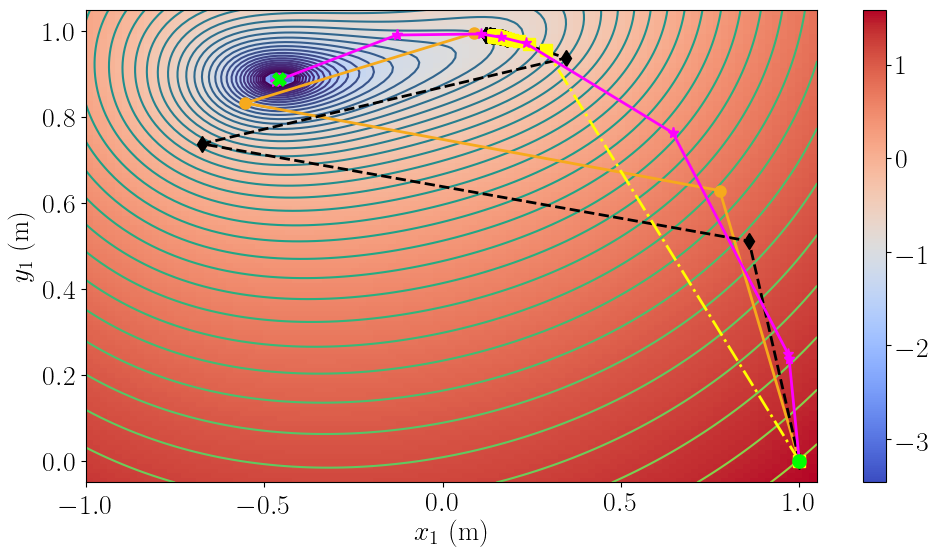}
		\caption{Base-10 logarithm of distance-geometric cost involving the joint position which is actuated by $\theta_1$.}
		\label{fig:heatmap_pos_1_3}
	\end{subfigure}
	\caption{Contour maps of different cost function components overlaid with solver trajectories for the IK problem instance depicted in \cref{fig:convergence}.
		The angle $\theta_1$ is the angle of the joint connected to the root and pointing towards the top left of corner of the plot of \texttt{Riem.\ TR} after 32 iterations in \cref{fig:convergence}.
		Angles $\theta_3$ and $\theta_4$ are the angles of the two child links of the link actuated by $\theta_1$.
		In \cref{fig:heatmap_pos_1_3}, $x_1 = \cos{\theta_1}$ and $y_1 = \sin{\theta_1}$ are the coordinates of the point actuated by $\theta_1$, and the cost function is the quartic terms of \cref{eq:EDM_completion} involving $x_1$ and $y_1$.
		This distance-geometric cost function is very well posed, and the Riemannian solver follows its contours to the global minimum.
		The other three methods attempt to minimize the ill-conditioned cost function in \cref{fig:heatmap_1_3} and return a suboptimal solution.
	}
	\label{fig:heatmaps}
\end{figure*}
To illustrate the optimization procedure and help elucidate the relatively superior performance of our method on branching tree-like robots, we conducted a simple empirical analysis on one of the many problem instances where \texttt{RTR} outperformed the benchmark algorithms.
\cref{fig:convergence} shows the convergence of four solvers with the same initial condition on a sample low-dimensional IK problem involving a 6-DOF binary planar tree robot with symmetric joint angle limits and end-effector position goals.
Only our algorithm, \texttt{RTR}, is able to find the global minimum.
The algorithms all perform similarly for the first eight iterations, but the three competitors are unable to escape from the same local minimum.
The difference in behaviour is explained by \cref{fig:heatmaps}, which compares contour maps of different cost function terms used by the algorithms, overlaid with the progress of each algorithm across iterations.
\cref{fig:heatmap_1_3} illustrates the Euclidean distance of the end-effector controlled by $\theta_1$ and $\theta_3$ and its goal position, which is used in the cost function of \texttt{L-BFGS-B} and \texttt{trust-constr}.
In contrast, the contour map in \cref{fig:heatmap_pos_1_3} is the logarithm of the quartic function containing the terms in the distance-geometric cost function of \cref{eq:EDM_completion} involving the position of the joint actuated by $\theta_1$.
The angular cost function is ill-conditioned, leading the algorithms that minimize it to converge to the local minimum of \cref{fig:convergence}, whereas \texttt{RTR} minimizes the distance-geometric cost of \cref{fig:heatmap_pos_1_3} and quickly converges to the global minimum, which has a large and well-conditioned basin of convergence.
In spite of its simplicity, this toy problem illustrates the behavioural differences of the algorithms in a state space with low enough dimension to visualize clearly.

\subsection{Revolute Manipulators}\label{subsec:rev_exp_1}
\begin{figure*}
	\centering
	\begin{subfigure}{0.22\textwidth}
		\centering
		\includegraphics[width = 0.8\textwidth]{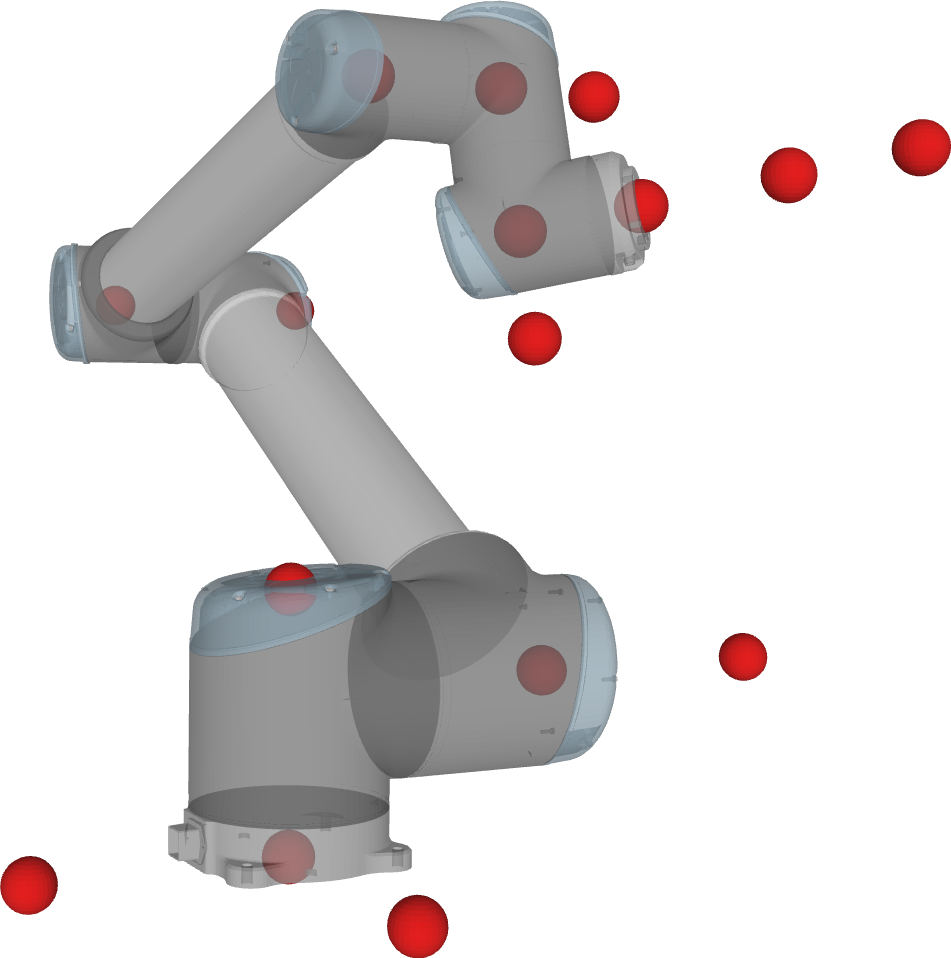}
		\caption{UR10}\label{fig:ur10_urdf}
	\end{subfigure}
	\begin{subfigure}{0.22\textwidth}
		\centering
		\includegraphics[width = 0.8\textwidth]{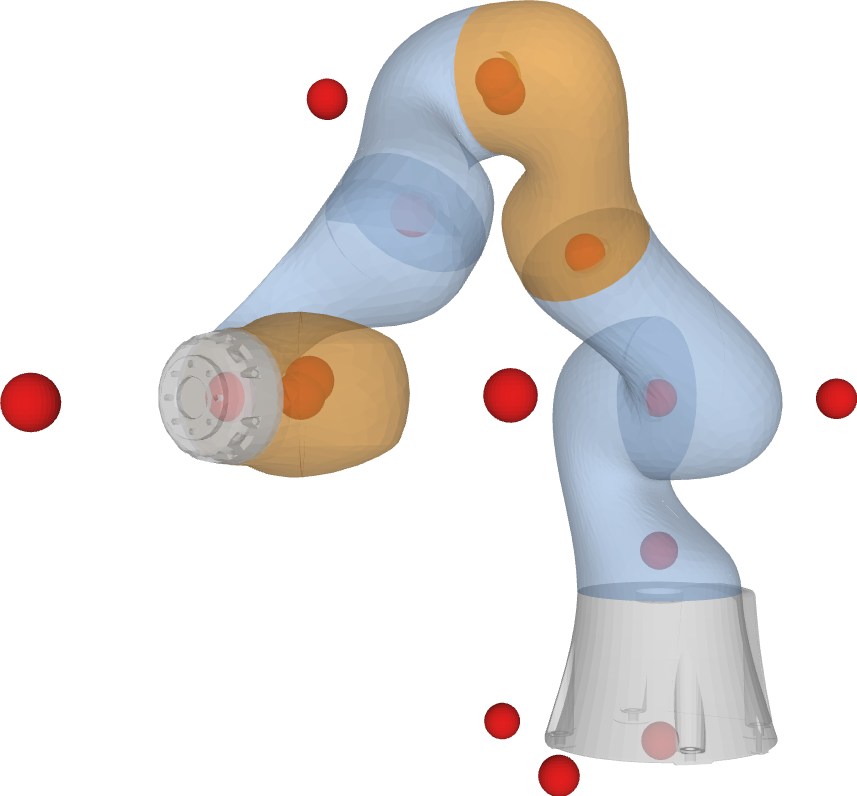}
		\caption{KUKA-IIWA}\label{fig:kuka_urdf}
	\end{subfigure}
	\begin{subfigure}{0.22\textwidth}
		\centering
		\includegraphics[width = \textwidth]{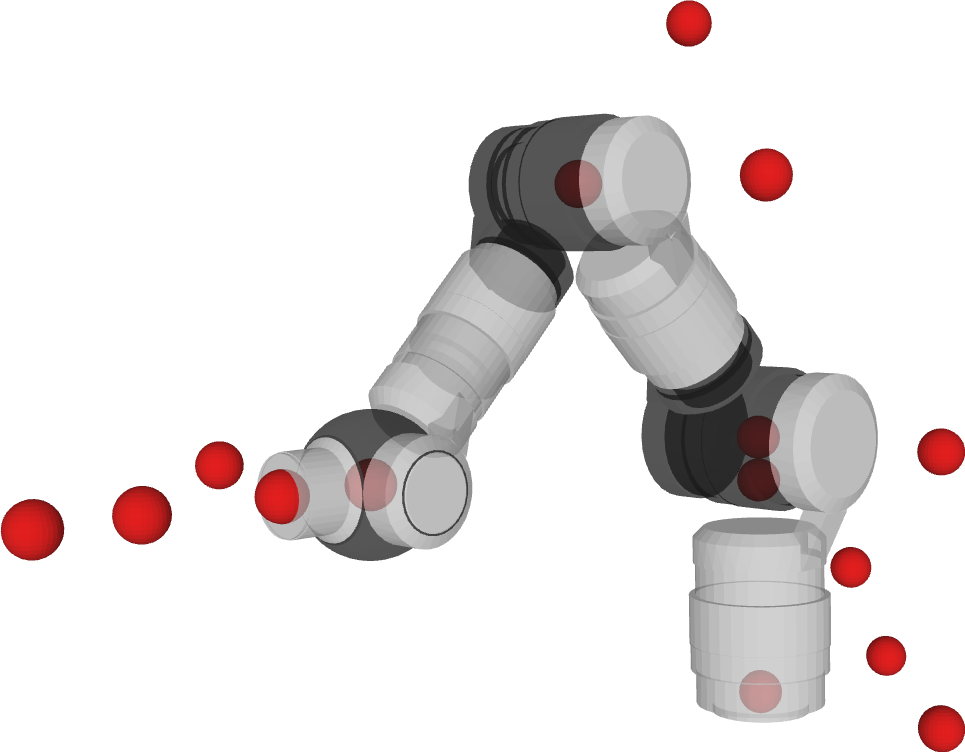}
		\caption{Schunk-LWA4D}\label{fig:lwa4d_urdf}
	\end{subfigure}
	\begin{subfigure}{0.22\textwidth}
		\centering
		\includegraphics[width = 0.7\textwidth]{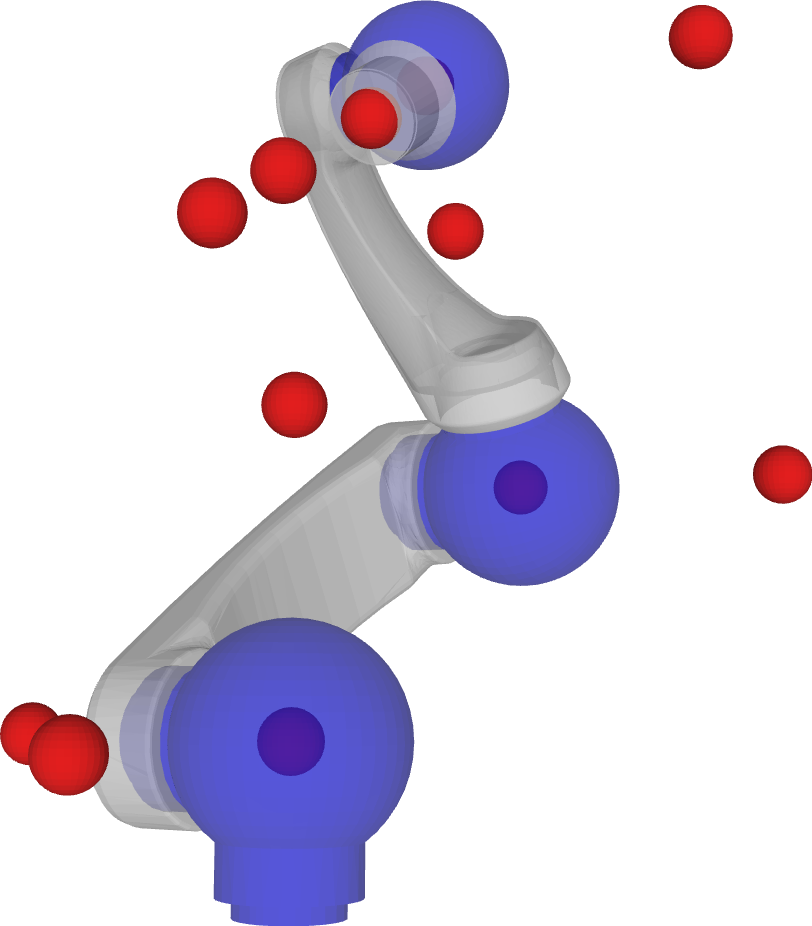}
		\caption{Schunk-LWA4P}\label{fig:lwa4p_urdf}
	\end{subfigure}
	\caption{Points in the distance-based models of the commercial manipulators used in our experiments. Note that the distances between pairs of points on individual rotation axes have been reduced for clarity.}
	\label{fig:urdf_big}
\end{figure*}
Next, \Cref{tab:revolute_chain_results} compares the performance of our algorithm and the trust region benchmark algorithms on 3D robots with revolute joints and within an unconstrained workspace.
These problems are of greater practical interest than the planar results in \cref{sec:planar_results} and showcase the expressiveness of our problem formulation.
All experiments are conducted for the Universal Robots UR10, KUKA IIWA, Schunk LWA4D, and Schunk LWA4P manipulators shown in \cref{fig:urdf_big}.
For these robots, the distance-geometric IK formulation derived using the procedure described in~\Cref{sec:kinematic_models} results in points that always overlap (i.e., have a fixed distance of zero).
While these points could be ``merged'' to reduce the graph size---thereby improving overall performance---we used the generic models for transparency.

In terms of success rate, our algorithm outperforms \texttt{trust-exact} and \texttt{trust-constr} on the UR10 and KUKA IIWA manipulators with and without joint angle limits, as well as on the Schunk LWA4P with joint limits.
The bound smoothing procedure used to initialize \texttt{RTR-B} reduces the overall number of iterations in all cases, but has a variable effect on the success rate.
Both \texttt{RTR} and \texttt{RTR-B} require a significantly larger number of iterations to converge compared to the planar case, while \texttt{trust-*} remains in a similar range to that observed in~\Cref{tab:planar_chain_results}.
We can partially attribute this to the iteration complexity discussed in~\cref{sec:algorithm}, which is increased by the unfavourably high ratio of points to DOF in the kinematic models of these mechanisms.
Again, this effect could be mitigated in future work by removing overlapping points from the kinematic models, reducing the overall number of variables.

The box-and-whiskers plots in \cref{fig:boxplots_ur10} summarize the position error statistics for each algorithm over \emph{all} runs with the UR10 manipulator, with and without joint limits.
In both cases the \texttt{trust-*} algorithms converge to significantly lower cost function values.
This suggests that the highly variable magnitudes of known distances may cause numerical issues in the gradient for our algorithms, causing early termination.
We suspect this issue can be avoided by using a weighting matrix to regularize elements of the cost function, as shown in~\cite{nguyenLocalizationIoTNetworks2019}.
The waterfall curves in \cref{fig:waterfall_ur10} corroborate these findings, showing that the success rate of our algorithm drops as the position error tolerance decreases, and suggesting that decreasing the gradient termination criteria may increase accuracy at the cost of increased computation time.

The mean runtime of both \texttt{RTR} and \texttt{RTR-B} remains below 1.0 s for all  robot models, reaching the lowest mean runtimes of less than 0.2 s on the UR10.
In the same instance, the \texttt{trust-*} algorithms have a slightly higher mean runtime of 0.3 s.
We observe the highest computation times for our algorithms with the 7-DOF Schunk LWA4D, with mean runtimes slightly below 1.0 s.
While the \texttt{trust-*} methods exhibit similarly worse performance with a mean runtime of 0.5 s, the overall increase in runtime is smaller due to the number of variables only increasing from six to seven.
\begin{table*}
	\centering
	\begin{adjustbox}{width=\linewidth}
		\input{tables/revolute}

	\end{adjustbox}
	\caption{Results for revolute chain manipulators over 2,000 random experiments with pose goals. The $+$ indicates joint angle limits.}
	\label{tab:revolute_chain_results}
\end{table*}

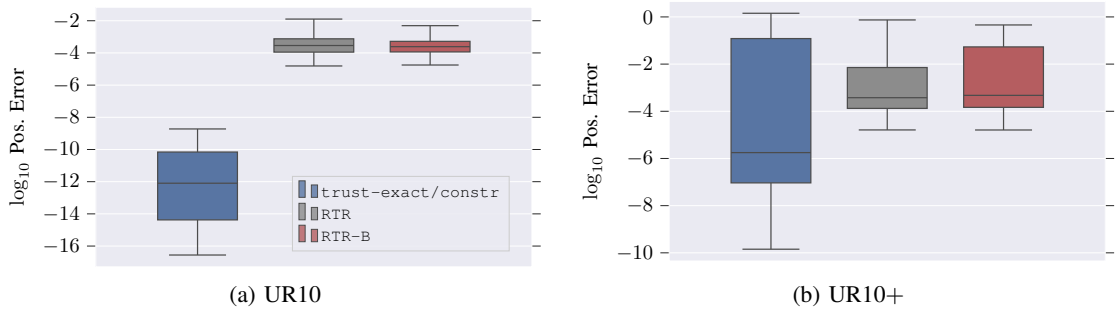
\begin{figure*}
	\centering
	\begin{subfigure}[t]{0.4\textwidth}
		\centering
		\begin{adjustbox}{width=\linewidth}
			\input{figs/results/boxplot_ur10_bounded_False.tex}
		\end{adjustbox}
		\caption{UR10}
		\label{fig:boxplot_ur10}
	\end{subfigure}
	~
	\begin{subfigure}[t]{0.4\textwidth}
		\centering
		\begin{adjustbox}{width=\linewidth}
			\input{figs/results/boxplot_ur10_bounded_True.tex}
		\end{adjustbox}
		\caption{UR10$+$}
		\label{fig:boxplot_ur10_bounded}
	\end{subfigure}
	\caption{Box-and-whiskers plots summarizing end-effector position error over 2,000 experiments with the UR10 manipulator, (a) without and (b) with joint angle limits.}
	\label{fig:boxplots_ur10}
\end{figure*}

\begin{figure*}
	\centering
	\begin{subfigure}[t]{0.4\textwidth}
		\centering
		\begin{adjustbox}{width=\linewidth}
			\input{figs/results/waterfall_ur10_bounded_False.tex}
		\end{adjustbox}
		\caption{UR10}
		\label{fig:waterfall_ur10}
	\end{subfigure}
	~
	\begin{subfigure}[t]{0.4\textwidth}
		\centering
		\begin{adjustbox}{width=\linewidth}
			\input{figs/results/waterfall_ur10_bounded_True.tex}
		\end{adjustbox}
		\caption{UR10$+$}
		\label{fig:waterfall_ur10_bounded}
	\end{subfigure}
	\caption{Waterfall curves of success rate versus position error tolerance for 2,000 experiments with the UR10 manipulator, without (a) and with (b) joint angle limits. The shaded regions are 95$\%$ Jeffreys confidence intervals centered on the solid lines. The rotation error tolerance is fixed at 0.01 rad.}
	\label{fig:waterfall_ur10}
\end{figure*}
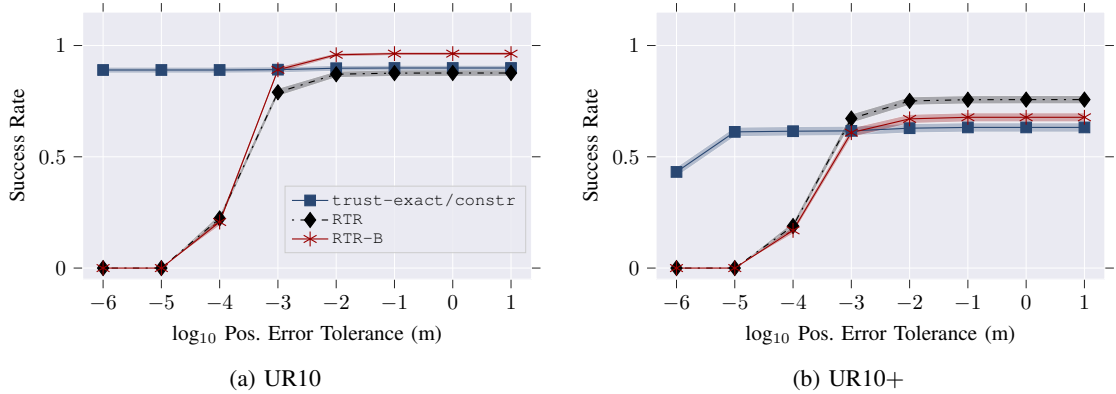

\subsection{Obstacle Avoidance}\label{subsec:revolute_results_unc}
\begin{figure*}
	\centering
	\begin{subfigure}{0.31\textwidth}
		\centering
		\includegraphics[width = \textwidth]{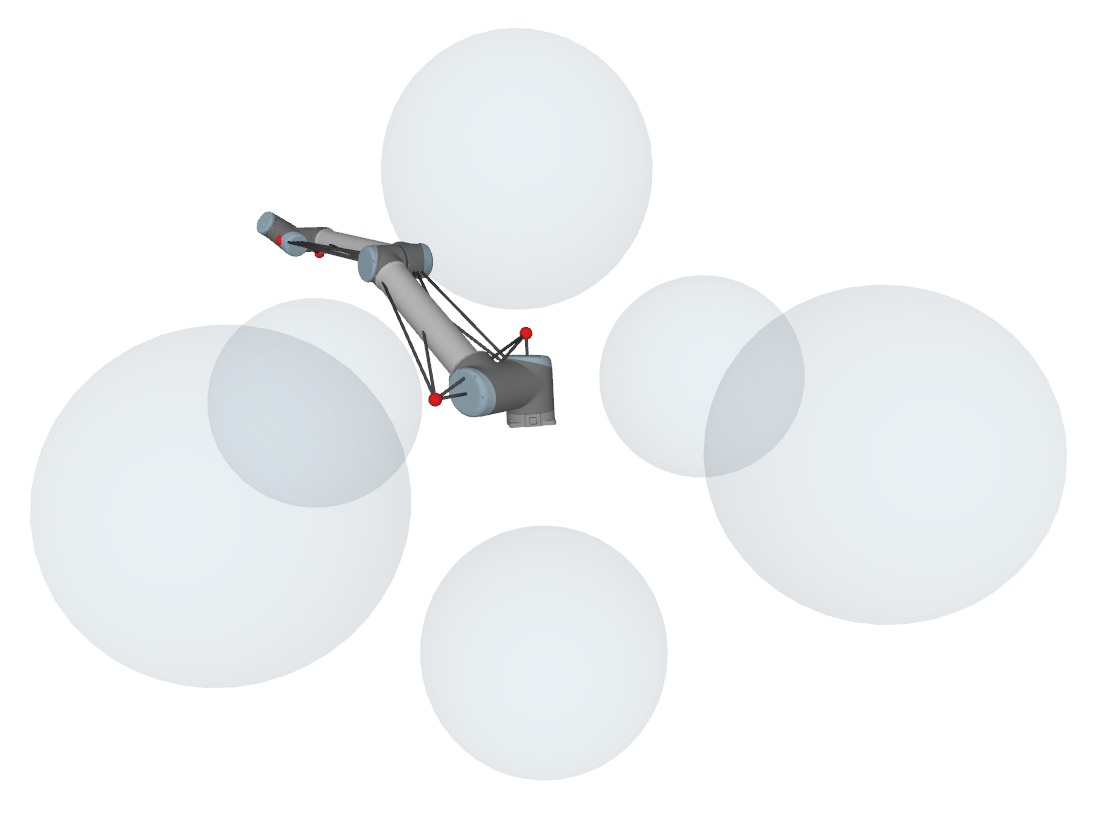}
		\caption{Universal Robotics UR10 (\textit{octahedron})}
	\end{subfigure}
	\begin{subfigure}{0.31\textwidth}
		\centering
		\includegraphics[width = 0.7\textwidth]{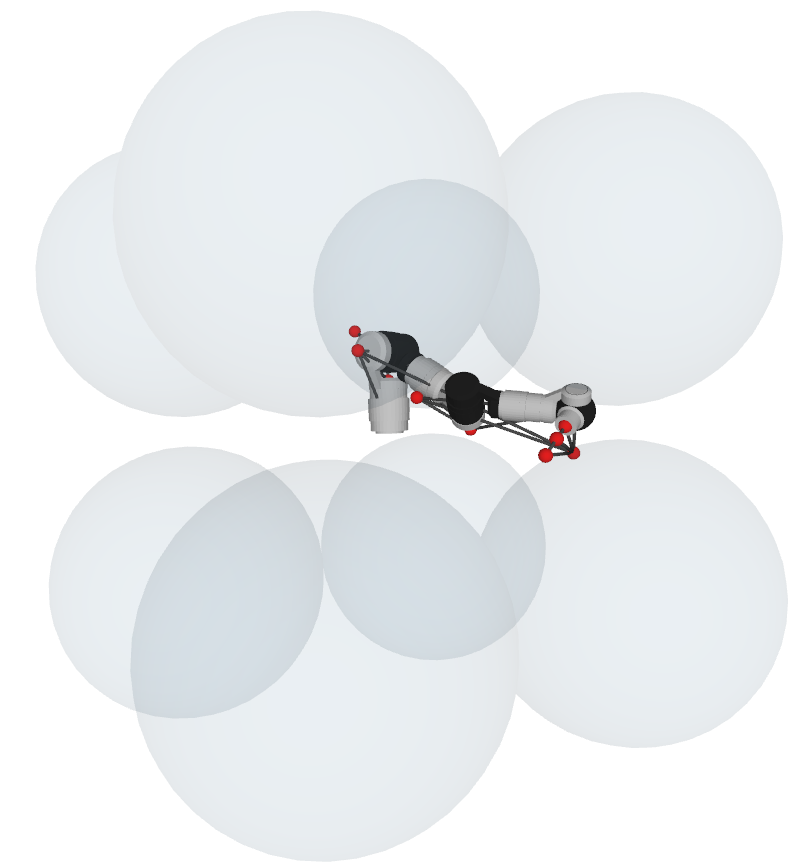}
		\caption{Schunk LWA4D (\textit{cube})}
	\end{subfigure}
	\begin{subfigure}{0.31\textwidth}
		\centering
		\includegraphics[width = 0.8\textwidth]{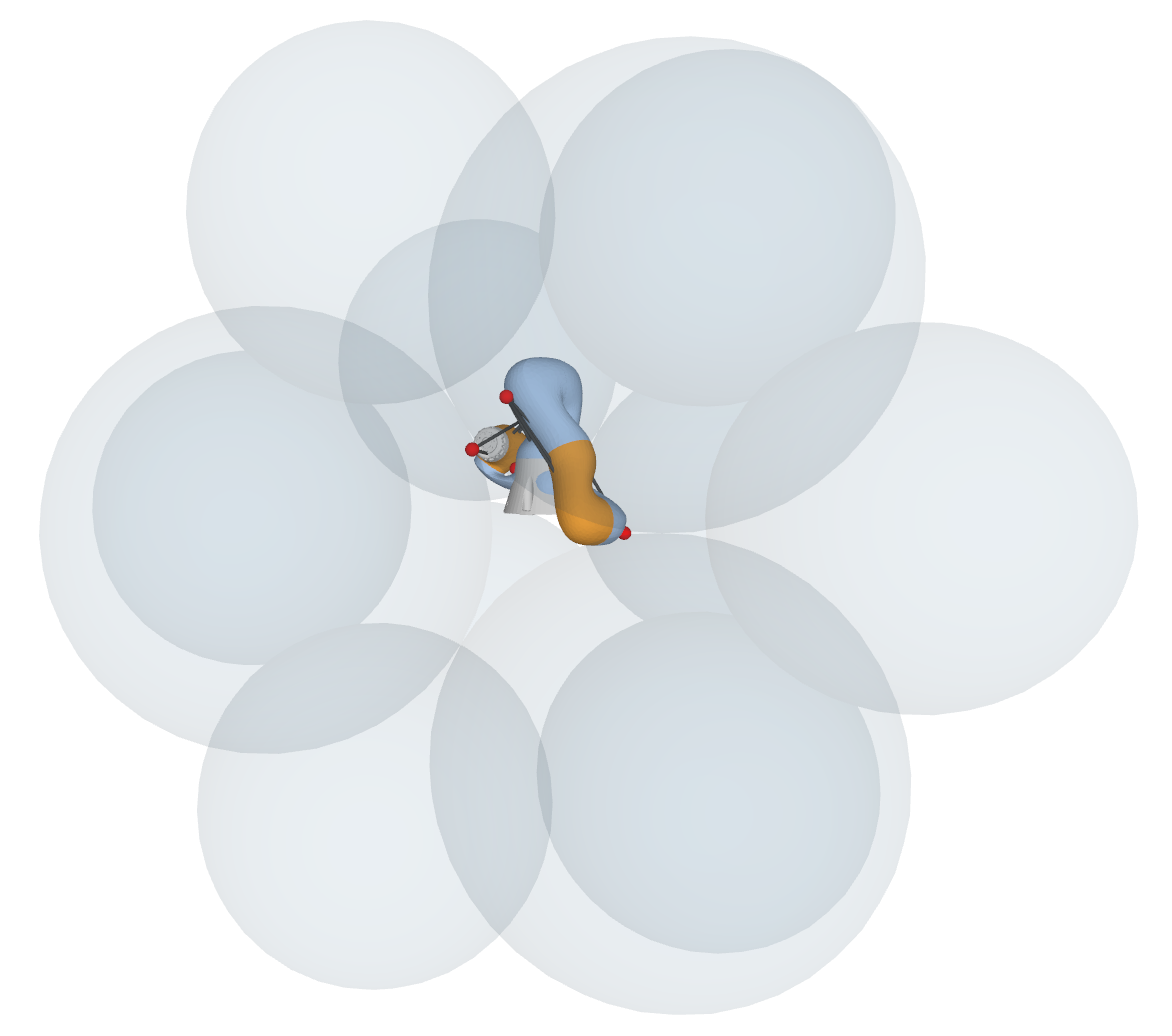}
		\caption{KUKA-IIWA (\textit{icosahedron})}
	\end{subfigure}
	\caption{A selection of robot and environment combinations used to generate IK problems in \Cref{subsec:revolute_results_unc}.}\label{fig:envs}
\end{figure*}
Finally, we analyze the performance of our algorithm on revolute manipulators in environments with a varying number of spherical obstacles, as shown in~\cref{fig:envs}.
For each environment and robot pair, 3,000 IK problems were randomly generated, some of which may not be solvable (i.e., no collision-free solutions exist).
For experiments performed in this section we chose $\Vector{e} = \log(\Matrix{T}(\Matrix{\Theta})^{-1}\Matrix{T}_{goal})$ as the error in \cref{prob:slsqp}, which is the vector of exponential coordinates describing the screw motion between the current pose and goal pose of the end-effector.
Further, the FK and Jacobian computations are carried out using the popular \textit{product of exponentials} approach~\cite{lynch2017modern}, avoiding trigonometric identities inherent to the DH parametrization.
We solve the problem in~\cref{prob:slsqp} using sequential quadratic programming, namely the \texttt{SLSQP} solver implemented in the \texttt{scipy} library.
Due to its speed and accuracy, this method is a popular choice for nonlinear programming solutions to IK~\cite{beeson2015trac}.
We empirically determined that a maximum of 200 iterations and an objective function value of $10^{-7}$ were effective stopping criteria for our experiments.

The results of our evaluation are shown in \cref{fig:exp1}.
Each column represents a different fixed set of obstacles in the manipulators' environment.
The first column contains results for obstacle-free problems, while the remaining columns use spherical obstacles placed on the vertices of an octahedron, cube, and icosahedron, respectively.
The top row compares the success rate of \texttt{RTR-B} and \texttt{SLSQP}, with dashed lines indicating the portion of problems that are known to be feasible (i.e., the configuration used to generate the problem is not in collision).
The box-and-whiskers plots in the middle two rows describe the distribution of position and rotation errors, with the thresholds for success (0.01 m and 0.01 rad) drawn as dashed lines.
Finally, the bottom row compares the distribution of solution times.

In terms of success rate, our algorithm outperforms \texttt{SLSQP} in all experiments.
The position and rotation error distributions displayed in the box-and-whiskers plots reveal that this is due to \texttt{RTR-B} providing solutions with lower position and rotation error on average.
The performance gap is particularly large for the UR10, which both algorithms struggled most with in all environments.
Finally, the runtime for our Riemannian solver is expectedly higher than for \texttt{SLSQP} in the obstacle-free case, as seen in previous experiments.
However, the addition of obstacles leads to significantly faster relative performance for \texttt{RTR-B} on all the manipulators tested.
More importantly, we note that the runtime of \texttt{SLSQP} exhibits a more significant relative increase than \texttt{RTR-B} when obstacles are introduced.
We suspect that this is due to the nonlinear mapping between joint angles and obstacle locations that makes the problem significantly more difficult to solve in configuration space.
In contrast, the effect is avoided by our distance-based approach because collision avoidance constraints are treated in the same way as structural and joint limit constrains.

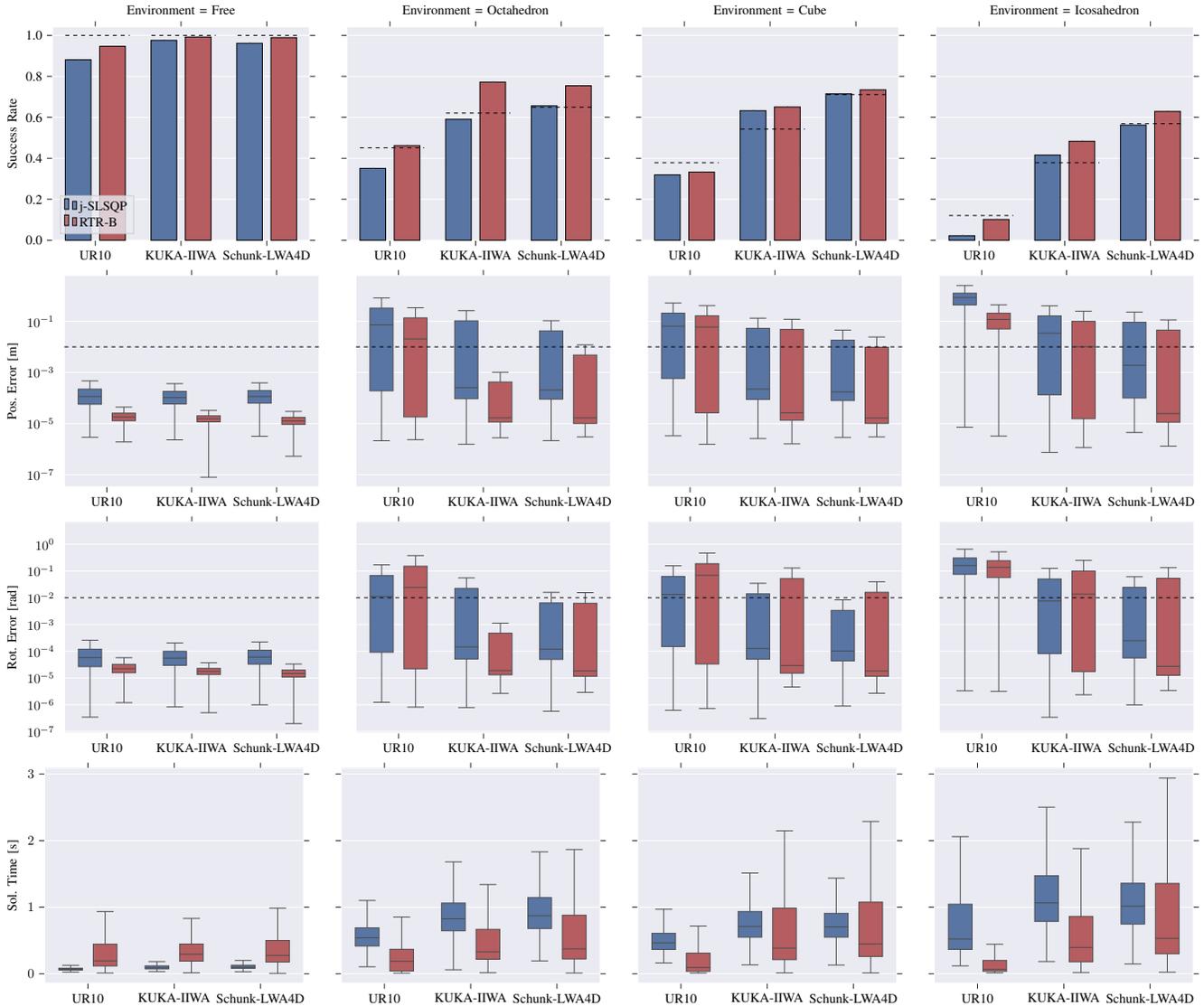
\begin{figure*}
	\captionsetup[subfloat]{position=top,labelformat=empty}
	\centering
	\resizebox{\textwidth}{!}{
		\begin{tabular}{@{}c@{}}
			\subfloat[]{
				\centering
				\begin{adjustbox}{width=\linewidth}
					\input{figs/results/exp1_srate.tex}
				\end{adjustbox}
			} \\[-3ex]
			\subfloat[]{
				\centering
				\begin{adjustbox}{width=\linewidth}
					\input{figs/results/exp1_poserr.tex}
				\end{adjustbox}
			} \\[-3ex]
			\subfloat[]{
				\centering
				\begin{adjustbox}{width=\linewidth}
					\input{figs/results/exp1_roterr.tex}
				\end{adjustbox}
			} \\[-3ex]
			\subfloat[]{
				\centering
				\begin{adjustbox}{width=\linewidth}
					\input{figs/results/exp1_soltime.tex}
				\end{adjustbox}
			}
		\end{tabular}
	}
	\caption{Experimental results comparing our method with local optimization. Each column contains the success rates, position errors, rotation errors, and solution times for 3,000 randomly generated problems in the titular environment. The success rate in the top row is measured relative to the total number of generated problems, many of which are infeasible. The upper bound defining success in each box-and-whiskers plot is shown as a dashed line. The dashed lines in the top row indicate the lower bound on the number of feasible problems for each robot-environment pair.}\label{fig:exp1}
\end{figure*}


%% file: tables/planar_chain.tex
\begin{tabular}{lrcrcrccrcc}
    \toprule
    Method & \multicolumn{2}{c}{\texttt{trust-exact/constr}} & \multicolumn{2}{c}{\texttt{FABRIK}} & \multicolumn{3}{c}{\texttt{RTR}} & \multicolumn{3}{c}{\texttt{RTR-B}} \\
    \cmidrule(r){2-3}  \cmidrule(r){4-5}   \cmidrule(r){6-8} \cmidrule(r){9-11}
     &Success ($\%$) &Iter. $\mu$ ($\sigma$) &Success ($\%$) &Iter. $\mu$ ($\sigma$) &Success ($\%$) &Iter. $\mu$ ($\sigma$) &Runtime [ms] $\mu$ ($\sigma$) &Success ($\%$) &Iter. $\mu$ ($\sigma$) &Runtime [ms] $\mu$ ($\sigma$)\\
    \midrule
    \midrule
    6-DOF & $100.0 \pm 0.1$ & 26 (5) & $100.0 \pm 0.1$ & 10 (19) & $100.0 \pm 0.1$ & 16 (2) & 2.28 (0.39) & $100.0 \pm 0.1$ & 9 (2) & 1.54 (0.37) \\
    6-DOF$+$ & $72.0 \pm 2.8$ & 35 (14) & $36.0 \pm 3.0$ & 1333 (916) & $91.0 \pm 1.7$ & 32 (22) & 5.24 (2.83) & $\mathbf{98.0 \pm 0.9}$ & 24 (17) & 5.02 (3.01) \\
    10-DOF & $100.0 \pm 0.1$ & 28 (2) & $100.0 \pm 0.1$ & 9 (13) & $100.0 \pm 0.1$ & 20 (2) & 3.27 (0.53) & $100.0 \pm 0.1$ & 10 (2) & 2.25 (0.43) \\
    10-DOF$+$ & $97.0 \pm 1.0$ & 43 (54) & $59.0 \pm 3.0$ & 856 (971) & $88.0 \pm 2.0$ & 74 (101) & 14.6 (1.15) & $\mathbf{98.0 \pm 0.8}$ & 23 (46) & 5.47 (3.85) \\
\end{tabular}

%% file: tables/planar_tree.tex
\begin{tabular}{lrcrcrccrcc}
    \toprule
    Method & \multicolumn{2}{c}{\texttt{trust-exact/constr}} & \multicolumn{2}{c}{\texttt{FABRIK}} & \multicolumn{3}{c}{\texttt{RTR}} & \multicolumn{3}{c}{\texttt{RTR-B}} \\
    \cmidrule(r){2-3}  \cmidrule(r){4-5}   \cmidrule(r){6-8} \cmidrule(r){9-11}
     &Success ($\%$) &Iter. $\mu$ ($\sigma$) &Success ($\%$) &Iter. $\mu$ ($\sigma$) &Success ($\%$) &Iter. $\mu$ ($\sigma$) &Runtime [ms] $\mu$ ($\sigma$) &Success ($\%$) &Iter. $\mu$ ($\sigma$) &Runtime [ms] $\mu$ ($\sigma$)\\
    \midrule
    \midrule
    14-DOF & $54.0 \pm 3.1$ & 14 (4) & $56.0 \pm 3.1$ & 949 (964) & $64.0 \pm 3.0$ & 20 (4) & 4.39 (1.00) & $\mathbf{67.0 \pm 2.9}$ &  9 (3) & 4.12 (1.11) \\
    14-DOF$+$ & $90.0 \pm 1.9$ & 189 (423) & $85.0 \pm 2.2$ & 463 (716) & $\mathbf{96.0 \pm 1.2}$ & 18 (4) & 4.98 (1.40) & $96.0 \pm 1.2$ & 8 (2) & 4.79 (1.76) \\
    30-DOF & $16.0 \pm 2.3$ & 23 (7) & $12.0 \pm 2.0$ & 1848 (479) & $\mathbf{20.0 \pm 2.5}$ & 34 (10) & 20.75 (8.52) & $18.0 \pm 2.4$ & 18 (10) & 45.25 (32.38) \\
    30-DOF$+$ & $49.0 \pm 3.1$ & 818 (830) & $0.0 \pm 0.3$ & 2000 (0) & $\mathbf{85.0 \pm 2.2}$ & 36 (19) & 31.64 (17.95) & $85.0 \pm 2.2$ & 20 (15) & 43.98 (32.62) \\
\end{tabular}

%% file: figs/results/boxplot_planar_tree_dof_14_bounded_True.tex
\begin{tikzpicture}

	\definecolor{color0}{rgb}{0.917647058823529,0.917647058823529,0.949019607843137}
	\definecolor{color1}{rgb}{0.71078431372549,0.363725490196079,0.375490196078431}
	\definecolor{color2}{rgb}{0.347058823529412,0.458823529411765,0.641176470588235}
	\definecolor{color3}{rgb}{0.756862745098039,0.700980392156863,0.498039215686275}

	\begin{axis}[
			axis background/.style={fill=color0},
			axis line style={white},
			legend cell align={left},
			legend style={fill opacity=0.8, draw opacity=1, text opacity=1, draw=white!80!black, fill=color0, nodes={scale=0.8, transform shape}},
			tick align=outside,
			x grid style={white},
			xmajorticks=false,
			xmin=-0.5, xmax=0.5,
			xtick style={color=white!15!black},
			xtick={0},
			xticklabels={14},
			y grid style={white},
			ylabel={$\log_{10}$ Pos. Error},
			ymajorgrids,
			ymin=-15.2500930724719, ymax=2.05863736751692,
			ytick style={color=white!15!black},
			ytick={-16,-14,-12,-10,-8,-6,-4,-2, 0, 2, 4},
			width=9cm,
			height=6cm
			]
		\path [draw=white!29.8039215686275!black, fill=color2, semithick]
		(axis cs:-0.3686,-7.21757334361375)
		--(axis cs:-0.2314,-7.21757334361375)
		--(axis cs:-0.2314,-6.2005410032743)
		--(axis cs:-0.3686,-6.2005410032743)
		--(axis cs:-0.3686,-7.21757334361375)
		--cycle;
		\path [draw=white!29.8039215686275!black, fill=color3, semithick]
		(axis cs:-0.1686,-6.05873361834331)
		--(axis cs:-0.0314,-6.05873361834331)
		--(axis cs:-0.0314,-6.00345663867157)
		--(axis cs:-0.1686,-6.00345663867157)
		--(axis cs:-0.1686,-6.05873361834331)
		--cycle;
		\path [draw=white!29.8039215686275!black, fill=white!54.9019607843137!black, semithick]
		(axis cs:0.0314,-14.2182172349764)
		--(axis cs:0.1686,-14.2182172349764)
		--(axis cs:0.1686,-11.2437721232048)
		--(axis cs:0.0314,-11.2437721232048)
		--(axis cs:0.0314,-14.2182172349764)
		--cycle;
		\path [draw=white!29.8039215686275!black, fill=color1, semithick]
		(axis cs:0.2314,-14.2279428203707)
		--(axis cs:0.3686,-14.2279428203707)
		--(axis cs:0.3686,-11.3413387821919)
		--(axis cs:0.2314,-11.3413387821919)
		--(axis cs:0.2314,-14.2279428203707)
		--cycle;
		\draw[draw=white!29.8039215686275!black,fill=color2,line width=0.3pt] (axis cs:0,0) rectangle (axis cs:0,0);
		\addlegendimage{ybar,ybar legend,draw=white!29.8039215686275!black,fill=color2,line width=0.3pt}
		\addlegendentry{\texttt{trust-constr}}

		\draw[draw=white!29.8039215686275!black,fill=color3,line width=0.3pt] (axis cs:0,0) rectangle (axis cs:0,0);
		\addlegendimage{ybar,ybar legend,draw=white!29.8039215686275!black,fill=color3,line width=0.3pt}
		\addlegendentry{\texttt{FABRIK}}

		\draw[draw=white!29.8039215686275!black,fill=white!54.9019607843137!black,line width=0.3pt] (axis cs:0,0) rectangle (axis cs:0,0);
		\addlegendimage{ybar,ybar legend,draw=white!29.8039215686275!black,fill=white!54.9019607843137!black,line width=0.3pt}
		\addlegendentry{\texttt{RTR}}

		\draw[draw=white!29.8039215686275!black,fill=color1,line width=0.3pt] (axis cs:0,0) rectangle (axis cs:0,0);
		\addlegendimage{ybar,ybar legend,draw=white!29.8039215686275!black,fill=color1,line width=0.3pt}
		\addlegendentry{\texttt{RTR-B}}

		\addplot [semithick, white!29.8039215686275!black, forget plot]
		table {%
				-0.3 -7.21757334361375
				-0.3 -8.45847554225671
			};
		\addplot [semithick, white!29.8039215686275!black, forget plot]
		table {%
				-0.3 -6.2005410032743
				-0.3 -4.75800408879936
			};
		\addplot [semithick, white!29.8039215686275!black, forget plot]
		table {%
				-0.349 -8.45847554225671
				-0.251 -8.45847554225671
			};
		\addplot [semithick, white!29.8039215686275!black, forget plot]
		table {%
				-0.349 -4.75800408879936
				-0.251 -4.75800408879936
			};
		\addplot [semithick, white!29.8039215686275!black, forget plot]
		table {%
				-0.1 -6.05873361834331
				-0.1 -6.14045937793233
			};
		\addplot [semithick, white!29.8039215686275!black, forget plot]
		table {%
				-0.1 -6.00345663867157
				-0.1 -6.00001205034161
			};
		\addplot [semithick, white!29.8039215686275!black, forget plot]
		table {%
				-0.149 -6.14045937793233
				-0.051 -6.14045937793233
			};
		\addplot [semithick, white!29.8039215686275!black, forget plot]
		table {%
				-0.149 -6.00001205034161
				-0.051 -6.00001205034161
			};
		\addplot [semithick, white!29.8039215686275!black, forget plot]
		table {%
				0.1 -14.2182172349764
				0.1 -14.7504697875351
			};
		\addplot [semithick, white!29.8039215686275!black, forget plot]
		table {%
				0.1 -11.2437721232048
				0.1 -9.04336099540379
			};
		\addplot [semithick, white!29.8039215686275!black, forget plot]
		table {%
				0.051 -14.7504697875351
				0.149 -14.7504697875351
			};
		\addplot [semithick, white!29.8039215686275!black, forget plot]
		table {%
				0.051 -9.04336099540379
				0.149 -9.04336099540379
			};
		\addplot [semithick, white!29.8039215686275!black, forget plot]
		table {%
				0.3 -14.2279428203707
				0.3 -14.7122437698007
			};
		\addplot [semithick, white!29.8039215686275!black, forget plot]
		table {%
				0.3 -11.3413387821919
				0.3 -9.02380628249027
			};
		\addplot [semithick, white!29.8039215686275!black, forget plot]
		table {%
				0.251 -14.7122437698007
				0.349 -14.7122437698007
			};
		\addplot [semithick, white!29.8039215686275!black, forget plot]
		table {%
				0.251 -9.02380628249027
				0.349 -9.02380628249027
			};
		\addplot [semithick, white!29.8039215686275!black, forget plot]
		table {%
				-0.3686 -6.80603988877453
				-0.2314 -6.80603988877453
			};
		\addplot [semithick, white!29.8039215686275!black, forget plot]
		table {%
				-0.1686 -6.02136418529904
				-0.0314 -6.02136418529904
			};
		\addplot [semithick, white!29.8039215686275!black, forget plot]
		table {%
				0.0314 -13.5146126197731
				0.1686 -13.5146126197731
			};
		\addplot [semithick, white!29.8039215686275!black, forget plot]
		table {%
				0.2314 -13.582679753653
				0.3686 -13.582679753653
			};
	\end{axis}

\end{tikzpicture}

%% file: figs/results/boxplot_planar_tree_dof_30_bounded_True.tex
\begin{tikzpicture}

	\definecolor{color0}{rgb}{0.917647058823529,0.917647058823529,0.949019607843137}
	\definecolor{color1}{rgb}{0.71078431372549,0.363725490196079,0.375490196078431}
	\definecolor{color2}{rgb}{0.347058823529412,0.458823529411765,0.641176470588235}
	\definecolor{color3}{rgb}{0.756862745098039,0.700980392156863,0.498039215686275}

	\begin{axis}[
			axis background/.style={fill=color0},
			axis line style={white},
			legend style={nodes={scale=0.5, transform shape}},
			tick align=outside,
			x grid style={white},
			xmajorticks=false,
			xmin=-0.5, xmax=0.5,
			xtick style={color=white!15!black},
			xtick={0},
			xticklabels={30},
			y grid style={white},
			ylabel={$\log_{10}$ Pos. Error},
			ymajorgrids,
			ymin=-15.2500930724719, ymax=2.05863736751692,
			ytick style={color=white!15!black},
			ytick={-16,-14,-12,-10,-8,-6,-4,-2, 0, 2, 4},
			width=9cm,
			height=6cm
		]
		\path [draw=white!29.8039215686275!black, fill=color2, semithick]
		(axis cs:-0.3686,-6.73437285030697)
		--(axis cs:-0.2314,-6.73437285030697)
		--(axis cs:-0.2314,-0.42588367252396)
		--(axis cs:-0.3686,-0.42588367252396)
		--(axis cs:-0.3686,-6.73437285030697)
		--cycle;
		\path [draw=white!29.8039215686275!black, fill=color3, semithick]
		(axis cs:-0.1686,0.169323355720213)
		--(axis cs:-0.0314,0.169323355720213)
		--(axis cs:-0.0314,0.738287570438005)
		--(axis cs:-0.1686,0.738287570438005)
		--(axis cs:-0.1686,0.169323355720213)
		--cycle;
		\path [draw=white!29.8039215686275!black, fill=white!54.9019607843137!black, semithick]
		(axis cs:0.0314,-10.6762709279857)
		--(axis cs:0.1686,-10.6762709279857)
		--(axis cs:0.1686,-9.05068245106047)
		--(axis cs:0.0314,-9.05068245106047)
		--(axis cs:0.0314,-10.6762709279857)
		--cycle;
		\path [draw=white!29.8039215686275!black, fill=color1, semithick]
		(axis cs:0.2314,-10.6838024640081)
		--(axis cs:0.3686,-10.6838024640081)
		--(axis cs:0.3686,-9.06703781201418)
		--(axis cs:0.2314,-9.06703781201418)
		--(axis cs:0.2314,-10.6838024640081)
		--cycle;
		\draw[draw=white!29.8039215686275!black,fill=white!54.9019607843137!black,line width=0.3pt] (axis cs:0,0) rectangle (axis cs:0,0);

		\draw[draw=white!29.8039215686275!black,fill=color1,line width=0.3pt] (axis cs:0,0) rectangle (axis cs:0,0);

		\draw[draw=white!29.8039215686275!black,fill=color2,line width=0.3pt] (axis cs:0,0) rectangle (axis cs:0,0);

		\draw[draw=white!29.8039215686275!black,fill=color3,line width=0.3pt] (axis cs:0,0) rectangle (axis cs:0,0);

		\addplot [semithick, white!29.8039215686275!black, forget plot]
		table {%
				-0.3 -6.73437285030697
				-0.3 -7.88913203835553
			};
		\addplot [semithick, white!29.8039215686275!black, forget plot]
		table {%
				-0.3 -0.42588367252396
				-0.3 1.18582874825968
			};
		\addplot [semithick, white!29.8039215686275!black, forget plot]
		table {%
				-0.349 -7.88913203835553
				-0.251 -7.88913203835553
			};
		\addplot [semithick, white!29.8039215686275!black, forget plot]
		table {%
				-0.349 1.18582874825968
				-0.251 1.18582874825968
			};
		\addplot [semithick, white!29.8039215686275!black, forget plot]
		table {%
				-0.1 0.169323355720213
				-0.1 -0.664552951165415
			};
		\addplot [semithick, white!29.8039215686275!black, forget plot]
		table {%
				-0.1 0.738287570438005
				-0.1 1.33905408084789
			};
		\addplot [semithick, white!29.8039215686275!black, forget plot]
		table {%
				-0.149 -0.664552951165415
				-0.051 -0.664552951165415
			};
		\addplot [semithick, white!29.8039215686275!black, forget plot]
		table {%
				-0.149 1.33905408084789
				-0.051 1.33905408084789
			};
		\addplot [semithick, white!29.8039215686275!black, forget plot]
		table {%
				0.1 -10.6762709279857
				0.1 -13.0526116525326
			};
		\addplot [semithick, white!29.8039215686275!black, forget plot]
		table {%
				0.1 -9.05068245106047
				0.1 -7.78906483096653
			};
		\addplot [semithick, white!29.8039215686275!black, forget plot]
		table {%
				0.051 -13.0526116525326
				0.149 -13.0526116525326
			};
		\addplot [semithick, white!29.8039215686275!black, forget plot]
		table {%
				0.051 -7.78906483096653
				0.149 -7.78906483096653
			};
		\addplot [semithick, white!29.8039215686275!black, forget plot]
		table {%
				0.3 -10.6838024640081
				0.3 -13.0005071641432
			};
		\addplot [semithick, white!29.8039215686275!black, forget plot]
		table {%
				0.3 -9.06703781201418
				0.3 -7.67554588121192
			};
		\addplot [semithick, white!29.8039215686275!black, forget plot]
		table {%
				0.251 -13.0005071641432
				0.349 -13.0005071641432
			};
		\addplot [semithick, white!29.8039215686275!black, forget plot]
		table {%
				0.251 -7.67554588121192
				0.349 -7.67554588121192
			};
		\addplot [semithick, white!29.8039215686275!black, forget plot]
		table {%
				-0.3686 -6.05407164631351
				-0.2314 -6.05407164631351
			};
		\addplot [semithick, white!29.8039215686275!black, forget plot]
		table {%
				-0.1686 0.508130911340562
				-0.0314 0.508130911340562
			};
		\addplot [semithick, white!29.8039215686275!black, forget plot]
		table {%
				0.0314 -9.92639606149873
				0.1686 -9.92639606149873
			};
		\addplot [semithick, white!29.8039215686275!black, forget plot]
		table {%
				0.2314 -9.8701787247363
				0.3686 -9.8701787247363
			};
	\end{axis}

\end{tikzpicture}

%% file: figs/results/waterfall_planar_tree_dof_14_bounded_True.tex
\begin{tikzpicture}

	\definecolor{color0}{rgb}{0.917647058823529,0.917647058823529,0.949019607843137}
	\definecolor{color1}{rgb}{0.16078431372549,0.27843137254902,0.447058823529412}
	\definecolor{color2}{rgb}{0.964705882352941,0.666666666666667,0.109803921568627}

	\begin{axis}[
			axis background/.style={fill=color0},
			axis line style={white},
			legend cell align={left},
			legend style={
					fill opacity=0.8,
					draw opacity=1,
					text opacity=1,
					at={(0.03,0.03)},
					anchor=south west,
					draw=white!80!black,
					fill=color0,
					nodes={scale=0.8, transform shape}
				},
			tick align=outside,
			x grid style={white},
			xlabel={log\(\displaystyle _{10}\) Pos. Error Tolerance (m)},
			xmajorgrids,
			xmin=-6.35, xmax=1.35,
			xtick style={color=white!15!black},
			y grid style={white},
			ylabel={Success Rate},
			ymajorgrids,
			ymin=-0.05, ymax=1.15,
			ytick style={color=white!15!black},
			width=9cm,
			height=6cm
		]
		\path [draw=white, fill=color1, opacity=0.29]
		(axis cs:-6,0.818195218290493)
		--(axis cs:-6,0.768105846329158)
		--(axis cs:-5,0.845966641787302)
		--(axis cs:-4,0.863048782112033)
		--(axis cs:-3,0.869484522616579)
		--(axis cs:-2,0.877015892683888)
		--(axis cs:-1,0.877015892683888)
		--(axis cs:0,0.877015892683888)
		--(axis cs:1,0.877015892683888)
		--(axis cs:1,0.914690474563871)
		--(axis cs:1,0.914690474563871)
		--(axis cs:0,0.914690474563871)
		--(axis cs:-1,0.914690474563871)
		--(axis cs:-2,0.914690474563871)
		--(axis cs:-3,0.908262222035956)
		--(axis cs:-4,0.90273257977768)
		--(axis cs:-5,0.887907061013872)
		--(axis cs:-6,0.818195218290493)
		--cycle;

		\path [draw=white, fill=color2, opacity=0.29]
		(axis cs:-6,0.855138759440034)
		--(axis cs:-6,0.808937033458904)
		--(axis cs:-5,0.814202204630636)
		--(axis cs:-4,0.821586471601234)
		--(axis cs:-3,0.82687066656374)
		--(axis cs:-2,0.832163336368779)
		--(axis cs:-1,0.837464844317301)
		--(axis cs:0,0.837464844317301)
		--(axis cs:1,0.837464844317301)
		--(axis cs:1,0.880455041053051)
		--(axis cs:1,0.880455041053051)
		--(axis cs:0,0.880455041053051)
		--(axis cs:-1,0.880455041053051)
		--(axis cs:-2,0.875785416385717)
		--(axis cs:-3,0.871106955792915)
		--(axis cs:-4,0.866420022359904)
		--(axis cs:-5,0.859844712585662)
		--(axis cs:-6,0.855138759440034)
		--cycle;

		\path [draw=white, fill=black, opacity=0.29]
		(axis cs:-6,0.966559652647457)
		--(axis cs:-6,0.940813159043345)
		--(axis cs:-5,0.940813159043345)
		--(axis cs:-4,0.940813159043345)
		--(axis cs:-3,0.940813159043345)
		--(axis cs:-2,0.949932698402058)
		--(axis cs:-1,0.967478399145281)
		--(axis cs:0,0.976000024847965)
		--(axis cs:1,0.976000024847965)
		--(axis cs:1,0.991205352794808)
		--(axis cs:1,0.991205352794808)
		--(axis cs:0,0.991205352794808)
		--(axis cs:-1,0.985764599624488)
		--(axis cs:-2,0.97339452113149)
		--(axis cs:-3,0.966559652647457)
		--(axis cs:-4,0.966559652647457)
		--(axis cs:-5,0.966559652647457)
		--(axis cs:-6,0.966559652647457)
		--cycle;

		\path [draw=white, fill=red!60!black, opacity=0.29]
		(axis cs:-6,0.969992807564852)
		--(axis cs:-6,0.945357176480583)
		--(axis cs:-5,0.945357176480583)
		--(axis cs:-4,0.945357176480583)
		--(axis cs:-3,0.945357176480583)
		--(axis cs:-2,0.949932698402058)
		--(axis cs:-1,0.971093535357963)
		--(axis cs:0,0.974762571884519)
		--(axis cs:1,0.974762571884519)
		--(axis cs:1,0.990448029510171)
		--(axis cs:1,0.990448029510171)
		--(axis cs:0,0.990448029510171)
		--(axis cs:-1,0.988133085872606)
		--(axis cs:-2,0.97339452113149)
		--(axis cs:-3,0.969992807564852)
		--(axis cs:-4,0.969992807564852)
		--(axis cs:-5,0.969992807564852)
		--(axis cs:-6,0.969992807564852)
		--cycle;

		\addplot [line width=0.52pt, color1, opacity=1.0, mark=square*, mark size=2.5, mark options={solid}]
		table {%
				-6 0.794
				-5 0.868
				-4 0.884
				-3 0.89
				-2 0.897
				-1 0.897
				0 0.897
				1 0.897
			};
		\addlegendentry{\texttt{trust-constr}}
		\addplot [line width=0.52pt, color2, opacity=1.0, dashed, mark=*, mark size=2.5, mark options={solid}]
		table {%
				-6 0.833
				-5 0.838
				-4 0.845
				-3 0.85
				-2 0.855
				-1 0.86
				0 0.86
				1 0.86
			};
		\addlegendentry{\texttt{FABRIK}}
		\addplot [line width=0.52pt, black, opacity=1.0, dash pattern=on 1pt off 3pt on 3pt off 3pt, mark=diamond*, mark size=3.5, mark options={solid}]
		table {%
				-6 0.955
				-5 0.955
				-4 0.955
				-3 0.955
				-2 0.963
				-1 0.978
				0 0.985
				1 0.985
			};
		\addlegendentry{\texttt{RTR}}
		\addplot [line width=0.52pt, red!60!black, opacity=1.0, mark=asterisk, mark size=3.5, mark options={solid}]
		table {%
				-6 0.959
				-5 0.959
				-4 0.959
				-3 0.959
				-2 0.963
				-1 0.981
				0 0.984
				1 0.984
			};
		\addlegendentry{\texttt{RTR-B}}
	\end{axis}

\end{tikzpicture}

%% file: figs/results/waterfall_planar_tree_dof_30_bounded_True.tex
\begin{tikzpicture}

	\definecolor{color0}{rgb}{0.917647058823529,0.917647058823529,0.949019607843137}
	\definecolor{color1}{rgb}{0.16078431372549,0.27843137254902,0.447058823529412}
	\definecolor{color2}{rgb}{0.964705882352941,0.666666666666667,0.109803921568627}

	\begin{axis}[
			axis background/.style={fill=color0},
			axis line style={white},
			tick align=outside,
			x grid style={white},
			xlabel={log\(\displaystyle _{10}\) Pos. Error Tolerance (m)},
			xmajorgrids,
			xmin=-6.35, xmax=1.35,
			xtick style={color=white!15!black},
			y grid style={white},
			ymajorgrids,
			ylabel={Success Rate},
			ymin=-0.05, ymax=1.15,
			ytick style={color=white!15!black},
            width=9cm,
            height=6cm
		]
		\path [draw=white, fill=color1, opacity=0.29]
		(axis cs:-6,0.415455872559329)
		--(axis cs:-6,0.355208736100581)
		--(axis cs:-5,0.413404023508557)
		--(axis cs:-4,0.427278327824702)
		--(axis cs:-3,0.445152695113199)
		--(axis cs:-2,0.456095727979951)
		--(axis cs:-1,0.457091292259424)
		--(axis cs:0,0.457091292259424)
		--(axis cs:1,0.457091292259424)
		--(axis cs:1,0.51897805893106)
		--(axis cs:1,0.51897805893106)
		--(axis cs:0,0.51897805893106)
		--(axis cs:-1,0.51897805893106)
		--(axis cs:-2,0.517979402474992)
		--(axis cs:-3,0.506986007222521)
		--(axis cs:-4,0.488964401043636)
		--(axis cs:-5,0.474919614641571)
		--(axis cs:-6,0.415455872559329)
		--cycle;

		\path [draw=white, fill=color2, opacity=0.29]
		(axis cs:-6,0.00250816434142244)
		--(axis cs:-6,0)
		--(axis cs:-5,0)
		--(axis cs:-4,0)
		--(axis cs:-3,0)
		--(axis cs:-2,0.000415831365153044)
		--(axis cs:-1,0.00135164878776638)
		--(axis cs:0,0.00135164878776638)
		--(axis cs:1,0.00135164878776638)
		--(axis cs:1,0.00948283061234889)
		--(axis cs:1,0.00948283061234889)
		--(axis cs:0,0.00948283061234889)
		--(axis cs:-1,0.00948283061234889)
		--(axis cs:-2,0.00640049976053014)
		--(axis cs:-3,0.00250816434142244)
		--(axis cs:-4,0.00250816434142244)
		--(axis cs:-5,0.00250816434142244)
		--(axis cs:-6,0.00250816434142244)
		--cycle;

		\path [draw=white, fill=black, opacity=0.29]
		(axis cs:-6,0.852311601125584)
		--(axis cs:-6,0.805781517885908)
		--(axis cs:-5,0.805781517885908)
		--(axis cs:-4,0.808937033458904)
		--(axis cs:-3,0.815256155424137)
		--(axis cs:-2,0.831104107425106)
		--(axis cs:-1,0.880251719073068)
		--(axis cs:0,0.901966466683103)
		--(axis cs:1,0.903059537837456)
		--(axis cs:1,0.936508505111118)
		--(axis cs:1,0.936508505111118)
		--(axis cs:0,0.935607334995834)
		--(axis cs:-1,0.917437347156946)
		--(axis cs:-2,0.874850419079706)
		--(axis cs:-3,0.860784986841285)
		--(axis cs:-4,0.855138759440034)
		--(axis cs:-5,0.852311601125584)
		--(axis cs:-6,0.852311601125584)
		--cycle;

		\path [draw=white, fill=red!60!black, opacity=0.29]
		(axis cs:-6,0.856080553553634)
		--(axis cs:-6,0.809989464088767)
		--(axis cs:-5,0.809989464088767)
		--(axis cs:-4,0.812095226345119)
		--(axis cs:-3,0.816310417062526)
		--(axis cs:-2,0.833222918734547)
		--(axis cs:-1,0.870558873237504)
		--(axis cs:0,0.895423785572038)
		--(axis cs:1,0.895423785572038)
		--(axis cs:1,0.930184578946515)
		--(axis cs:1,0.930184578946515)
		--(axis cs:0,0.930184578946515)
		--(axis cs:-1,0.909182102556705)
		--(axis cs:-2,0.87672006035825)
		--(axis cs:-3,0.861724950316633)
		--(axis cs:-4,0.857963240962008)
		--(axis cs:-5,0.856080553553634)
		--(axis cs:-6,0.856080553553634)
		--cycle;

		\addplot [line width=0.52pt, color1, opacity=1.0, mark=square*, mark size=2.5, mark options={solid}]
		table {%
				-6 0.385
				-5 0.444
				-4 0.458
				-3 0.476
				-2 0.487
				-1 0.488
				0 0.488
				1 0.488
			};
		\addplot [line width=0.52pt, color2, opacity=1.0, dashed, mark=*, mark size=2.5, mark options={solid}]
		table {%
				-6 0
				-5 0
				-4 0
				-3 0
				-2 0.002
				-1 0.004
				0 0.004
				1 0.004
			};
		\addplot [line width=0.52pt, black, opacity=1.0, dash pattern=on 1pt off 3pt on 3pt off 3pt, mark=diamond*, mark size=3.5, mark options={solid}]
		table {%
				-6 0.83
				-5 0.83
				-4 0.833
				-3 0.839
				-2 0.854
				-1 0.9
				0 0.92
				1 0.921
			};
		\addplot [line width=0.52pt, red!60!black, opacity=1.0, mark=asterisk, mark size=3.5, mark options={solid}]
		table {%
				-6 0.834
				-5 0.834
				-4 0.836
				-3 0.84
				-2 0.856
				-1 0.891
				0 0.914
				1 0.914
			};
	\end{axis}

\end{tikzpicture}

%% file: tables/revolute.tex
\begin{tabular}{lrcrccrcc}
    \toprule
    Method & \multicolumn{2}{c}{\texttt{trust-exact/constr}} & \multicolumn{3}{c}{\texttt{RTR}} & \multicolumn{3}{c}{\texttt{RTR-B}} \\
    \cmidrule(r){2-3}  \cmidrule(r){4-6}   \cmidrule(r){7-9}
     &Success ($\%$) &Iter. $\mu$ ($\sigma$) &Success ($\%$) &Iter. $\mu$ ($\sigma$) &Runtime [ms] $\mu$ ($\sigma$) &Success ($\%$) &Iter. $\mu$ ($\sigma$) &Runtime [ms] $\mu$ ($\sigma$) \\
    \midrule
    \midrule
    UR10  & $90.0 \pm 1.3$ & 12 (7) & $87 \pm 1.0$ & 364 (523) & 205.8 (264.1) & $\mathbf{95.0 \pm 1.0}$ & $\mathbf{319 \,(493)}$ & 198.4 (239.3) \\
    UR10$+$ & $63.0 \pm 2.1$ & 36 (19) & $\mathbf{77.0 \pm 1.8}$ & $\mathbf{364\,(523)}$ & 206.6 (264.2) & $66.0 \pm 2.0$ & 251 (457) & 138.2 (193.0) \\
    KUKA & $100.0 \pm 0.1$ & 20 (6) & $100.0 \pm 0.2$ & 317 (373) & 242.8 (211.1) & $100.0 \pm 0.1$ & 315 (386) & 268.3 (220.0) \\
    KUKA$+$ & $87.0 \pm 1.5$ & 48 (18) & $82.0 \pm 1.7$ & 734 (771) & 432.4 (395.1) & $\mathbf{89 \pm 1.3}$ & $\mathbf{506\,(660)}$ & 386.8 (374.8)  \\
    LWA4P & $\mathbf{100.0 \pm 0.2}$ & $\mathbf{20\,(17)}$ & $89.0 \pm 1.4$ & 513 (624) & 416.6 (542.4) & $90.0 \pm 1.3$ & 503 (614) & 290.0 (330.1) \\
    LWA4P$+$ & $77.0 \pm 1.8$ & 46 (25) & $\mathbf{87.0 \pm 1.5}$ & $\mathbf{435\,(569)}$ & 246.7 (282.2) & $81 \pm 1.7$ & 324 (482) & 201.1 (219.7) \\
    LWA4D & $\mathbf{100.0 \pm 0.1}$ & $\mathbf{24\,(17)}$ & $99.0 \pm 0.5$ & 868 (747) &969.9 (878.0) & $97.0 \pm 0.7$ & 713 (699) & 895.6 (897.1) \\
    LWA4D$+$ & $\mathbf{96.0 \pm 0.8}$ & $\mathbf{47\,(17)}$ & $91.0 \pm 1.3$ & 867 (769) & 900.2 (874.1) & $90.0 \pm 1.3$ & 825 (758) & 991.2 (943.8) \\
\end{tabular}

%% file: figs/results/boxplot_ur10_bounded_False.tex
\begin{tikzpicture}

	\definecolor{color0}{rgb}{0.917647058823529,0.917647058823529,0.949019607843137}
	\definecolor{color1}{rgb}{0.347058823529412,0.458823529411765,0.641176470588235}
	\definecolor{color2}{rgb}{0.71078431372549,0.363725490196079,0.375490196078431}

	\begin{axis}[
			axis background/.style={fill=color0},
			axis line style={white},
			legend cell align={left},
			legend style={
					fill opacity=0.8,
					draw opacity=1,
					text opacity=1,
					at={(0.45,0.35)},
					anchor=north west,
					draw=white!80!black,
					fill=color0,
          nodes={scale=0.8, transform shape}
				},
			tick align=outside,
			x grid style={white},
			xmajorticks=false,
			xmin=-0.5, xmax=0.5,
			xtick style={color=white!15!black},
			xtick={0},
			xticklabels={ur10},
			y grid style={white},
			ylabel={$\log_{10}$ Pos. Error},
			ymajorgrids,
			ymin=-17.2896601281028, ymax=-1.163432063259,
			ytick style={color=white!15!black},
			ytick={-18,-16,-14,-12,-10,-8,-6,-4,-2,0},
			width=9cm,
			height=6cm
		]
		\path [draw=white!29.8039215686275!black, fill=color1, semithick]
		(axis cs:-0.358133333333333,-14.3740933133648)
		--(axis cs:-0.1752,-14.3740933133648)
		--(axis cs:-0.1752,-10.1576470484392)
		--(axis cs:-0.358133333333333,-10.1576470484392)
		--(axis cs:-0.358133333333333,-14.3740933133648)
		--cycle;
		\path [draw=white!29.8039215686275!black, fill=white!54.9019607843137!black, semithick]
		(axis cs:-0.0914666666666667,-3.94972475304225)
		--(axis cs:0.0914666666666667,-3.94972475304225)
		--(axis cs:0.0914666666666667,-3.12225077857308)
		--(axis cs:-0.0914666666666667,-3.12225077857308)
		--(axis cs:-0.0914666666666667,-3.94972475304225)
		--cycle;
		\path [draw=white!29.8039215686275!black, fill=color2, semithick]
		(axis cs:0.1752,-3.93821977446198)
		--(axis cs:0.358133333333333,-3.93821977446198)
		--(axis cs:0.358133333333333,-3.27984985034484)
		--(axis cs:0.1752,-3.27984985034484)
		--(axis cs:0.1752,-3.93821977446198)
		--cycle;
		\draw[draw=white!29.8039215686275!black,fill=color1,line width=0.3pt] (axis cs:0,0) rectangle (axis cs:0,0);
		\addlegendimage{ybar,ybar legend,draw=white!29.8039215686275!black,fill=color1,line width=0.3pt}
		\addlegendentry{\texttt{trust-exact/constr}}

		\draw[draw=white!29.8039215686275!black,fill=white!54.9019607843137!black,line width=0.3pt] (axis cs:0,0) rectangle (axis cs:0,0);
		\addlegendimage{ybar,ybar legend,draw=white!29.8039215686275!black,fill=white!54.9019607843137!black,line width=0.3pt}
		\addlegendentry{\texttt{RTR}}

		\draw[draw=white!29.8039215686275!black,fill=color2,line width=0.3pt] (axis cs:0,0) rectangle (axis cs:0,0);
		\addlegendimage{ybar,ybar legend,draw=white!29.8039215686275!black,fill=color2,line width=0.3pt}
		\addlegendentry{\texttt{RTR-B}}

		\addplot [semithick, white!29.8039215686275!black, forget plot]
		table {%
				-0.266666666666667 -14.3740933133648
				-0.266666666666667 -16.556649761519
			};
		\addplot [semithick, white!29.8039215686275!black, forget plot]
		table {%
				-0.266666666666667 -10.1576470484392
				-0.266666666666667 -8.72338587568205
			};
		\addplot [semithick, white!29.8039215686275!black, forget plot]
		table {%
				-0.332 -16.556649761519
				-0.201333333333333 -16.556649761519
			};
		\addplot [semithick, white!29.8039215686275!black, forget plot]
		table {%
				-0.332 -8.72338587568205
				-0.201333333333333 -8.72338587568205
			};
		\addplot [semithick, white!29.8039215686275!black, forget plot]
		table {%
				0 -3.94972475304225
				0 -4.80985953249471
			};
		\addplot [semithick, white!29.8039215686275!black, forget plot]
		table {%
				0 -3.12225077857308
				0 -1.8964424298428
			};
		\addplot [semithick, white!29.8039215686275!black, forget plot]
		table {%
				-0.0653333333333333 -4.80985953249471
				0.0653333333333333 -4.80985953249471
			};
		\addplot [semithick, white!29.8039215686275!black, forget plot]
		table {%
				-0.0653333333333333 -1.8964424298428
				0.0653333333333333 -1.8964424298428
			};
		\addplot [semithick, white!29.8039215686275!black, forget plot]
		table {%
				0.266666666666667 -3.93821977446198
				0.266666666666667 -4.75071755554685
			};
		\addplot [semithick, white!29.8039215686275!black, forget plot]
		table {%
				0.266666666666667 -3.27984985034484
				0.266666666666667 -2.30609641547047
			};
		\addplot [semithick, white!29.8039215686275!black, forget plot]
		table {%
				0.201333333333333 -4.75071755554685
				0.332 -4.75071755554685
			};
		\addplot [semithick, white!29.8039215686275!black, forget plot]
		table {%
				0.201333333333333 -2.30609641547047
				0.332 -2.30609641547047
			};
		\addplot [semithick, white!29.8039215686275!black, forget plot]
		table {%
				-0.358133333333333 -12.0948502079205
				-0.1752 -12.0948502079205
			};
		\addplot [semithick, white!29.8039215686275!black, forget plot]
		table {%
				-0.0914666666666667 -3.53922831823539
				0.0914666666666667 -3.53922831823539
			};
		\addplot [semithick, white!29.8039215686275!black, forget plot]
		table {%
				0.1752 -3.61119378304041
				0.358133333333333 -3.61119378304041
			};
	\end{axis}

\end{tikzpicture}

%% file: figs/results/boxplot_ur10_bounded_True.tex
\begin{tikzpicture}

	\definecolor{color0}{rgb}{0.917647058823529,0.917647058823529,0.949019607843137}
	\definecolor{color1}{rgb}{0.347058823529412,0.458823529411765,0.641176470588235}
	\definecolor{color2}{rgb}{0.71078431372549,0.363725490196079,0.375490196078431}

	\begin{axis}[
			axis background/.style={fill=color0},
			axis line style={white},
			tick align=outside,
			x grid style={white},
			xmajorticks=false,
			xmin=-0.5, xmax=0.5,
			xtick style={color=white!15!black},
			xtick={0},
			xticklabels={ur10},
			y grid style={white},
			ylabel={$\log_{10}$ Pos. Error},
			ymajorgrids,
			ymin=-10.3463844175168, ymax=0.655106088532708,
			ytick style={color=white!15!black},
			ytick={-12,-10,-8,-6,-4,-2,0,2},
			width=9cm,
			height=6cm
		]
		\path [draw=white!29.8039215686275!black, fill=color1, semithick]
		(axis cs:-0.358133333333333,-7.03499914390769)
		--(axis cs:-0.1752,-7.03499914390769)
		--(axis cs:-0.1752,-0.917636638779786)
		--(axis cs:-0.358133333333333,-0.917636638779786)
		--(axis cs:-0.358133333333333,-7.03499914390769)
		--cycle;
		\path [draw=white!29.8039215686275!black, fill=white!54.9019607843137!black, semithick]
		(axis cs:-0.0914666666666667,-3.87846280041835)
		--(axis cs:0.0914666666666667,-3.87846280041835)
		--(axis cs:0.0914666666666667,-2.14478312967737)
		--(axis cs:-0.0914666666666667,-2.14478312967737)
		--(axis cs:-0.0914666666666667,-3.87846280041835)
		--cycle;
		\path [draw=white!29.8039215686275!black, fill=color2, semithick]
		(axis cs:0.1752,-3.83370263844915)
		--(axis cs:0.358133333333333,-3.83370263844915)
		--(axis cs:0.358133333333333,-1.27059828644112)
		--(axis cs:0.1752,-1.27059828644112)
		--(axis cs:0.1752,-3.83370263844915)
		--cycle;
		\draw[draw=white!29.8039215686275!black,fill=color1,line width=0.3pt] (axis cs:0,0) rectangle (axis cs:0,0);

		\draw[draw=white!29.8039215686275!black,fill=white!54.9019607843137!black,line width=0.3pt] (axis cs:0,0) rectangle (axis cs:0,0);

		\draw[draw=white!29.8039215686275!black,fill=color2,line width=0.3pt] (axis cs:0,0) rectangle (axis cs:0,0);

		\addplot [semithick, white!29.8039215686275!black, forget plot]
		table {%
				-0.266666666666667 -7.03499914390769
				-0.266666666666667 -9.84631666724179
			};
		\addplot [semithick, white!29.8039215686275!black, forget plot]
		table {%
				-0.266666666666667 -0.917636638779786
				-0.266666666666667 0.155038338257731
			};
		\addplot [semithick, white!29.8039215686275!black, forget plot]
		table {%
				-0.332 -9.84631666724179
				-0.201333333333333 -9.84631666724179
			};
		\addplot [semithick, white!29.8039215686275!black, forget plot]
		table {%
				-0.332 0.155038338257731
				-0.201333333333333 0.155038338257731
			};
		\addplot [semithick, white!29.8039215686275!black, forget plot]
		table {%
				0 -3.87846280041835
				0 -4.79046324901803
			};
		\addplot [semithick, white!29.8039215686275!black, forget plot]
		table {%
				0 -2.14478312967737
				0 -0.128089156961937
			};
		\addplot [semithick, white!29.8039215686275!black, forget plot]
		table {%
				-0.0653333333333333 -4.79046324901803
				0.0653333333333333 -4.79046324901803
			};
		\addplot [semithick, white!29.8039215686275!black, forget plot]
		table {%
				-0.0653333333333333 -0.128089156961937
				0.0653333333333333 -0.128089156961937
			};
		\addplot [semithick, white!29.8039215686275!black, forget plot]
		table {%
				0.266666666666667 -3.83370263844915
				0.266666666666667 -4.79285597854558
			};
		\addplot [semithick, white!29.8039215686275!black, forget plot]
		table {%
				0.266666666666667 -1.27059828644112
				0.266666666666667 -0.340402114783476
			};
		\addplot [semithick, white!29.8039215686275!black, forget plot]
		table {%
				0.201333333333333 -4.79285597854558
				0.332 -4.79285597854558
			};
		\addplot [semithick, white!29.8039215686275!black, forget plot]
		table {%
				0.201333333333333 -0.340402114783476
				0.332 -0.340402114783476
			};
		\addplot [semithick, white!29.8039215686275!black, forget plot]
		table {%
				-0.358133333333333 -5.75195676411938
				-0.1752 -5.75195676411938
			};
		\addplot [semithick, white!29.8039215686275!black, forget plot]
		table {%
				-0.0914666666666667 -3.42433652832338
				0.0914666666666667 -3.42433652832338
			};
		\addplot [semithick, white!29.8039215686275!black, forget plot]
		table {%
				0.1752 -3.32487485662387
				0.358133333333333 -3.32487485662387
			};
	\end{axis}

\end{tikzpicture}

%% file: figs/results/waterfall_ur10_bounded_False.tex
\begin{tikzpicture}

	\definecolor{color0}{rgb}{0.917647058823529,0.917647058823529,0.949019607843137}
	\definecolor{color1}{rgb}{0.16078431372549,0.27843137254902,0.447058823529412}

	\begin{axis}[
			axis background/.style={fill=color0},
			axis line style={white},
			legend cell align={left},
			legend style={
					fill opacity=0.8,
					draw opacity=1,
					text opacity=1,
					at={(0.45,0.35)},
					anchor=north west,
					draw=white!80!black,
					fill=color0,
					nodes={scale=0.8, transform shape}
				},
			tick align=outside,
			x grid style={white},
			xlabel={log\(\displaystyle _{10}\) Pos. Error Tolerance (m)},
			xmajorgrids,
			xmin=-6.35, xmax=1.35,
			xtick style={color=white!15!black},
			y grid style={white},
			ylabel={Success Rate},
			ymajorgrids,
			ymin=-0.05, ymax=1.15,
			ytick style={color=white!15!black},
			width=9cm,
			height=6cm
		]
		\path [draw=white, fill=color1, opacity=0.29]
		(axis cs:-6,0.90314751261968)
		--(axis cs:-6,0.875724784561096)
		--(axis cs:-5,0.875724784561096)
		--(axis cs:-4,0.875724784561096)
		--(axis cs:-3,0.877304817509467)
		--(axis cs:-2,0.884161269171258)
		--(axis cs:-1,0.885217551899566)
		--(axis cs:0,0.885217551899566)
		--(axis cs:1,0.885217551899566)
		--(axis cs:1,0.911628743038352)
		--(axis cs:1,0.911628743038352)
		--(axis cs:0,0.911628743038352)
		--(axis cs:-1,0.911628743038352)
		--(axis cs:-2,0.910687914665301)
		--(axis cs:-3,0.904563145647629)
		--(axis cs:-4,0.90314751261968)
		--(axis cs:-5,0.90314751261968)
		--(axis cs:-6,0.90314751261968)
		--cycle;

		\path [draw=white, fill=black, opacity=0.29]
		(axis cs:-6,0.00125502633866996)
		--(axis cs:-6,0)
		--(axis cs:-5,0)
		--(axis cs:-4,0.205165943594171)
		--(axis cs:-3,0.771737803190953)
		--(axis cs:-2,0.855776133625265)
		--(axis cs:-1,0.861014730995337)
		--(axis cs:0,0.861539004137503)
		--(axis cs:1,0.861539004137503)
		--(axis cs:1,0.890372303914597)
		--(axis cs:1,0.890372303914597)
		--(axis cs:0,0.890372303914597)
		--(axis cs:-1,0.889898022044709)
		--(axis cs:-2,0.88515106976592)
		--(axis cs:-3,0.807423565987449)
		--(axis cs:-4,0.241635098263937)
		--(axis cs:-5,0.00125502633866996)
		--(axis cs:-6,0.00125502633866996)
		--cycle;

		\path [draw=white, fill=red!60!black, opacity=0.29]
		(axis cs:-6,0.00125502633866996)
		--(axis cs:-6,0)
		--(axis cs:-5,0)
		--(axis cs:-4,0.189674982086987)
		--(axis cs:-3,0.876778049783862)
		--(axis cs:-2,0.949086799562536)
		--(axis cs:-1,0.954600103782466)
		--(axis cs:0,0.954600103782466)
		--(axis cs:1,0.954600103782466)
		--(axis cs:1,0.971060238203285)
		--(axis cs:1,0.971060238203285)
		--(axis cs:0,0.971060238203285)
		--(axis cs:-1,0.971060238203285)
		--(axis cs:-2,0.966587891485102)
		--(axis cs:-3,0.904091358034759)
		--(axis cs:-4,0.225172322945761)
		--(axis cs:-5,0.00125502633866996)
		--(axis cs:-6,0.00125502633866996)
		--cycle;

		\addplot [line width=0.52pt, color1, opacity=1.0, mark=square*, mark size=2.5, mark options={solid}]
		table {%
				-6 0.89
				-5 0.89
				-4 0.89
				-3 0.8915
				-2 0.898
				-1 0.899
				0 0.899
				1 0.899
			};
		\addlegendentry{\texttt{trust-exact/constr}}
		\addplot [line width=0.52pt, black, opacity=1.0, dash pattern=on 1pt off 3pt on 3pt off 3pt, mark=diamond*, mark size=3.5, mark options={solid}]
		table {%
				-6 0
				-5 0
				-4 0.223
				-3 0.79
				-2 0.871
				-1 0.876
				0 0.8765
				1 0.8765
			};
		\addlegendentry{\texttt{RTR}}
		\addplot [line width=0.52pt, red!60!black, opacity=1.0, mark=asterisk, mark size=3.5, mark options={solid}]
		table {%
				-6 0
				-5 0
				-4 0.207
				-3 0.891
				-2 0.9585
				-1 0.9635
				0 0.9635
				1 0.9635
			};
		\addlegendentry{\texttt{RTR-B}}
	\end{axis}

\end{tikzpicture}

%% file: figs/results/waterfall_ur10_bounded_True.tex
\begin{tikzpicture}

	\definecolor{color0}{rgb}{0.917647058823529,0.917647058823529,0.949019607843137}
	\definecolor{color1}{rgb}{0.16078431372549,0.27843137254902,0.447058823529412}

	\begin{axis}[
			axis background/.style={fill=color0},
			axis line style={white},
			tick align=outside,
			x grid style={white},
			xlabel={log\(\displaystyle _{10}\) Pos. Error Tolerance (m)},
			xmajorgrids,
			xmin=-6.35, xmax=1.35,
			xtick style={color=white!15!black},
			y grid style={white},
			ylabel={Success Rate},
			ymajorgrids,
			ymin=-0.05, ymax=1.15,
			ytick style={color=white!15!black},
			width=9cm,
			height=6cm
		]
		\path [draw=white, fill=color1, opacity=0.29]
		(axis cs:-6,0.453794794750727)
		--(axis cs:-6,0.410401858750292)
		--(axis cs:-5,0.59049436754273)
		--(axis cs:-4,0.59352063645919)
		--(axis cs:-3,0.595034091222919)
		--(axis cs:-2,0.607149480369027)
		--(axis cs:-1,0.610685754385021)
		--(axis cs:0,0.610685754385021)
		--(axis cs:1,0.610685754385021)
		--(axis cs:1,0.652932508480512)
		--(axis cs:1,0.652932508480512)
		--(axis cs:0,0.652932508480512)
		--(axis cs:-1,0.652932508480512)
		--(axis cs:-2,0.649478904167661)
		--(axis cs:-3,0.637628996281155)
		--(axis cs:-4,0.636146788925477)
		--(axis cs:-5,0.633181733609159)
		--(axis cs:-6,0.453794794750727)
		--cycle;

		\path [draw=white, fill=black, opacity=0.29]
		(axis cs:-6,0.00125502633866996)
		--(axis cs:-6,0)
		--(axis cs:-5,0)
		--(axis cs:-4,0.172298233381162)
		--(axis cs:-3,0.651187141169958)
		--(axis cs:-2,0.731694561764896)
		--(axis cs:-1,0.737328390878813)
		--(axis cs:0,0.737840763285224)
		--(axis cs:1,0.737840763285224)
		--(axis cs:1,0.775416025955465)
		--(axis cs:1,0.775416025955465)
		--(axis cs:0,0.775416025955465)
		--(axis cs:-1,0.774929844164983)
		--(axis cs:-2,0.769579577186646)
		--(axis cs:-3,0.692315446552138)
		--(axis cs:-4,0.206601115379196)
		--(axis cs:-5,0.00125502633866996)
		--(axis cs:-6,0.00125502633866996)
		--cycle;

		\path [draw=white, fill=red!60!black, opacity=0.29]
		(axis cs:-6,0.00125502633866996)
		--(axis cs:-6,0)
		--(axis cs:-5,0)
		--(axis cs:-4,0.15354041524985)
		--(axis cs:-3,0.586964795919246)
		--(axis cs:-2,0.649665375962096)
		--(axis cs:-1,0.65676898873443)
		--(axis cs:0,0.65676898873443)
		--(axis cs:1,0.65676898873443)
		--(axis cs:1,0.697717693844559)
		--(axis cs:1,0.697717693844559)
		--(axis cs:0,0.697717693844559)
		--(axis cs:-1,0.697717693844559)
		--(axis cs:-2,0.69084154953538)
		--(axis cs:-3,0.629721426970925)
		--(axis cs:-4,0.186415310512347)
		--(axis cs:-5,0.00125502633866996)
		--(axis cs:-6,0.00125502633866996)
		--cycle;

		\addplot [line width=0.52pt, color1, opacity=1.0, mark=square*, mark size=2.5, mark options={solid}]
		table {%
				-6 0.432
				-5 0.612
				-4 0.615
				-3 0.6165
				-2 0.6285
				-1 0.632
				0 0.632
				1 0.632
			};
		\addplot [line width=0.52pt, black, opacity=1.0, dash pattern=on 1pt off 3pt on 3pt off 3pt, mark=diamond*, mark size=3.5, mark options={solid}]
		table {%
				-6 0
				-5 0
				-4 0.189
				-3 0.672
				-2 0.751
				-1 0.7565
				0 0.757
				1 0.757
			};
		\addplot [line width=0.52pt, red!60!black, opacity=1.0, mark=asterisk, mark size=3.5, mark options={solid}]
		table {%
				-6 0
				-5 0
				-4 0.1695
				-3 0.6085
				-2 0.6705
				-1 0.6775
				0 0.6775
				1 0.6775
			};
	\end{axis}

\end{tikzpicture}

%% file: figs/results/exp1_srate.tex
\begin{tikzpicture}

\definecolor{color0}{rgb}{0.917647058823529,0.917647058823529,0.949019607843137}
\definecolor{color1}{rgb}{0.347058823529412,0.458823529411765,0.641176470588235}
\definecolor{color2}{rgb}{0.71078431372549,0.363725490196079,0.375490196078431}

\begin{groupplot}[group style={group size=4 by 1}]
\nextgroupplot[
axis background/.style={fill=color0},
axis line style={white},
legend cell align={left},
legend style={
  fill opacity=0.8,
  draw opacity=1,
  text opacity=1,
  at={(0.03,0.03)},
  anchor=south west,
  draw=white!80!black,
  fill=color0
},
tick align=outside,
tick pos=both,
title={Environment = Free},
x grid style={white},
xmin=-0.5, xmax=2.5,
xtick style={color=white!15!black},
xtick={0,1,2},
xticklabels={UR10,KUKA-IIWA,Schunk-LWA4D},
y grid style={white},
ylabel={Success Rate},
ymajorgrids,
ymin=0, ymax=1.0416,
ytick style={color=white!15!black},
ytick={0,0.2,0.4,0.6,0.8,1,1.2},
yticklabels={0.0,0.2,0.4,0.6,0.8,1.0,1.2}
]
\draw[draw=black,fill=color1] (axis cs:-0.35,0) rectangle (axis cs:-0.05,0.881);
\addlegendimage{ybar,ybar legend,draw=black,fill=color1};
\addlegendentry{j-SLSQP}

\draw[draw=black,fill=color1] (axis cs:0.65,0) rectangle (axis cs:0.95,0.9758);
\draw[draw=black,fill=color1] (axis cs:1.65,0) rectangle (axis cs:1.95,0.9614);
\draw[draw=black,fill=color2] (axis cs:0.05,0) rectangle (axis cs:0.35,0.9466);
\addlegendimage{ybar,ybar legend,draw=black,fill=color2};
\addlegendentry{RTR-B}

\draw[draw=black,fill=color2] (axis cs:1.05,0) rectangle (axis cs:1.35,0.992);
\draw[draw=black,fill=color2] (axis cs:2.05,0) rectangle (axis cs:2.35,0.9888);
\addplot [line width=1.08pt, white!26!black, forget plot]
table {%
-0.2 0.881
-0.2 0.881
};
\addplot [line width=1.08pt, white!26!black, forget plot]
table {%
0.8 0.9758
0.8 0.9758
};
\addplot [line width=1.08pt, white!26!black, forget plot]
table {%
1.8 0.9614
1.8 0.9614
};
\addplot [line width=1.08pt, white!26!black, forget plot]
table {%
0.2 0.9466
0.2 0.9466
};
\addplot [line width=1.08pt, white!26!black, forget plot]
table {%
1.2 0.992
1.2 0.992
};
\addplot [line width=1.08pt, white!26!black, forget plot]
table {%
2.2 0.9888
2.2 0.9888
};
\addplot [semithick, black, dashed, forget plot]
table {%
-0.35 1
0.4 1
};
\addplot [semithick, black, dashed, forget plot]
table {%
0.655 1
1.405 1
};
\addplot [semithick, black, dashed, forget plot]
table {%
1.66 1
2.38 1
};

\nextgroupplot[
axis background/.style={fill=color0},
axis line style={white},
scaled y ticks=manual:{}{\pgfmathparse{#1}},
tick align=outside,
tick pos=both,
title={Environment = Octahedron},
x grid style={white},
xmin=-0.5, xmax=2.5,
xtick style={color=white!15!black},
xtick={0,1,2},
xticklabels={UR10,KUKA-IIWA,Schunk-LWA4D},
y grid style={white},
ymajorgrids,
ymin=0, ymax=1.0416,
ytick style={color=white!15!black},
yticklabels={}
]
\draw[draw=black,fill=color1] (axis cs:-0.35,0) rectangle (axis cs:-0.05,0.3508);
\draw[draw=black,fill=color1] (axis cs:0.65,0) rectangle (axis cs:0.95,0.5904);
\draw[draw=black,fill=color1] (axis cs:1.65,0) rectangle (axis cs:1.95,0.656);
\draw[draw=black,fill=color2] (axis cs:0.05,0) rectangle (axis cs:0.35,0.4618);
\draw[draw=black,fill=color2] (axis cs:1.05,0) rectangle (axis cs:1.35,0.772);
\draw[draw=black,fill=color2] (axis cs:2.05,0) rectangle (axis cs:2.35,0.7542);
\addplot [line width=1.08pt, white!26!black]
table {%
-0.2 0.3508
-0.2 0.3508
};
\addplot [line width=1.08pt, white!26!black]
table {%
0.8 0.5904
0.8 0.5904
};
\addplot [line width=1.08pt, white!26!black]
table {%
1.8 0.656
1.8 0.656
};
\addplot [line width=1.08pt, white!26!black]
table {%
0.2 0.4618
0.2 0.4618
};
\addplot [line width=1.08pt, white!26!black]
table {%
1.2 0.772
1.2 0.772
};
\addplot [line width=1.08pt, white!26!black]
table {%
2.2 0.7542
2.2 0.7542
};
\addplot [semithick, black, dashed]
table {%
-0.35 0.451
0.4 0.451
};
\addplot [semithick, black, dashed]
table {%
0.655 0.6208
1.405 0.6208
};
\addplot [semithick, black, dashed]
table {%
1.66 0.6496
2.38 0.6496
};

\nextgroupplot[
axis background/.style={fill=color0},
axis line style={white},
scaled y ticks=manual:{}{\pgfmathparse{#1}},
tick align=outside,
tick pos=both,
title={Environment = Cube},
x grid style={white},
xmin=-0.5, xmax=2.5,
xtick style={color=white!15!black},
xtick={0,1,2},
xticklabels={UR10,KUKA-IIWA,Schunk-LWA4D},
y grid style={white},
ymajorgrids,
ymin=0, ymax=1.0416,
ytick style={color=white!15!black},
yticklabels={}
]
\draw[draw=black,fill=color1] (axis cs:-0.35,0) rectangle (axis cs:-0.05,0.319);
\draw[draw=black,fill=color1] (axis cs:0.65,0) rectangle (axis cs:0.95,0.6324);
\draw[draw=black,fill=color1] (axis cs:1.65,0) rectangle (axis cs:1.95,0.715);
\draw[draw=black,fill=color2] (axis cs:0.05,0) rectangle (axis cs:0.35,0.3326);
\draw[draw=black,fill=color2] (axis cs:1.05,0) rectangle (axis cs:1.35,0.6506);
\draw[draw=black,fill=color2] (axis cs:2.05,0) rectangle (axis cs:2.35,0.7342);
\addplot [line width=1.08pt, white!26!black]
table {%
-0.2 0.319
-0.2 0.319
};
\addplot [line width=1.08pt, white!26!black]
table {%
0.8 0.6324
0.8 0.6324
};
\addplot [line width=1.08pt, white!26!black]
table {%
1.8 0.715
1.8 0.715
};
\addplot [line width=1.08pt, white!26!black]
table {%
0.2 0.3326
0.2 0.3326
};
\addplot [line width=1.08pt, white!26!black]
table {%
1.2 0.6506
1.2 0.6506
};
\addplot [line width=1.08pt, white!26!black]
table {%
2.2 0.7342
2.2 0.7342
};
\addplot [semithick, black, dashed]
table {%
-0.349999999999999 0.3786
0.400000000000001 0.3786
};
\addplot [semithick, black, dashed]
table {%
0.655000000000001 0.5424
1.405 0.5424
};
\addplot [semithick, black, dashed]
table {%
1.66 0.7112
2.38 0.7112
};

\nextgroupplot[
axis background/.style={fill=color0},
axis line style={white},
scaled y ticks=manual:{}{\pgfmathparse{#1}},
tick align=outside,
tick pos=both,
title={Environment = Icosahedron},
x grid style={white},
xmin=-0.5, xmax=2.5,
xtick style={color=white!15!black},
xtick={0,1,2},
xticklabels={UR10,KUKA-IIWA,Schunk-LWA4D},
y grid style={white},
ymajorgrids,
ymin=0, ymax=1.0416,
ytick style={color=white!15!black},
yticklabels={}
]
\draw[draw=black,fill=color1] (axis cs:-0.35,0) rectangle (axis cs:-0.05,0.0222);
\draw[draw=black,fill=color1] (axis cs:0.65,0) rectangle (axis cs:0.95,0.4156);
\draw[draw=black,fill=color1] (axis cs:1.65,0) rectangle (axis cs:1.95,0.5622);
\draw[draw=black,fill=color2] (axis cs:0.05,0) rectangle (axis cs:0.35,0.1014);
\draw[draw=black,fill=color2] (axis cs:1.05,0) rectangle (axis cs:1.35,0.4832);
\draw[draw=black,fill=color2] (axis cs:2.05,0) rectangle (axis cs:2.35,0.6288);
\addplot [line width=1.08pt, white!26!black]
table {%
-0.2 0.0222
-0.2 0.0222
};
\addplot [line width=1.08pt, white!26!black]
table {%
0.8 0.4156
0.8 0.4156
};
\addplot [line width=1.08pt, white!26!black]
table {%
1.8 0.5622
1.8 0.5622
};
\addplot [line width=1.08pt, white!26!black]
table {%
0.2 0.1014
0.2 0.1014
};
\addplot [line width=1.08pt, white!26!black]
table {%
1.2 0.4832
1.2 0.4832
};
\addplot [line width=1.08pt, white!26!black]
table {%
2.2 0.6288
2.2 0.6288
};
\addplot [semithick, black, dashed]
table {%
-0.35 0.1206
0.4 0.1206
};
\addplot [semithick, black, dashed]
table {%
0.655 0.3786
1.405 0.3786
};
\addplot [semithick, black, dashed]
table {%
1.66 0.569
2.38 0.569
};
\end{groupplot}

\end{tikzpicture}

%% file: figs/results/exp1_poserr.tex
\begin{tikzpicture}

\definecolor{color0}{rgb}{0.917647058823529,0.917647058823529,0.949019607843137}
\definecolor{color1}{rgb}{0.347058823529412,0.458823529411765,0.641176470588235}
\definecolor{color2}{rgb}{0.71078431372549,0.363725490196079,0.375490196078431}

\begin{groupplot}[group style={group size=4 by 1}]
\nextgroupplot[
axis background/.style={fill=color0},
axis line style={white},
log basis y={10},
tick align=outside,
tick pos=both,
x grid style={white},
xmin=-0.5, xmax=2.5,
xtick style={color=white!15!black},
xtick={0,1,2},
xticklabels={UR10,KUKA-IIWA,Schunk-LWA4D},
y grid style={white},
ylabel={Pos. Error [m]},
ymajorgrids,
ymin=3.35596916887581e-08, ymax=6.1588513422699,
ymode=log,
ytick style={color=white!15!black, draw=none},
]
\path [draw=white!29.8039215686275!black, fill=color1, semithick]
(axis cs:-0.3372,5.7065265489573e-05)
--(axis cs:-0.0628,5.7065265489573e-05)
--(axis cs:-0.0628,0.000223310930161978)
--(axis cs:-0.3372,0.000223310930161978)
--(axis cs:-0.3372,5.7065265489573e-05)
--cycle;
\path [draw=white!29.8039215686275!black, fill=color2, semithick]
(axis cs:0.0628,1.29309096100591e-05)
--(axis cs:0.3372,1.29309096100591e-05)
--(axis cs:0.3372,2.55227251766005e-05)
--(axis cs:0.0628,2.55227251766005e-05)
--(axis cs:0.0628,1.29309096100591e-05)
--cycle;
\path [draw=white!29.8039215686275!black, fill=color1, semithick]
(axis cs:0.6628,5.83061061859982e-05)
--(axis cs:0.9372,5.83061061859982e-05)
--(axis cs:0.9372,0.000181438206006929)
--(axis cs:0.6628,0.000181438206006929)
--(axis cs:0.6628,5.83061061859982e-05)
--cycle;
\path [draw=white!29.8039215686275!black, fill=color2, semithick]
(axis cs:1.0628,1.1663201049502e-05)
--(axis cs:1.3372,1.1663201049502e-05)
--(axis cs:1.3372,2.00823407195995e-05)
--(axis cs:1.0628,2.00823407195995e-05)
--(axis cs:1.0628,1.1663201049502e-05)
--cycle;
\path [draw=white!29.8039215686275!black, fill=color1, semithick]
(axis cs:1.6628,6.26860385724138e-05)
--(axis cs:1.9372,6.26860385724138e-05)
--(axis cs:1.9372,0.00019439868545324)
--(axis cs:1.6628,0.00019439868545324)
--(axis cs:1.6628,6.26860385724138e-05)
--cycle;
\path [draw=white!29.8039215686275!black, fill=color2, semithick]
(axis cs:2.0628,9.3329444963215e-06)
--(axis cs:2.3372,9.3329444963215e-06)
--(axis cs:2.3372,1.75911306211547e-05)
--(axis cs:2.0628,1.75911306211547e-05)
--(axis cs:2.0628,9.3329444963215e-06)
--cycle;
\draw[draw=white!29.8039215686275!black,fill=color1,line width=0.3pt] (axis cs:0,0) rectangle (axis cs:0,0);

\draw[draw=white!29.8039215686275!black,fill=color2,line width=0.3pt] (axis cs:0,0) rectangle (axis cs:0,0);

\addplot [semithick, white!29.8039215686275!black]
table {%
-0.2 5.7065265489573e-05
-0.2 2.90301766356009e-06
};
\addplot [semithick, white!29.8039215686275!black]
table {%
-0.2 0.000223310930161978
-0.2 0.00046675491582891
};
\addplot [semithick, white!29.8039215686275!black]
table {%
-0.298 2.90301766356009e-06
-0.102 2.90301766356009e-06
};
\addplot [semithick, white!29.8039215686275!black]
table {%
-0.298 0.00046675491582891
-0.102 0.00046675491582891
};
\addplot [semithick, white!29.8039215686275!black]
table {%
0.2 1.29309096100591e-05
0.2 1.91694043323985e-06
};
\addplot [semithick, white!29.8039215686275!black]
table {%
0.2 2.55227251766005e-05
0.2 4.4250210461469e-05
};
\addplot [semithick, white!29.8039215686275!black]
table {%
0.102 1.91694043323985e-06
0.298 1.91694043323985e-06
};
\addplot [semithick, white!29.8039215686275!black]
table {%
0.102 4.4250210461469e-05
0.298 4.4250210461469e-05
};
\addplot [semithick, white!29.8039215686275!black]
table {%
0.8 5.83061061859982e-05
0.8 2.30910686791445e-06
};
\addplot [semithick, white!29.8039215686275!black]
table {%
0.8 0.000181438206006929
0.8 0.000364921501459716
};
\addplot [semithick, white!29.8039215686275!black]
table {%
0.702 2.30910686791445e-06
0.898 2.30910686791445e-06
};
\addplot [semithick, white!29.8039215686275!black]
table {%
0.702 0.000364921501459716
0.898 0.000364921501459716
};
\addplot [semithick, white!29.8039215686275!black]
table {%
1.2 1.1663201049502e-05
1.2 7.96966121977095e-08
};
\addplot [semithick, white!29.8039215686275!black]
table {%
1.2 2.00823407195995e-05
1.2 3.26618601338646e-05
};
\addplot [semithick, white!29.8039215686275!black]
table {%
1.102 7.96966121977095e-08
1.298 7.96966121977095e-08
};
\addplot [semithick, white!29.8039215686275!black]
table {%
1.102 3.26618601338646e-05
1.298 3.26618601338646e-05
};
\addplot [semithick, white!29.8039215686275!black]
table {%
1.8 6.26860385724138e-05
1.8 3.18428296867469e-06
};
\addplot [semithick, white!29.8039215686275!black]
table {%
1.8 0.00019439868545324
1.8 0.000391352619919491
};
\addplot [semithick, white!29.8039215686275!black]
table {%
1.702 3.18428296867469e-06
1.898 3.18428296867469e-06
};
\addplot [semithick, white!29.8039215686275!black]
table {%
1.702 0.000391352619919491
1.898 0.000391352619919491
};
\addplot [semithick, white!29.8039215686275!black]
table {%
2.2 9.3329444963215e-06
2.2 5.24324033961633e-07
};
\addplot [semithick, white!29.8039215686275!black]
table {%
2.2 1.75911306211547e-05
2.2 2.98470176429265e-05
};
\addplot [semithick, white!29.8039215686275!black]
table {%
2.102 5.24324033961633e-07
2.298 5.24324033961633e-07
};
\addplot [semithick, white!29.8039215686275!black]
table {%
2.102 2.98470176429265e-05
2.298 2.98470176429265e-05
};
\addplot [semithick, black, dashed]
table {%
-0.5 0.0100000000000001
2.5 0.0100000000000001
};
\addplot [semithick, white!29.8039215686275!black]
table {%
-0.3372 0.000115238909203396
-0.0628 0.000115238909203396
};
\addplot [semithick, white!29.8039215686275!black]
table {%
0.0628 1.79957182311474e-05
0.3372 1.79957182311474e-05
};
\addplot [semithick, white!29.8039215686275!black]
table {%
0.6628 0.00010413764352707
0.9372 0.00010413764352707
};
\addplot [semithick, white!29.8039215686275!black]
table {%
1.0628 1.53907795715278e-05
1.3372 1.53907795715278e-05
};
\addplot [semithick, white!29.8039215686275!black]
table {%
1.6628 0.000114759461694158
1.9372 0.000114759461694158
};
\addplot [semithick, white!29.8039215686275!black]
table {%
2.0628 1.26666057754655e-05
2.3372 1.26666057754655e-05
};

\nextgroupplot[
axis background/.style={fill=color0},
axis line style={white},
log basis y={10},
ymajorticks=false,
scaled y ticks=manual:{}{\pgfmathparse{#1}},
tick align=outside,
tick pos=both,
x grid style={white},
xmin=-0.5, xmax=2.5,
xtick style={color=white!15!black},
xtick={0,1,2},
xticklabels={UR10,KUKA-IIWA,Schunk-LWA4D},
y grid style={white},
ymajorgrids,
ymin=3.35596916887581e-08, ymax=6.1588513422699,
ymode=log,
ytick style={color=white!15!black},
]
\path [draw=white!29.8039215686275!black, fill=color1, semithick]
(axis cs:-0.3372,0.00019115245892409)
--(axis cs:-0.0628,0.00019115245892409)
--(axis cs:-0.0628,0.328071704937978)
--(axis cs:-0.3372,0.328071704937978)
--(axis cs:-0.3372,0.00019115245892409)
--cycle;
\path [draw=white!29.8039215686275!black, fill=color2, semithick]
(axis cs:0.0628,1.82217343810861e-05)
--(axis cs:0.3372,1.82217343810861e-05)
--(axis cs:0.3372,0.135740302279957)
--(axis cs:0.0628,0.135740302279957)
--(axis cs:0.0628,1.82217343810861e-05)
--cycle;
\path [draw=white!29.8039215686275!black, fill=color1, semithick]
(axis cs:0.6628,9.37003204374951e-05)
--(axis cs:0.9372,9.37003204374951e-05)
--(axis cs:0.9372,0.104471188714368)
--(axis cs:0.6628,0.104471188714368)
--(axis cs:0.6628,9.37003204374951e-05)
--cycle;
\path [draw=white!29.8039215686275!black, fill=color2, semithick]
(axis cs:1.0628,1.14145574777874e-05)
--(axis cs:1.3372,1.14145574777874e-05)
--(axis cs:1.3372,0.00042287961219204)
--(axis cs:1.0628,0.00042287961219204)
--(axis cs:1.0628,1.14145574777874e-05)
--cycle;
\path [draw=white!29.8039215686275!black, fill=color1, semithick]
(axis cs:1.6628,9.05401724114818e-05)
--(axis cs:1.9372,9.05401724114818e-05)
--(axis cs:1.9372,0.0423248469252724)
--(axis cs:1.6628,0.0423248469252724)
--(axis cs:1.6628,9.05401724114818e-05)
--cycle;
\path [draw=white!29.8039215686275!black, fill=color2, semithick]
(axis cs:2.0628,1.01408898524499e-05)
--(axis cs:2.3372,1.01408898524499e-05)
--(axis cs:2.3372,0.00479359843500579)
--(axis cs:2.0628,0.00479359843500579)
--(axis cs:2.0628,1.01408898524499e-05)
--cycle;
\draw[draw=white!29.8039215686275!black,fill=color1,line width=0.3pt] (axis cs:0,0) rectangle (axis cs:0,0);
\draw[draw=white!29.8039215686275!black,fill=color2,line width=0.3pt] (axis cs:0,0) rectangle (axis cs:0,0);
\addplot [semithick, white!29.8039215686275!black]
table {%
-0.2 0.00019115245892409
-0.2 2.1382742306553e-06
};
\addplot [semithick, white!29.8039215686275!black]
table {%
-0.2 0.328071704937978
-0.2 0.81988901956589
};
\addplot [semithick, white!29.8039215686275!black]
table {%
-0.298 2.1382742306553e-06
-0.102 2.1382742306553e-06
};
\addplot [semithick, white!29.8039215686275!black]
table {%
-0.298 0.81988901956589
-0.102 0.81988901956589
};
\addplot [semithick, white!29.8039215686275!black]
table {%
0.2 1.82217343810861e-05
0.2 2.35037834933798e-06
};
\addplot [semithick, white!29.8039215686275!black]
table {%
0.2 0.135740302279957
0.2 0.339299239572804
};
\addplot [semithick, white!29.8039215686275!black]
table {%
0.102 2.35037834933798e-06
0.298 2.35037834933798e-06
};
\addplot [semithick, white!29.8039215686275!black]
table {%
0.102 0.339299239572804
0.298 0.339299239572804
};
\addplot [semithick, white!29.8039215686275!black]
table {%
0.8 9.37003204374951e-05
0.8 1.55894480000465e-06
};
\addplot [semithick, white!29.8039215686275!black]
table {%
0.8 0.104471188714368
0.8 0.260815151459521
};
\addplot [semithick, white!29.8039215686275!black]
table {%
0.702 1.55894480000465e-06
0.898 1.55894480000465e-06
};
\addplot [semithick, white!29.8039215686275!black]
table {%
0.702 0.260815151459521
0.898 0.260815151459521
};
\addplot [semithick, white!29.8039215686275!black]
table {%
1.2 1.14145574777874e-05
1.2 2.78441549211944e-06
};
\addplot [semithick, white!29.8039215686275!black]
table {%
1.2 0.00042287961219204
1.2 0.00101899943091368
};
\addplot [semithick, white!29.8039215686275!black]
table {%
1.102 2.78441549211944e-06
1.298 2.78441549211944e-06
};
\addplot [semithick, white!29.8039215686275!black]
table {%
1.102 0.00101899943091368
1.298 0.00101899943091368
};
\addplot [semithick, white!29.8039215686275!black]
table {%
1.8 9.05401724114818e-05
1.8 2.15549031005167e-06
};
\addplot [semithick, white!29.8039215686275!black]
table {%
1.8 0.0423248469252724
1.8 0.105452061098531
};
\addplot [semithick, white!29.8039215686275!black]
table {%
1.702 2.15549031005167e-06
1.898 2.15549031005167e-06
};
\addplot [semithick, white!29.8039215686275!black]
table {%
1.702 0.105452061098531
1.898 0.105452061098531
};
\addplot [semithick, white!29.8039215686275!black]
table {%
2.2 1.01408898524499e-05
2.2 3.01416855496594e-06
};
\addplot [semithick, white!29.8039215686275!black]
table {%
2.2 0.00479359843500579
2.2 0.0119630647230644
};
\addplot [semithick, white!29.8039215686275!black]
table {%
2.102 3.01416855496594e-06
2.298 3.01416855496594e-06
};
\addplot [semithick, white!29.8039215686275!black]
table {%
2.102 0.0119630647230644
2.298 0.0119630647230644
};
\addplot [semithick, black, dashed]
table {%
-0.5 0.0100000000000001
2.5 0.0100000000000001
};
\addplot [semithick, white!29.8039215686275!black]
table {%
-0.3372 0.0725265326169572
-0.0628 0.0725265326169572
};
\addplot [semithick, white!29.8039215686275!black]
table {%
0.0628 0.0204721316564041
0.3372 0.0204721316564041
};
\addplot [semithick, white!29.8039215686275!black]
table {%
0.6628 0.00025503027405174
0.9372 0.00025503027405174
};
\addplot [semithick, white!29.8039215686275!black]
table {%
1.0628 1.6821447193247e-05
1.3372 1.6821447193247e-05
};
\addplot [semithick, white!29.8039215686275!black]
table {%
1.6628 0.00020581981509381
1.9372 0.00020581981509381
};
\addplot [semithick, white!29.8039215686275!black]
table {%
2.0628 1.67952337747134e-05
2.3372 1.67952337747134e-05
};

\nextgroupplot[
axis background/.style={fill=color0},
axis line style={white},
log basis y={10},
ymajorticks=false,
scaled y ticks=manual:{}{\pgfmathparse{#1}},
tick align=outside,
tick pos=both,
x grid style={white},
xmin=-0.5, xmax=2.5,
xtick style={color=white!15!black},
xtick={0,1,2},
xticklabels={UR10,KUKA-IIWA,Schunk-LWA4D},
y grid style={white},
ymajorgrids,
ymin=3.35596916887581e-08, ymax=6.1588513422699,
ymode=log,
ytick style={color=white!15!black},
]
\path [draw=white!29.8039215686275!black, fill=color1, semithick]
(axis cs:-0.3372,0.000583563716546602)
--(axis cs:-0.0628,0.000583563716546602)
--(axis cs:-0.0628,0.209065831583933)
--(axis cs:-0.3372,0.209065831583933)
--(axis cs:-0.3372,0.000583563716546602)
--cycle;
\path [draw=white!29.8039215686275!black, fill=color2, semithick]
(axis cs:0.0628,2.64387487334383e-05)
--(axis cs:0.3372,2.64387487334383e-05)
--(axis cs:0.3372,0.165069004323356)
--(axis cs:0.0628,0.165069004323356)
--(axis cs:0.0628,2.64387487334383e-05)
--cycle;
\path [draw=white!29.8039215686275!black, fill=color1, semithick]
(axis cs:0.6628,8.76295220732159e-05)
--(axis cs:0.9372,8.76295220732159e-05)
--(axis cs:0.9372,0.0531390324444536)
--(axis cs:0.6628,0.0531390324444536)
--(axis cs:0.6628,8.76295220732159e-05)
--cycle;
\path [draw=white!29.8039215686275!black, fill=color2, semithick]
(axis cs:1.0628,1.34514914413371e-05)
--(axis cs:1.3372,1.34514914413371e-05)
--(axis cs:1.3372,0.0482006994598365)
--(axis cs:1.0628,0.0482006994598365)
--(axis cs:1.0628,1.34514914413371e-05)
--cycle;
\path [draw=white!29.8039215686275!black, fill=color1, semithick]
(axis cs:1.6628,7.94662504684663e-05)
--(axis cs:1.9372,7.94662504684663e-05)
--(axis cs:1.9372,0.0182224187774618)
--(axis cs:1.6628,0.0182224187774618)
--(axis cs:1.6628,7.94662504684663e-05)
--cycle;
\path [draw=white!29.8039215686275!black, fill=color2, semithick]
(axis cs:2.0628,1.02066562723634e-05)
--(axis cs:2.3372,1.02066562723634e-05)
--(axis cs:2.3372,0.0097514359090194)
--(axis cs:2.0628,0.0097514359090194)
--(axis cs:2.0628,1.02066562723634e-05)
--cycle;
\draw[draw=white!29.8039215686275!black,fill=color1,line width=0.3pt] (axis cs:0,0) rectangle (axis cs:0,0);
\draw[draw=white!29.8039215686275!black,fill=color2,line width=0.3pt] (axis cs:0,0) rectangle (axis cs:0,0);
\addplot [semithick, white!29.8039215686275!black]
table {%
-0.2 0.000583563716546602
-0.2 3.33244848276246e-06
};
\addplot [semithick, white!29.8039215686275!black]
table {%
-0.2 0.209065831583933
-0.2 0.521715005959235
};
\addplot [semithick, white!29.8039215686275!black]
table {%
-0.298 3.33244848276246e-06
-0.102 3.33244848276246e-06
};
\addplot [semithick, white!29.8039215686275!black]
table {%
-0.298 0.521715005959235
-0.102 0.521715005959235
};
\addplot [semithick, white!29.8039215686275!black]
table {%
0.2 2.64387487334383e-05
0.2 1.55292642754085e-06
};
\addplot [semithick, white!29.8039215686275!black]
table {%
0.2 0.165069004323356
0.2 0.411891219196692
};
\addplot [semithick, white!29.8039215686275!black]
table {%
0.102 1.55292642754085e-06
0.298 1.55292642754085e-06
};
\addplot [semithick, white!29.8039215686275!black]
table {%
0.102 0.411891219196692
0.298 0.411891219196692
};
\addplot [semithick, white!29.8039215686275!black]
table {%
0.8 8.76295220732159e-05
0.8 2.60017952209733e-06
};
\addplot [semithick, white!29.8039215686275!black]
table {%
0.8 0.0531390324444536
0.8 0.132570266018314
};
\addplot [semithick, white!29.8039215686275!black]
table {%
0.702 2.60017952209733e-06
0.898 2.60017952209733e-06
};
\addplot [semithick, white!29.8039215686275!black]
table {%
0.702 0.132570266018314
0.898 0.132570266018314
};
\addplot [semithick, white!29.8039215686275!black]
table {%
1.2 1.34514914413371e-05
1.2 1.61570318975464e-06
};
\addplot [semithick, white!29.8039215686275!black]
table {%
1.2 0.0482006994598365
1.2 0.120244995906214
};
\addplot [semithick, white!29.8039215686275!black]
table {%
1.102 1.61570318975464e-06
1.298 1.61570318975464e-06
};
\addplot [semithick, white!29.8039215686275!black]
table {%
1.102 0.120244995906214
1.298 0.120244995906214
};
\addplot [semithick, white!29.8039215686275!black]
table {%
1.8 7.94662504684663e-05
1.8 2.86269882426967e-06
};
\addplot [semithick, white!29.8039215686275!black]
table {%
1.8 0.0182224187774618
1.8 0.0453639695769107
};
\addplot [semithick, white!29.8039215686275!black]
table {%
1.702 2.86269882426967e-06
1.898 2.86269882426967e-06
};
\addplot [semithick, white!29.8039215686275!black]
table {%
1.702 0.0453639695769107
1.898 0.0453639695769107
};
\addplot [semithick, white!29.8039215686275!black]
table {%
2.2 1.02066562723634e-05
2.2 3.01558440352968e-06
};
\addplot [semithick, white!29.8039215686275!black]
table {%
2.2 0.0097514359090194
2.2 0.0242992046784253
};
\addplot [semithick, white!29.8039215686275!black]
table {%
2.102 3.01558440352968e-06
2.298 3.01558440352968e-06
};
\addplot [semithick, white!29.8039215686275!black]
table {%
2.102 0.0242992046784253
2.298 0.0242992046784253
};
\addplot [semithick, black, dashed]
table {%
-0.499999999999999 0.0100000000000001
2.5 0.0100000000000001
};
\addplot [semithick, white!29.8039215686275!black]
table {%
-0.3372 0.0645209786350438
-0.0628 0.0645209786350438
};
\addplot [semithick, white!29.8039215686275!black]
table {%
0.0628 0.0594879768881757
0.3372 0.0594879768881757
};
\addplot [semithick, white!29.8039215686275!black]
table {%
0.6628 0.000223775720311974
0.9372 0.000223775720311974
};
\addplot [semithick, white!29.8039215686275!black]
table {%
1.0628 2.6526405190425e-05
1.3372 2.6526405190425e-05
};
\addplot [semithick, white!29.8039215686275!black]
table {%
1.6628 0.000173309257503088
1.9372 0.000173309257503088
};
\addplot [semithick, white!29.8039215686275!black]
table {%
2.0628 1.65281378455605e-05
2.3372 1.65281378455605e-05
};

\nextgroupplot[
axis background/.style={fill=color0},
axis line style={white},
log basis y={10},
ymajorticks=false,
scaled y ticks=manual:{}{\pgfmathparse{#1}},
tick align=outside,
tick pos=both,
x grid style={white},
xmin=-0.5, xmax=2.5,
xtick style={color=white!15!black},
xtick={0,1,2},
xticklabels={UR10,KUKA-IIWA,Schunk-LWA4D},
y grid style={white},
ymajorgrids,
ymin=3.35596916887581e-08, ymax=6.1588513422699,
ymode=log,
ytick style={color=white!15!black},
]
\path [draw=white!29.8039215686275!black, fill=color1, semithick]
(axis cs:-0.3372,0.437735647520892)
--(axis cs:-0.0628,0.437735647520892)
--(axis cs:-0.0628,1.26979874741592)
--(axis cs:-0.3372,1.26979874741592)
--(axis cs:-0.3372,0.437735647520892)
--cycle;
\path [draw=white!29.8039215686275!black, fill=color2, semithick]
(axis cs:0.0628,0.0502893720720812)
--(axis cs:0.3372,0.0502893720720812)
--(axis cs:0.3372,0.207046156214751)
--(axis cs:0.0628,0.207046156214751)
--(axis cs:0.0628,0.0502893720720812)
--cycle;
\path [draw=white!29.8039215686275!black, fill=color1, semithick]
(axis cs:0.6628,0.000134088947093138)
--(axis cs:0.9372,0.000134088947093138)
--(axis cs:0.9372,0.162579182200365)
--(axis cs:0.6628,0.162579182200365)
--(axis cs:0.6628,0.000134088947093138)
--cycle;
\path [draw=white!29.8039215686275!black, fill=color2, semithick]
(axis cs:1.0628,1.55249286605058e-05)
--(axis cs:1.3372,1.55249286605058e-05)
--(axis cs:1.3372,0.0996720474932979)
--(axis cs:1.0628,0.0996720474932979)
--(axis cs:1.0628,1.55249286605058e-05)
--cycle;
\path [draw=white!29.8039215686275!black, fill=color1, semithick]
(axis cs:1.6628,9.99305820539526e-05)
--(axis cs:1.9372,9.99305820539526e-05)
--(axis cs:1.9372,0.0920171195944405)
--(axis cs:1.6628,0.0920171195944405)
--(axis cs:1.6628,9.99305820539526e-05)
--cycle;
\path [draw=white!29.8039215686275!black, fill=color2, semithick]
(axis cs:2.0628,1.12415403776532e-05)
--(axis cs:2.3372,1.12415403776532e-05)
--(axis cs:2.3372,0.0453779103372089)
--(axis cs:2.0628,0.0453779103372089)
--(axis cs:2.0628,1.12415403776532e-05)
--cycle;
\draw[draw=white!29.8039215686275!black,fill=color1,line width=0.3pt] (axis cs:0,0) rectangle (axis cs:0,0);
\draw[draw=white!29.8039215686275!black,fill=color2,line width=0.3pt] (axis cs:0,0) rectangle (axis cs:0,0);
\addplot [semithick, white!29.8039215686275!black]
table {%
-0.2 0.437735647520892
-0.2 7.18997895470759e-06
};
\addplot [semithick, white!29.8039215686275!black]
table {%
-0.2 1.26979874741592
-0.2 2.50651572091916
};
\addplot [semithick, white!29.8039215686275!black]
table {%
-0.298 7.18997895470759e-06
-0.102 7.18997895470759e-06
};
\addplot [semithick, white!29.8039215686275!black]
table {%
-0.298 2.50651572091916
-0.102 2.50651572091916
};
\addplot [semithick, white!29.8039215686275!black]
table {%
0.2 0.0502893720720812
0.2 3.24746568037454e-06
};
\addplot [semithick, white!29.8039215686275!black]
table {%
0.2 0.207046156214751
0.2 0.441833270019934
};
\addplot [semithick, white!29.8039215686275!black]
table {%
0.102 3.24746568037454e-06
0.298 3.24746568037454e-06
};
\addplot [semithick, white!29.8039215686275!black]
table {%
0.102 0.441833270019934
0.298 0.441833270019934
};
\addplot [semithick, white!29.8039215686275!black]
table {%
0.8 0.000134088947093138
0.8 7.49057557064699e-07
};
\addplot [semithick, white!29.8039215686275!black]
table {%
0.8 0.162579182200365
0.8 0.404492622554682
};
\addplot [semithick, white!29.8039215686275!black]
table {%
0.702 7.49057557064699e-07
0.898 7.49057557064699e-07
};
\addplot [semithick, white!29.8039215686275!black]
table {%
0.702 0.404492622554682
0.898 0.404492622554682
};
\addplot [semithick, white!29.8039215686275!black]
table {%
1.2 1.55249286605058e-05
1.2 1.15206749103557e-06
};
\addplot [semithick, white!29.8039215686275!black]
table {%
1.2 0.0996720474932979
1.2 0.248822459961658
};
\addplot [semithick, white!29.8039215686275!black]
table {%
1.102 1.15206749103557e-06
1.298 1.15206749103557e-06
};
\addplot [semithick, white!29.8039215686275!black]
table {%
1.102 0.248822459961658
1.298 0.248822459961658
};
\addplot [semithick, white!29.8039215686275!black]
table {%
1.8 9.99305820539526e-05
1.8 4.51822212348614e-06
};
\addplot [semithick, white!29.8039215686275!black]
table {%
1.8 0.0920171195944405
1.8 0.229891796021393
};
\addplot [semithick, white!29.8039215686275!black]
table {%
1.702 4.51822212348614e-06
1.898 4.51822212348614e-06
};
\addplot [semithick, white!29.8039215686275!black]
table {%
1.702 0.229891796021393
1.898 0.229891796021393
};
\addplot [semithick, white!29.8039215686275!black]
table {%
2.2 1.12415403776532e-05
2.2 1.30772738513022e-06
};
\addplot [semithick, white!29.8039215686275!black]
table {%
2.2 0.0453779103372089
2.2 0.113423127482499
};
\addplot [semithick, white!29.8039215686275!black]
table {%
2.102 1.30772738513022e-06
2.298 1.30772738513022e-06
};
\addplot [semithick, white!29.8039215686275!black]
table {%
2.102 0.113423127482499
2.298 0.113423127482499
};
\addplot [semithick, black, dashed]
table {%
-0.5 0.0100000000000001
2.5 0.0100000000000001
};
\addplot [semithick, white!29.8039215686275!black]
table {%
-0.3372 0.85202899085211
-0.0628 0.85202899085211
};
\addplot [semithick, white!29.8039215686275!black]
table {%
0.0628 0.118324287371211
0.3372 0.118324287371211
};
\addplot [semithick, white!29.8039215686275!black]
table {%
0.6628 0.0340796021723186
0.9372 0.0340796021723186
};
\addplot [semithick, white!29.8039215686275!black]
table {%
1.0628 0.0101463345780943
1.3372 0.0101463345780943
};
\addplot [semithick, white!29.8039215686275!black]
table {%
1.6628 0.0018973426981804
1.9372 0.0018973426981804
};
\addplot [semithick, white!29.8039215686275!black]
table {%
2.0628 2.47261224414994e-05
2.3372 2.47261224414994e-05
};
\end{groupplot}

\end{tikzpicture}

%% file: figs/results/exp1_roterr.tex
\begin{tikzpicture}

\definecolor{color0}{rgb}{0.917647058823529,0.917647058823529,0.949019607843137}
\definecolor{color1}{rgb}{0.347058823529412,0.458823529411765,0.641176470588235}
\definecolor{color2}{rgb}{0.71078431372549,0.363725490196079,0.375490196078431}

\begin{groupplot}[group style={group size=4 by 1}]
\nextgroupplot[
axis background/.style={fill=color0},
axis line style={white},
log basis y={10},
tick align=outside,
tick pos=both,
x grid style={white},
xmin=-0.5, xmax=2.5,
xtick style={color=white!15!black},
xtick={0,1,2},
xticklabels={UR10,KUKA-IIWA,Schunk-LWA4D},
y grid style={white},
ylabel={Rot. Error [rad]},
ymajorgrids,
ymin=8.65125545117774e-08, ymax=6.80487218436774,
ymode=log,
ytick style={color=white!15!black, draw=none},
]
\path [draw=white!29.8039215686275!black, fill=color1, semithick]
(axis cs:-0.3372,2.64091112863974e-05)
--(axis cs:-0.0628,2.64091112863974e-05)
--(axis cs:-0.0628,0.000119195654725539)
--(axis cs:-0.3372,0.000119195654725539)
--(axis cs:-0.3372,2.64091112863974e-05)
--cycle;
\path [draw=white!29.8039215686275!black, fill=color2, semithick]
(axis cs:0.0628,1.55858024363547e-05)
--(axis cs:0.3372,1.55858024363547e-05)
--(axis cs:0.3372,3.23723701811144e-05)
--(axis cs:0.0628,3.23723701811144e-05)
--(axis cs:0.0628,1.55858024363547e-05)
--cycle;
\path [draw=white!29.8039215686275!black, fill=color1, semithick]
(axis cs:0.6628,2.96079293505014e-05)
--(axis cs:0.9372,2.96079293505014e-05)
--(axis cs:0.9372,9.90122510387828e-05)
--(axis cs:0.6628,9.90122510387828e-05)
--(axis cs:0.6628,2.96079293505014e-05)
--cycle;
\path [draw=white!29.8039215686275!black, fill=color2, semithick]
(axis cs:1.0628,1.33872690592201e-05)
--(axis cs:1.3372,1.33872690592201e-05)
--(axis cs:1.3372,2.27733731976355e-05)
--(axis cs:1.0628,2.27733731976355e-05)
--(axis cs:1.0628,1.33872690592201e-05)
--cycle;
\path [draw=white!29.8039215686275!black, fill=color1, semithick]
(axis cs:1.6628,3.28403591118612e-05)
--(axis cs:1.9372,3.28403591118612e-05)
--(axis cs:1.9372,0.000108694043589505)
--(axis cs:1.6628,0.000108694043589505)
--(axis cs:1.6628,3.28403591118612e-05)
--cycle;
\path [draw=white!29.8039215686275!black, fill=color2, semithick]
(axis cs:2.0628,1.0834773538698e-05)
--(axis cs:2.3372,1.0834773538698e-05)
--(axis cs:2.3372,1.97534380501239e-05)
--(axis cs:2.0628,1.97534380501239e-05)
--(axis cs:2.0628,1.0834773538698e-05)
--cycle;
\draw[draw=white!29.8039215686275!black,fill=color1,line width=0.3pt] (axis cs:0,0) rectangle (axis cs:0,0);

\draw[draw=white!29.8039215686275!black,fill=color2,line width=0.3pt] (axis cs:0,0) rectangle (axis cs:0,0);

\addplot [semithick, white!29.8039215686275!black]
table {%
-0.2 2.64091112863974e-05
-0.2 3.44020020384051e-07
};
\addplot [semithick, white!29.8039215686275!black]
table {%
-0.2 0.000119195654725539
-0.2 0.00025795561698164
};
\addplot [semithick, white!29.8039215686275!black]
table {%
-0.298 3.44020020384051e-07
-0.102 3.44020020384051e-07
};
\addplot [semithick, white!29.8039215686275!black]
table {%
-0.298 0.00025795561698164
-0.102 0.00025795561698164
};
\addplot [semithick, white!29.8039215686275!black]
table {%
0.2 1.55858024363547e-05
0.2 1.18966899253209e-06
};
\addplot [semithick, white!29.8039215686275!black]
table {%
0.2 3.23723701811144e-05
0.2 5.72029319213895e-05
};
\addplot [semithick, white!29.8039215686275!black]
table {%
0.102 1.18966899253209e-06
0.298 1.18966899253209e-06
};
\addplot [semithick, white!29.8039215686275!black]
table {%
0.102 5.72029319213895e-05
0.298 5.72029319213895e-05
};
\addplot [semithick, white!29.8039215686275!black]
table {%
0.8 2.96079293505014e-05
0.8 8.2738353680981e-07
};
\addplot [semithick, white!29.8039215686275!black]
table {%
0.8 9.90122510387828e-05
0.8 0.000202865480399472
};
\addplot [semithick, white!29.8039215686275!black]
table {%
0.702 8.2738353680981e-07
0.898 8.2738353680981e-07
};
\addplot [semithick, white!29.8039215686275!black]
table {%
0.702 0.000202865480399472
0.898 0.000202865480399472
};
\addplot [semithick, white!29.8039215686275!black]
table {%
1.2 1.33872690592201e-05
1.2 5.0223766405442e-07
};
\addplot [semithick, white!29.8039215686275!black]
table {%
1.2 2.27733731976355e-05
1.2 3.68377319840056e-05
};
\addplot [semithick, white!29.8039215686275!black]
table {%
1.102 5.0223766405442e-07
1.298 5.0223766405442e-07
};
\addplot [semithick, white!29.8039215686275!black]
table {%
1.102 3.68377319840056e-05
1.298 3.68377319840056e-05
};
\addplot [semithick, white!29.8039215686275!black]
table {%
1.8 3.28403591118612e-05
1.8 9.95593851118729e-07
};
\addplot [semithick, white!29.8039215686275!black]
table {%
1.8 0.000108694043589505
1.8 0.000221084885364414
};
\addplot [semithick, white!29.8039215686275!black]
table {%
1.702 9.95593851118729e-07
1.898 9.95593851118729e-07
};
\addplot [semithick, white!29.8039215686275!black]
table {%
1.702 0.000221084885364414
1.898 0.000221084885364414
};
\addplot [semithick, white!29.8039215686275!black]
table {%
2.2 1.0834773538698e-05
2.2 1.97686242482388e-07
};
\addplot [semithick, white!29.8039215686275!black]
table {%
2.2 1.97534380501239e-05
2.2 3.3101929292389e-05
};
\addplot [semithick, white!29.8039215686275!black]
table {%
2.102 1.97686242482388e-07
2.298 1.97686242482388e-07
};
\addplot [semithick, white!29.8039215686275!black]
table {%
2.102 3.3101929292389e-05
2.298 3.3101929292389e-05
};
\addplot [semithick, black, dashed]
table {%
-0.5 0.01
2.5 0.01
};
\addplot [semithick, white!29.8039215686275!black]
table {%
-0.3372 5.78924830207809e-05
-0.0628 5.78924830207809e-05
};
\addplot [semithick, white!29.8039215686275!black]
table {%
0.0628 2.22213884642253e-05
0.3372 2.22213884642253e-05
};
\addplot [semithick, white!29.8039215686275!black]
table {%
0.6628 5.55875777958101e-05
0.9372 5.55875777958101e-05
};
\addplot [semithick, white!29.8039215686275!black]
table {%
1.0628 1.76346695159443e-05
1.3372 1.76346695159443e-05
};
\addplot [semithick, white!29.8039215686275!black]
table {%
1.6628 6.03880727402978e-05
1.9372 6.03880727402978e-05
};
\addplot [semithick, white!29.8039215686275!black]
table {%
2.0628 1.46518920992613e-05
2.3372 1.46518920992613e-05
};

\nextgroupplot[
axis background/.style={fill=color0},
axis line style={white},
log basis y={10},
ymajorticks = false,
scaled y ticks=manual:{}{\pgfmathparse{#1}},
tick align=outside,
tick pos=both,
x grid style={white},
xmin=-0.5, xmax=2.5,
xtick style={color=white!15!black},
xtick={0,1,2},
xticklabels={UR10,KUKA-IIWA,Schunk-LWA4D},
y grid style={white},
ymajorgrids,
ymin=8.65125545117774e-08, ymax=6.80487218436774,
ymode=log,
ytick style={color=white!15!black, draw=none},
]
\path [draw=white!29.8039215686275!black, fill=color1, semithick]
(axis cs:-0.3372,9.1092354624717e-05)
--(axis cs:-0.0628,9.1092354624717e-05)
--(axis cs:-0.0628,0.0677518155287633)
--(axis cs:-0.3372,0.0677518155287633)
--(axis cs:-0.3372,9.1092354624717e-05)
--cycle;
\path [draw=white!29.8039215686275!black, fill=color2, semithick]
(axis cs:0.0628,2.20125383078715e-05)
--(axis cs:0.3372,2.20125383078715e-05)
--(axis cs:0.3372,0.150768465073365)
--(axis cs:0.0628,0.150768465073365)
--(axis cs:0.0628,2.20125383078715e-05)
--cycle;
\path [draw=white!29.8039215686275!black, fill=color1, semithick]
(axis cs:0.6628,5.11349624272067e-05)
--(axis cs:0.9372,5.11349624272067e-05)
--(axis cs:0.9372,0.0221491396041431)
--(axis cs:0.6628,0.0221491396041431)
--(axis cs:0.6628,5.11349624272067e-05)
--cycle;
\path [draw=white!29.8039215686275!black, fill=color2, semithick]
(axis cs:1.0628,1.30662027556154e-05)
--(axis cs:1.3372,1.30662027556154e-05)
--(axis cs:1.3372,0.000471583751194843)
--(axis cs:1.0628,0.000471583751194843)
--(axis cs:1.0628,1.30662027556154e-05)
--cycle;
\path [draw=white!29.8039215686275!black, fill=color1, semithick]
(axis cs:1.6628,4.9290036738573e-05)
--(axis cs:1.9372,4.9290036738573e-05)
--(axis cs:1.9372,0.00639000492843643)
--(axis cs:1.6628,0.00639000492843643)
--(axis cs:1.6628,4.9290036738573e-05)
--cycle;
\path [draw=white!29.8039215686275!black, fill=color2, semithick]
(axis cs:2.0628,1.1485344882615e-05)
--(axis cs:2.3372,1.1485344882615e-05)
--(axis cs:2.3372,0.0062105284792275)
--(axis cs:2.0628,0.0062105284792275)
--(axis cs:2.0628,1.1485344882615e-05)
--cycle;
\draw[draw=white!29.8039215686275!black,fill=color1,line width=0.3pt] (axis cs:0,0) rectangle (axis cs:0,0);
\draw[draw=white!29.8039215686275!black,fill=color2,line width=0.3pt] (axis cs:0,0) rectangle (axis cs:0,0);
\addplot [semithick, white!29.8039215686275!black]
table {%
-0.2 9.1092354624717e-05
-0.2 1.23121830632239e-06
};
\addplot [semithick, white!29.8039215686275!black]
table {%
-0.2 0.0677518155287633
-0.2 0.168909874565572
};
\addplot [semithick, white!29.8039215686275!black]
table {%
-0.298 1.23121830632239e-06
-0.102 1.23121830632239e-06
};
\addplot [semithick, white!29.8039215686275!black]
table {%
-0.298 0.168909874565572
-0.102 0.168909874565572
};
\addplot [semithick, white!29.8039215686275!black]
table {%
0.2 2.20125383078715e-05
0.2 8.11805634014404e-07
};
\addplot [semithick, white!29.8039215686275!black]
table {%
0.2 0.150768465073365
0.2 0.376033066175639
};
\addplot [semithick, white!29.8039215686275!black]
table {%
0.102 8.11805634014404e-07
0.298 8.11805634014404e-07
};
\addplot [semithick, white!29.8039215686275!black]
table {%
0.102 0.376033066175639
0.298 0.376033066175639
};
\addplot [semithick, white!29.8039215686275!black]
table {%
0.8 5.11349624272067e-05
0.8 7.83410038238481e-07
};
\addplot [semithick, white!29.8039215686275!black]
table {%
0.8 0.0221491396041431
0.8 0.055099671973129
};
\addplot [semithick, white!29.8039215686275!black]
table {%
0.702 7.83410038238481e-07
0.898 7.83410038238481e-07
};
\addplot [semithick, white!29.8039215686275!black]
table {%
0.702 0.055099671973129
0.898 0.055099671973129
};
\addplot [semithick, white!29.8039215686275!black]
table {%
1.2 1.30662027556154e-05
1.2 2.63737924726149e-06
};
\addplot [semithick, white!29.8039215686275!black]
table {%
1.2 0.000471583751194843
1.2 0.00111612317164174
};
\addplot [semithick, white!29.8039215686275!black]
table {%
1.102 2.63737924726149e-06
1.298 2.63737924726149e-06
};
\addplot [semithick, white!29.8039215686275!black]
table {%
1.102 0.00111612317164174
1.298 0.00111612317164174
};
\addplot [semithick, white!29.8039215686275!black]
table {%
1.8 4.9290036738573e-05
1.8 5.69567526984711e-07
};
\addplot [semithick, white!29.8039215686275!black]
table {%
1.8 0.00639000492843643
1.8 0.0158998153994043
};
\addplot [semithick, white!29.8039215686275!black]
table {%
1.702 5.69567526984711e-07
1.898 5.69567526984711e-07
};
\addplot [semithick, white!29.8039215686275!black]
table {%
1.702 0.0158998153994043
1.898 0.0158998153994043
};
\addplot [semithick, white!29.8039215686275!black]
table {%
2.2 1.1485344882615e-05
2.2 2.89059483209507e-06
};
\addplot [semithick, white!29.8039215686275!black]
table {%
2.2 0.0062105284792275
2.2 0.0154667367120264
};
\addplot [semithick, white!29.8039215686275!black]
table {%
2.102 2.89059483209507e-06
2.298 2.89059483209507e-06
};
\addplot [semithick, white!29.8039215686275!black]
table {%
2.102 0.0154667367120264
2.298 0.0154667367120264
};
\addplot [semithick, black, dashed]
table {%
-0.5 0.01
2.5 0.01
};
\addplot [semithick, white!29.8039215686275!black]
table {%
-0.3372 0.0109778799010987
-0.0628 0.0109778799010987
};
\addplot [semithick, white!29.8039215686275!black]
table {%
0.0628 0.024205903327993
0.3372 0.024205903327993
};
\addplot [semithick, white!29.8039215686275!black]
table {%
0.6628 0.000143967332956459
0.9372 0.000143967332956459
};
\addplot [semithick, white!29.8039215686275!black]
table {%
1.0628 1.88609081118391e-05
1.3372 1.88609081118391e-05
};
\addplot [semithick, white!29.8039215686275!black]
table {%
1.6628 0.000118839144983385
1.9372 0.000118839144983385
};
\addplot [semithick, white!29.8039215686275!black]
table {%
2.0628 1.83158307541718e-05
2.3372 1.83158307541718e-05
};

\nextgroupplot[
axis background/.style={fill=color0},
axis line style={white},
log basis y={10},
ymajorticks = false,
scaled y ticks=manual:{}{\pgfmathparse{#1}},
tick align=outside,
tick pos=both,
x grid style={white},
xmin=-0.5, xmax=2.5,
xtick style={color=white!15!black},
xtick={0,1,2},
xticklabels={UR10,KUKA-IIWA,Schunk-LWA4D},
y grid style={white},
ymajorgrids,
ymin=8.65125545117774e-08, ymax=6.80487218436774,
ymode=log,
ytick style={color=white!15!black, draw=none},
]
\path [draw=white!29.8039215686275!black, fill=color1, semithick]
(axis cs:-0.3372,0.000148285351167755)
--(axis cs:-0.0628,0.000148285351167755)
--(axis cs:-0.0628,0.0625674570775265)
--(axis cs:-0.3372,0.0625674570775265)
--(axis cs:-0.3372,0.000148285351167755)
--cycle;
\path [draw=white!29.8039215686275!black, fill=color2, semithick]
(axis cs:0.0628,3.32160708649894e-05)
--(axis cs:0.3372,3.32160708649894e-05)
--(axis cs:0.3372,0.187447367626966)
--(axis cs:0.0628,0.187447367626966)
--(axis cs:0.0628,3.32160708649894e-05)
--cycle;
\path [draw=white!29.8039215686275!black, fill=color1, semithick]
(axis cs:0.6628,5.01198579699401e-05)
--(axis cs:0.9372,5.01198579699401e-05)
--(axis cs:0.9372,0.0139576787798233)
--(axis cs:0.6628,0.0139576787798233)
--(axis cs:0.6628,5.01198579699401e-05)
--cycle;
\path [draw=white!29.8039215686275!black, fill=color2, semithick]
(axis cs:1.0628,1.51062830588194e-05)
--(axis cs:1.3372,1.51062830588194e-05)
--(axis cs:1.3372,0.0517294340570892)
--(axis cs:1.0628,0.0517294340570892)
--(axis cs:1.0628,1.51062830588194e-05)
--cycle;
\path [draw=white!29.8039215686275!black, fill=color1, semithick]
(axis cs:1.6628,4.3487295229251e-05)
--(axis cs:1.9372,4.3487295229251e-05)
--(axis cs:1.9372,0.00335475117323998)
--(axis cs:1.6628,0.00335475117323998)
--(axis cs:1.6628,4.3487295229251e-05)
--cycle;
\path [draw=white!29.8039215686275!black, fill=color2, semithick]
(axis cs:2.0628,1.15321293357576e-05)
--(axis cs:2.3372,1.15321293357576e-05)
--(axis cs:2.3372,0.0159666377343013)
--(axis cs:2.0628,0.0159666377343013)
--(axis cs:2.0628,1.15321293357576e-05)
--cycle;
\draw[draw=white!29.8039215686275!black,fill=color1,line width=0.3pt] (axis cs:0,0) rectangle (axis cs:0,0);
\draw[draw=white!29.8039215686275!black,fill=color2,line width=0.3pt] (axis cs:0,0) rectangle (axis cs:0,0);
\addplot [semithick, white!29.8039215686275!black]
table {%
-0.2 0.000148285351167755
-0.2 6.18353304900132e-07
};
\addplot [semithick, white!29.8039215686275!black]
table {%
-0.2 0.0625674570775265
-0.2 0.156118840728108
};
\addplot [semithick, white!29.8039215686275!black]
table {%
-0.298 6.18353304900132e-07
-0.102 6.18353304900132e-07
};
\addplot [semithick, white!29.8039215686275!black]
table {%
-0.298 0.156118840728108
-0.102 0.156118840728108
};
\addplot [semithick, white!29.8039215686275!black]
table {%
0.2 3.32160708649894e-05
0.2 7.23895895570003e-07
};
\addplot [semithick, white!29.8039215686275!black]
table {%
0.2 0.187447367626966
0.2 0.46737222855999
};
\addplot [semithick, white!29.8039215686275!black]
table {%
0.102 7.23895895570003e-07
0.298 7.23895895570003e-07
};
\addplot [semithick, white!29.8039215686275!black]
table {%
0.102 0.46737222855999
0.298 0.46737222855999
};
\addplot [semithick, white!29.8039215686275!black]
table {%
0.8 5.01198579699401e-05
0.8 3.02460538300641e-07
};
\addplot [semithick, white!29.8039215686275!black]
table {%
0.8 0.0139576787798233
0.8 0.0346963160996335
};
\addplot [semithick, white!29.8039215686275!black]
table {%
0.702 3.02460538300641e-07
0.898 3.02460538300641e-07
};
\addplot [semithick, white!29.8039215686275!black]
table {%
0.702 0.0346963160996335
0.898 0.0346963160996335
};
\addplot [semithick, white!29.8039215686275!black]
table {%
1.2 1.51062830588194e-05
1.2 4.60293950783345e-06
};
\addplot [semithick, white!29.8039215686275!black]
table {%
1.2 0.0517294340570892
1.2 0.129279689550576
};
\addplot [semithick, white!29.8039215686275!black]
table {%
1.102 4.60293950783345e-06
1.298 4.60293950783345e-06
};
\addplot [semithick, white!29.8039215686275!black]
table {%
1.102 0.129279689550576
1.298 0.129279689550576
};
\addplot [semithick, white!29.8039215686275!black]
table {%
1.8 4.3487295229251e-05
1.8 8.90585946886803e-07
};
\addplot [semithick, white!29.8039215686275!black]
table {%
1.8 0.00335475117323998
1.8 0.00829291503014714
};
\addplot [semithick, white!29.8039215686275!black]
table {%
1.702 8.90585946886803e-07
1.898 8.90585946886803e-07
};
\addplot [semithick, white!29.8039215686275!black]
table {%
1.702 0.00829291503014714
1.898 0.00829291503014714
};
\addplot [semithick, white!29.8039215686275!black]
table {%
2.2 1.15321293357576e-05
2.2 2.68092555227157e-06
};
\addplot [semithick, white!29.8039215686275!black]
table {%
2.2 0.0159666377343013
2.2 0.0394689402168295
};
\addplot [semithick, white!29.8039215686275!black]
table {%
2.102 2.68092555227157e-06
2.298 2.68092555227157e-06
};
\addplot [semithick, white!29.8039215686275!black]
table {%
2.102 0.0394689402168295
2.298 0.0394689402168295
};
\addplot [semithick, black, dashed]
table {%
-0.499999999999999 0.01
2.5 0.01
};
\addplot [semithick, white!29.8039215686275!black]
table {%
-0.3372 0.0132342554894136
-0.0628 0.0132342554894136
};
\addplot [semithick, white!29.8039215686275!black]
table {%
0.0628 0.0690071584159176
0.3372 0.0690071584159176
};
\addplot [semithick, white!29.8039215686275!black]
table {%
0.6628 0.000126096626522283
0.9372 0.000126096626522283
};
\addplot [semithick, white!29.8039215686275!black]
table {%
1.0628 2.90303186824052e-05
1.3372 2.90303186824052e-05
};
\addplot [semithick, white!29.8039215686275!black]
table {%
1.6628 0.000101578922549197
1.9372 0.000101578922549197
};
\addplot [semithick, white!29.8039215686275!black]
table {%
2.0628 1.82139008513373e-05
2.3372 1.82139008513373e-05
};

\nextgroupplot[
axis background/.style={fill=color0},
axis line style={white},
log basis y={10},
ymajorticks = false,
scaled y ticks=manual:{}{\pgfmathparse{#1}},
tick align=outside,
tick pos=both,
x grid style={white},
xmin=-0.5, xmax=2.5,
xtick style={color=white!15!black},
xtick={0,1,2},
xticklabels={UR10,KUKA-IIWA,Schunk-LWA4D},
y grid style={white},
ymajorgrids,
ymin=8.65125545117774e-08, ymax=6.80487218436774,
ymode=log,
ytick style={color=white!15!black, draw=none},
]
\path [draw=white!29.8039215686275!black, fill=color1, semithick]
(axis cs:-0.3372,0.0749136077548679)
--(axis cs:-0.0628,0.0749136077548679)
--(axis cs:-0.0628,0.305843115368692)
--(axis cs:-0.3372,0.305843115368692)
--(axis cs:-0.3372,0.0749136077548679)
--cycle;
\path [draw=white!29.8039215686275!black, fill=color2, semithick]
(axis cs:0.0628,0.0565046177477226)
--(axis cs:0.3372,0.0565046177477226)
--(axis cs:0.3372,0.241354933842535)
--(axis cs:0.0628,0.241354933842535)
--(axis cs:0.0628,0.0565046177477226)
--cycle;
\path [draw=white!29.8039215686275!black, fill=color1, semithick]
(axis cs:0.6628,8.146038881427e-05)
--(axis cs:0.9372,8.146038881427e-05)
--(axis cs:0.9372,0.0502612462296052)
--(axis cs:0.6628,0.0502612462296052)
--(axis cs:0.6628,8.146038881427e-05)
--cycle;
\path [draw=white!29.8039215686275!black, fill=color2, semithick]
(axis cs:1.0628,1.7325826535879e-05)
--(axis cs:1.3372,1.7325826535879e-05)
--(axis cs:1.3372,0.0993832480304262)
--(axis cs:1.0628,0.0993832480304262)
--(axis cs:1.0628,1.7325826535879e-05)
--cycle;
\path [draw=white!29.8039215686275!black, fill=color1, semithick]
(axis cs:1.6628,5.61030138040788e-05)
--(axis cs:1.9372,5.61030138040788e-05)
--(axis cs:1.9372,0.0243665098999752)
--(axis cs:1.6628,0.0243665098999752)
--(axis cs:1.6628,5.61030138040788e-05)
--cycle;
\path [draw=white!29.8039215686275!black, fill=color2, semithick]
(axis cs:2.0628,1.24982550895883e-05)
--(axis cs:2.3372,1.24982550895883e-05)
--(axis cs:2.3372,0.0532098439140134)
--(axis cs:2.0628,0.0532098439140134)
--(axis cs:2.0628,1.24982550895883e-05)
--cycle;
\draw[draw=white!29.8039215686275!black,fill=color1,line width=0.3pt] (axis cs:0,0) rectangle (axis cs:0,0);
\draw[draw=white!29.8039215686275!black,fill=color2,line width=0.3pt] (axis cs:0,0) rectangle (axis cs:0,0);
\addplot [semithick, white!29.8039215686275!black]
table {%
-0.2 0.0749136077548679
-0.2 3.30322144065849e-06
};
\addplot [semithick, white!29.8039215686275!black]
table {%
-0.2 0.305843115368692
-0.2 0.65222091391119
};
\addplot [semithick, white!29.8039215686275!black]
table {%
-0.298 3.30322144065849e-06
-0.102 3.30322144065849e-06
};
\addplot [semithick, white!29.8039215686275!black]
table {%
-0.298 0.65222091391119
-0.102 0.65222091391119
};
\addplot [semithick, white!29.8039215686275!black]
table {%
0.2 0.0565046177477226
0.2 3.15514274995795e-06
};
\addplot [semithick, white!29.8039215686275!black]
table {%
0.2 0.241354933842535
0.2 0.518380234236541
};
\addplot [semithick, white!29.8039215686275!black]
table {%
0.102 3.15514274995795e-06
0.298 3.15514274995795e-06
};
\addplot [semithick, white!29.8039215686275!black]
table {%
0.102 0.518380234236541
0.298 0.518380234236541
};
\addplot [semithick, white!29.8039215686275!black]
table {%
0.8 8.146038881427e-05
0.8 3.40125328615707e-07
};
\addplot [semithick, white!29.8039215686275!black]
table {%
0.8 0.0502612462296052
0.8 0.125378253959124
};
\addplot [semithick, white!29.8039215686275!black]
table {%
0.702 3.40125328615707e-07
0.898 3.40125328615707e-07
};
\addplot [semithick, white!29.8039215686275!black]
table {%
0.702 0.125378253959124
0.898 0.125378253959124
};
\addplot [semithick, white!29.8039215686275!black]
table {%
1.2 1.7325826535879e-05
1.2 2.35716400554599e-06
};
\addplot [semithick, white!29.8039215686275!black]
table {%
1.2 0.0993832480304262
1.2 0.247745110277109
};
\addplot [semithick, white!29.8039215686275!black]
table {%
1.102 2.35716400554599e-06
1.298 2.35716400554599e-06
};
\addplot [semithick, white!29.8039215686275!black]
table {%
1.102 0.247745110277109
1.298 0.247745110277109
};
\addplot [semithick, white!29.8039215686275!black]
table {%
1.8 5.61030138040788e-05
1.8 9.89104913191603e-07
};
\addplot [semithick, white!29.8039215686275!black]
table {%
1.8 0.0243665098999752
1.8 0.0607213196933888
};
\addplot [semithick, white!29.8039215686275!black]
table {%
1.702 9.89104913191603e-07
1.898 9.89104913191603e-07
};
\addplot [semithick, white!29.8039215686275!black]
table {%
1.702 0.0607213196933888
1.898 0.0607213196933888
};
\addplot [semithick, white!29.8039215686275!black]
table {%
2.2 1.24982550895883e-05
2.2 3.3928213051286e-06
};
\addplot [semithick, white!29.8039215686275!black]
table {%
2.2 0.0532098439140134
2.2 0.132964227630425
};
\addplot [semithick, white!29.8039215686275!black]
table {%
2.102 3.3928213051286e-06
2.298 3.3928213051286e-06
};
\addplot [semithick, white!29.8039215686275!black]
table {%
2.102 0.132964227630425
2.298 0.132964227630425
};
\addplot [semithick, black, dashed]
table {%
-0.5 0.01
2.5 0.01
};
\addplot [semithick, white!29.8039215686275!black]
table {%
-0.3372 0.16019836838353
-0.0628 0.16019836838353
};
\addplot [semithick, white!29.8039215686275!black]
table {%
0.0628 0.136420310338634
0.3372 0.136420310338634
};
\addplot [semithick, white!29.8039215686275!black]
table {%
0.6628 0.00753608026971334
0.9372 0.00753608026971334
};
\addplot [semithick, white!29.8039215686275!black]
table {%
1.0628 0.0136129940842778
1.3372 0.0136129940842778
};
\addplot [semithick, white!29.8039215686275!black]
table {%
1.6628 0.00024722613863962
1.9372 0.00024722613863962
};
\addplot [semithick, white!29.8039215686275!black]
table {%
2.0628 2.69507871222772e-05
2.3372 2.69507871222772e-05
};
\end{groupplot}

\end{tikzpicture}

%% file: figs/results/exp1_soltime.tex
\begin{tikzpicture}

\definecolor{color0}{rgb}{0.917647058823529,0.917647058823529,0.949019607843137}
\definecolor{color1}{rgb}{0.347058823529412,0.458823529411765,0.641176470588235}
\definecolor{color2}{rgb}{0.71078431372549,0.363725490196079,0.375490196078431}

\begin{groupplot}[group style={group size=4 by 1}]
\nextgroupplot[
axis background/.style={fill=color0},
axis line style={white},
tick align=outside,
tick pos=both,
x grid style={white},
xmin=-0.5, xmax=2.5,
xtick style={color=white!15!black},
xtick={0,1,2},
xticklabels={UR10,KUKA-IIWA,Schunk-LWA4D},
y grid style={white},
ylabel={Sol. Time [s]},
ymajorgrids,
ymin=-0.139595532417297, ymax=3.08485510349274,
ytick style={color=white!15!black}
]
\path [draw=white!29.8039215686275!black, fill=color1, semithick]
(axis cs:-0.3372,0.0565568208694458)
--(axis cs:-0.0628,0.0565568208694458)
--(axis cs:-0.0628,0.0852171778678894)
--(axis cs:-0.3372,0.0852171778678894)
--(axis cs:-0.3372,0.0565568208694458)
--cycle;
\path [draw=white!29.8039215686275!black, fill=color2, semithick]
(axis cs:0.0628,0.117690086364746)
--(axis cs:0.3372,0.117690086364746)
--(axis cs:0.3372,0.444899797439575)
--(axis cs:0.0628,0.444899797439575)
--(axis cs:0.0628,0.117690086364746)
--cycle;
\path [draw=white!29.8039215686275!black, fill=color1, semithick]
(axis cs:0.6628,0.0751581788063049)
--(axis cs:0.9372,0.0751581788063049)
--(axis cs:0.9372,0.118743240833282)
--(axis cs:0.6628,0.118743240833282)
--(axis cs:0.6628,0.0751581788063049)
--cycle;
\path [draw=white!29.8039215686275!black, fill=color2, semithick]
(axis cs:1.0628,0.190719366073608)
--(axis cs:1.3372,0.190719366073608)
--(axis cs:1.3372,0.447404205799103)
--(axis cs:1.0628,0.447404205799103)
--(axis cs:1.0628,0.190719366073608)
--cycle;
\path [draw=white!29.8039215686275!black, fill=color1, semithick]
(axis cs:1.6628,0.082186758518219)
--(axis cs:1.9372,0.082186758518219)
--(axis cs:1.9372,0.129732847213745)
--(axis cs:1.6628,0.129732847213745)
--(axis cs:1.6628,0.082186758518219)
--cycle;
\path [draw=white!29.8039215686275!black, fill=color2, semithick]
(axis cs:2.0628,0.176780104637146)
--(axis cs:2.3372,0.176780104637146)
--(axis cs:2.3372,0.500215590000153)
--(axis cs:2.0628,0.500215590000153)
--(axis cs:2.0628,0.176780104637146)
--cycle;
\draw[draw=white!29.8039215686275!black,fill=color1,line width=0.3pt] (axis cs:0,0) rectangle (axis cs:0,0);

\draw[draw=white!29.8039215686275!black,fill=color2,line width=0.3pt] (axis cs:0,0) rectangle (axis cs:0,0);

\addplot [semithick, white!29.8039215686275!black]
table {%
-0.2 0.0565568208694458
-0.2 0.025972843170166
};
\addplot [semithick, white!29.8039215686275!black]
table {%
-0.2 0.0852171778678894
-0.2 0.128201723098755
};
\addplot [semithick, white!29.8039215686275!black]
table {%
-0.298 0.025972843170166
-0.102 0.025972843170166
};
\addplot [semithick, white!29.8039215686275!black]
table {%
-0.298 0.128201723098755
-0.102 0.128201723098755
};
\addplot [semithick, white!29.8039215686275!black]
table {%
0.2 0.117690086364746
0.2 0.0119962692260742
};
\addplot [semithick, white!29.8039215686275!black]
table {%
0.2 0.444899797439575
0.2 0.935462474822998
};
\addplot [semithick, white!29.8039215686275!black]
table {%
0.102 0.0119962692260742
0.298 0.0119962692260742
};
\addplot [semithick, white!29.8039215686275!black]
table {%
0.102 0.935462474822998
0.298 0.935462474822998
};
\addplot [semithick, white!29.8039215686275!black]
table {%
0.8 0.0751581788063049
0.8 0.0341262817382812
};
\addplot [semithick, white!29.8039215686275!black]
table {%
0.8 0.118743240833282
0.8 0.183897972106934
};
\addplot [semithick, white!29.8039215686275!black]
table {%
0.702 0.0341262817382812
0.898 0.0341262817382812
};
\addplot [semithick, white!29.8039215686275!black]
table {%
0.702 0.183897972106934
0.898 0.183897972106934
};
\addplot [semithick, white!29.8039215686275!black]
table {%
1.2 0.190719366073608
1.2 0.0168032646179199
};
\addplot [semithick, white!29.8039215686275!black]
table {%
1.2 0.447404205799103
1.2 0.832048892974854
};
\addplot [semithick, white!29.8039215686275!black]
table {%
1.102 0.0168032646179199
1.298 0.0168032646179199
};
\addplot [semithick, white!29.8039215686275!black]
table {%
1.102 0.832048892974854
1.298 0.832048892974854
};
\addplot [semithick, white!29.8039215686275!black]
table {%
1.8 0.082186758518219
1.8 0.0317087173461914
};
\addplot [semithick, white!29.8039215686275!black]
table {%
1.8 0.129732847213745
1.8 0.200841665267944
};
\addplot [semithick, white!29.8039215686275!black]
table {%
1.702 0.0317087173461914
1.898 0.0317087173461914
};
\addplot [semithick, white!29.8039215686275!black]
table {%
1.702 0.200841665267944
1.898 0.200841665267944
};
\addplot [semithick, white!29.8039215686275!black]
table {%
2.2 0.176780104637146
2.2 0.00697040557861328
};
\addplot [semithick, white!29.8039215686275!black]
table {%
2.2 0.500215590000153
2.2 0.985054969787598
};
\addplot [semithick, white!29.8039215686275!black]
table {%
2.102 0.00697040557861328
2.298 0.00697040557861328
};
\addplot [semithick, white!29.8039215686275!black]
table {%
2.102 0.985054969787598
2.298 0.985054969787598
};
\addplot [semithick, white!29.8039215686275!black]
table {%
-0.3372 0.0678942203521729
-0.0628 0.0678942203521729
};
\addplot [semithick, white!29.8039215686275!black]
table {%
0.0628 0.195088982582092
0.3372 0.195088982582092
};
\addplot [semithick, white!29.8039215686275!black]
table {%
0.6628 0.0915271043777466
0.9372 0.0915271043777466
};
\addplot [semithick, white!29.8039215686275!black]
table {%
1.0628 0.293885588645935
1.3372 0.293885588645935
};
\addplot [semithick, white!29.8039215686275!black]
table {%
1.6628 0.101571679115295
1.9372 0.101571679115295
};
\addplot [semithick, white!29.8039215686275!black]
table {%
2.0628 0.276691317558289
2.3372 0.276691317558289
};

\nextgroupplot[
axis background/.style={fill=color0},
axis line style={white},
scaled y ticks=manual:{}{\pgfmathparse{#1}},
tick align=outside,
tick pos=both,
x grid style={white},
xmin=-0.5, xmax=2.5,
xtick style={color=white!15!black},
xtick={0,1,2},
xticklabels={UR10,KUKA-IIWA,Schunk-LWA4D},
y grid style={white},
ymajorgrids,
ymin=-0.139595532417297, ymax=3.08485510349274,
ytick style={color=white!15!black},
yticklabels={}
]
\path [draw=white!29.8039215686275!black, fill=color1, semithick]
(axis cs:-0.3372,0.416788399219513)
--(axis cs:-0.0628,0.416788399219513)
--(axis cs:-0.0628,0.690722286701202)
--(axis cs:-0.3372,0.690722286701202)
--(axis cs:-0.3372,0.416788399219513)
--cycle;
\path [draw=white!29.8039215686275!black, fill=color2, semithick]
(axis cs:0.0628,0.0432780981063843)
--(axis cs:0.3372,0.0432780981063843)
--(axis cs:0.3372,0.368590891361237)
--(axis cs:0.0628,0.368590891361237)
--(axis cs:0.0628,0.0432780981063843)
--cycle;
\path [draw=white!29.8039215686275!black, fill=color1, semithick]
(axis cs:0.6628,0.646153748035431)
--(axis cs:0.9372,0.646153748035431)
--(axis cs:0.9372,1.06259942054749)
--(axis cs:0.6628,1.06259942054749)
--(axis cs:0.6628,0.646153748035431)
--cycle;
\path [draw=white!29.8039215686275!black, fill=color2, semithick]
(axis cs:1.0628,0.218245506286621)
--(axis cs:1.3372,0.218245506286621)
--(axis cs:1.3372,0.667793393135071)
--(axis cs:1.0628,0.667793393135071)
--(axis cs:1.0628,0.218245506286621)
--cycle;
\path [draw=white!29.8039215686275!black, fill=color1, semithick]
(axis cs:1.6628,0.680405020713806)
--(axis cs:1.9372,0.680405020713806)
--(axis cs:1.9372,1.14421474933624)
--(axis cs:1.6628,1.14421474933624)
--(axis cs:1.6628,0.680405020713806)
--cycle;
\path [draw=white!29.8039215686275!black, fill=color2, semithick]
(axis cs:2.0628,0.223277926445007)
--(axis cs:2.3372,0.223277926445007)
--(axis cs:2.3372,0.880723714828491)
--(axis cs:2.0628,0.880723714828491)
--(axis cs:2.0628,0.223277926445007)
--cycle;
\draw[draw=white!29.8039215686275!black,fill=color1,line width=0.3pt] (axis cs:0,0) rectangle (axis cs:0,0);
\draw[draw=white!29.8039215686275!black,fill=color2,line width=0.3pt] (axis cs:0,0) rectangle (axis cs:0,0);
\addplot [semithick, white!29.8039215686275!black]
table {%
-0.2 0.416788399219513
-0.2 0.105473279953003
};
\addplot [semithick, white!29.8039215686275!black]
table {%
-0.2 0.690722286701202
-0.2 1.10159254074097
};
\addplot [semithick, white!29.8039215686275!black]
table {%
-0.298 0.105473279953003
-0.102 0.105473279953003
};
\addplot [semithick, white!29.8039215686275!black]
table {%
-0.298 1.10159254074097
-0.102 1.10159254074097
};
\addplot [semithick, white!29.8039215686275!black]
table {%
0.2 0.0432780981063843
0.2 0.0101883411407471
};
\addplot [semithick, white!29.8039215686275!black]
table {%
0.2 0.368590891361237
0.2 0.852591276168823
};
\addplot [semithick, white!29.8039215686275!black]
table {%
0.102 0.0101883411407471
0.298 0.0101883411407471
};
\addplot [semithick, white!29.8039215686275!black]
table {%
0.102 0.852591276168823
0.298 0.852591276168823
};
\addplot [semithick, white!29.8039215686275!black]
table {%
0.8 0.646153748035431
0.8 0.0611188411712646
};
\addplot [semithick, white!29.8039215686275!black]
table {%
0.8 1.06259942054749
0.8 1.68047404289246
};
\addplot [semithick, white!29.8039215686275!black]
table {%
0.702 0.0611188411712646
0.898 0.0611188411712646
};
\addplot [semithick, white!29.8039215686275!black]
table {%
0.702 1.68047404289246
0.898 1.68047404289246
};
\addplot [semithick, white!29.8039215686275!black]
table {%
1.2 0.218245506286621
1.2 0.0181562900543213
};
\addplot [semithick, white!29.8039215686275!black]
table {%
1.2 0.667793393135071
1.2 1.34162402153015
};
\addplot [semithick, white!29.8039215686275!black]
table {%
1.102 0.0181562900543213
1.298 0.0181562900543213
};
\addplot [semithick, white!29.8039215686275!black]
table {%
1.102 1.34162402153015
1.298 1.34162402153015
};
\addplot [semithick, white!29.8039215686275!black]
table {%
1.8 0.680405020713806
1.8 0.194380521774292
};
\addplot [semithick, white!29.8039215686275!black]
table {%
1.8 1.14421474933624
1.8 1.83095860481262
};
\addplot [semithick, white!29.8039215686275!black]
table {%
1.702 0.194380521774292
1.898 0.194380521774292
};
\addplot [semithick, white!29.8039215686275!black]
table {%
1.702 1.83095860481262
1.898 1.83095860481262
};
\addplot [semithick, white!29.8039215686275!black]
table {%
2.2 0.223277926445007
2.2 0.0115914344787598
};
\addplot [semithick, white!29.8039215686275!black]
table {%
2.2 0.880723714828491
2.2 1.86629867553711
};
\addplot [semithick, white!29.8039215686275!black]
table {%
2.102 0.0115914344787598
2.298 0.0115914344787598
};
\addplot [semithick, white!29.8039215686275!black]
table {%
2.102 1.86629867553711
2.298 1.86629867553711
};
\addplot [semithick, white!29.8039215686275!black]
table {%
-0.3372 0.54092276096344
-0.0628 0.54092276096344
};
\addplot [semithick, white!29.8039215686275!black]
table {%
0.0628 0.188432097434998
0.3372 0.188432097434998
};
\addplot [semithick, white!29.8039215686275!black]
table {%
0.6628 0.825693249702454
0.9372 0.825693249702454
};
\addplot [semithick, white!29.8039215686275!black]
table {%
1.0628 0.330649018287659
1.3372 0.330649018287659
};
\addplot [semithick, white!29.8039215686275!black]
table {%
1.6628 0.871965169906616
1.9372 0.871965169906616
};
\addplot [semithick, white!29.8039215686275!black]
table {%
2.0628 0.374291658401489
2.3372 0.374291658401489
};

\nextgroupplot[
axis background/.style={fill=color0},
axis line style={white},
scaled y ticks=manual:{}{\pgfmathparse{#1}},
tick align=outside,
tick pos=both,
x grid style={white},
xmin=-0.5, xmax=2.5,
xtick style={color=white!15!black},
xtick={0,1,2},
xticklabels={UR10,KUKA-IIWA,Schunk-LWA4D},
y grid style={white},
ymajorgrids,
ymin=-0.139595532417297, ymax=3.08485510349274,
ytick style={color=white!15!black},
yticklabels={}
]
\path [draw=white!29.8039215686275!black, fill=color1, semithick]
(axis cs:-0.3372,0.365969479084015)
--(axis cs:-0.0628,0.365969479084015)
--(axis cs:-0.0628,0.607787847518921)
--(axis cs:-0.3372,0.607787847518921)
--(axis cs:-0.3372,0.365969479084015)
--cycle;
\path [draw=white!29.8039215686275!black, fill=color2, semithick]
(axis cs:0.0628,0.0374237298965454)
--(axis cs:0.3372,0.0374237298965454)
--(axis cs:0.3372,0.31050181388855)
--(axis cs:0.0628,0.31050181388855)
--(axis cs:0.0628,0.0374237298965454)
--cycle;
\path [draw=white!29.8039215686275!black, fill=color1, semithick]
(axis cs:0.6628,0.549589574337006)
--(axis cs:0.9372,0.549589574337006)
--(axis cs:0.9372,0.935479879379272)
--(axis cs:0.6628,0.935479879379272)
--(axis cs:0.6628,0.549589574337006)
--cycle;
\path [draw=white!29.8039215686275!black, fill=color2, semithick]
(axis cs:1.0628,0.213026940822601)
--(axis cs:1.3372,0.213026940822601)
--(axis cs:1.3372,0.98656690120697)
--(axis cs:1.0628,0.98656690120697)
--(axis cs:1.0628,0.213026940822601)
--cycle;
\path [draw=white!29.8039215686275!black, fill=color1, semithick]
(axis cs:1.6628,0.550035715103149)
--(axis cs:1.9372,0.550035715103149)
--(axis cs:1.9372,0.90835839509964)
--(axis cs:1.6628,0.90835839509964)
--(axis cs:1.6628,0.550035715103149)
--cycle;
\path [draw=white!29.8039215686275!black, fill=color2, semithick]
(axis cs:2.0628,0.2603480219841)
--(axis cs:2.3372,0.2603480219841)
--(axis cs:2.3372,1.07628679275513)
--(axis cs:2.0628,1.07628679275513)
--(axis cs:2.0628,0.2603480219841)
--cycle;
\draw[draw=white!29.8039215686275!black,fill=color1,line width=0.3pt] (axis cs:0,0) rectangle (axis cs:0,0);
\draw[draw=white!29.8039215686275!black,fill=color2,line width=0.3pt] (axis cs:0,0) rectangle (axis cs:0,0);
\addplot [semithick, white!29.8039215686275!black]
table {%
-0.2 0.365969479084015
-0.2 0.163259029388428
};
\addplot [semithick, white!29.8039215686275!black]
table {%
-0.2 0.607787847518921
-0.2 0.969435214996338
};
\addplot [semithick, white!29.8039215686275!black]
table {%
-0.298 0.163259029388428
-0.102 0.163259029388428
};
\addplot [semithick, white!29.8039215686275!black]
table {%
-0.298 0.969435214996338
-0.102 0.969435214996338
};
\addplot [semithick, white!29.8039215686275!black]
table {%
0.2 0.0374237298965454
0.2 0.0112881660461426
};
\addplot [semithick, white!29.8039215686275!black]
table {%
0.2 0.31050181388855
0.2 0.717166900634766
};
\addplot [semithick, white!29.8039215686275!black]
table {%
0.102 0.0112881660461426
0.298 0.0112881660461426
};
\addplot [semithick, white!29.8039215686275!black]
table {%
0.102 0.717166900634766
0.298 0.717166900634766
};
\addplot [semithick, white!29.8039215686275!black]
table {%
0.8 0.549589574337006
0.8 0.133914947509766
};
\addplot [semithick, white!29.8039215686275!black]
table {%
0.8 0.935479879379272
0.8 1.51327133178711
};
\addplot [semithick, white!29.8039215686275!black]
table {%
0.702 0.133914947509766
0.898 0.133914947509766
};
\addplot [semithick, white!29.8039215686275!black]
table {%
0.702 1.51327133178711
0.898 1.51327133178711
};
\addplot [semithick, white!29.8039215686275!black]
table {%
1.2 0.213026940822601
1.2 0.01409912109375
};
\addplot [semithick, white!29.8039215686275!black]
table {%
1.2 0.98656690120697
1.2 2.14633750915527
};
\addplot [semithick, white!29.8039215686275!black]
table {%
1.102 0.01409912109375
1.298 0.01409912109375
};
\addplot [semithick, white!29.8039215686275!black]
table {%
1.102 2.14633750915527
1.298 2.14633750915527
};
\addplot [semithick, white!29.8039215686275!black]
table {%
1.8 0.550035715103149
1.8 0.131407260894775
};
\addplot [semithick, white!29.8039215686275!black]
table {%
1.8 0.90835839509964
1.8 1.43693375587463
};
\addplot [semithick, white!29.8039215686275!black]
table {%
1.702 0.131407260894775
1.898 0.131407260894775
};
\addplot [semithick, white!29.8039215686275!black]
table {%
1.702 1.43693375587463
1.898 1.43693375587463
};
\addplot [semithick, white!29.8039215686275!black]
table {%
2.2 0.2603480219841
2.2 0.0134665966033936
};
\addplot [semithick, white!29.8039215686275!black]
table {%
2.2 1.07628679275513
2.2 2.28703141212463
};
\addplot [semithick, white!29.8039215686275!black]
table {%
2.102 0.0134665966033936
2.298 0.0134665966033936
};
\addplot [semithick, white!29.8039215686275!black]
table {%
2.102 2.28703141212463
2.298 2.28703141212463
};
\addplot [semithick, white!29.8039215686275!black]
table {%
-0.3372 0.465184688568115
-0.0628 0.465184688568115
};
\addplot [semithick, white!29.8039215686275!black]
table {%
0.0628 0.0941095352172852
0.3372 0.0941095352172852
};
\addplot [semithick, white!29.8039215686275!black]
table {%
0.6628 0.710405468940735
0.9372 0.710405468940735
};
\addplot [semithick, white!29.8039215686275!black]
table {%
1.0628 0.385559320449829
1.3372 0.385559320449829
};
\addplot [semithick, white!29.8039215686275!black]
table {%
1.6628 0.704173684120178
1.9372 0.704173684120178
};
\addplot [semithick, white!29.8039215686275!black]
table {%
2.0628 0.448408722877502
2.3372 0.448408722877502
};

\nextgroupplot[
axis background/.style={fill=color0},
axis line style={white},
scaled y ticks=manual:{}{\pgfmathparse{#1}},
tick align=outside,
tick pos=both,
x grid style={white},
xmin=-0.5, xmax=2.5,
xtick style={color=white!15!black},
xtick={0,1,2},
xticklabels={UR10,KUKA-IIWA,Schunk-LWA4D},
y grid style={white},
ymajorgrids,
ymin=-0.139595532417297, ymax=3.08485510349274,
ytick style={color=white!15!black},
yticklabels={}
]
\path [draw=white!29.8039215686275!black, fill=color1, semithick]
(axis cs:-0.3372,0.366843938827515)
--(axis cs:-0.0628,0.366843938827515)
--(axis cs:-0.0628,1.04589033126831)
--(axis cs:-0.3372,1.04589033126831)
--(axis cs:-0.3372,0.366843938827515)
--cycle;
\path [draw=white!29.8039215686275!black, fill=color2, semithick]
(axis cs:0.0628,0.0383570194244385)
--(axis cs:0.3372,0.0383570194244385)
--(axis cs:0.3372,0.200863778591156)
--(axis cs:0.0628,0.200863778591156)
--(axis cs:0.0628,0.0383570194244385)
--cycle;
\path [draw=white!29.8039215686275!black, fill=color1, semithick]
(axis cs:0.6628,0.788287699222565)
--(axis cs:0.9372,0.788287699222565)
--(axis cs:0.9372,1.47399193048477)
--(axis cs:0.6628,1.47399193048477)
--(axis cs:0.6628,0.788287699222565)
--cycle;
\path [draw=white!29.8039215686275!black, fill=color2, semithick]
(axis cs:1.0628,0.181256949901581)
--(axis cs:1.3372,0.181256949901581)
--(axis cs:1.3372,0.860934019088745)
--(axis cs:1.0628,0.860934019088745)
--(axis cs:1.0628,0.181256949901581)
--cycle;
\path [draw=white!29.8039215686275!black, fill=color1, semithick]
(axis cs:1.6628,0.747448325157166)
--(axis cs:1.9372,0.747448325157166)
--(axis cs:1.9372,1.35981070995331)
--(axis cs:1.6628,1.35981070995331)
--(axis cs:1.6628,0.747448325157166)
--cycle;
\path [draw=white!29.8039215686275!black, fill=color2, semithick]
(axis cs:2.0628,0.301813304424286)
--(axis cs:2.3372,0.301813304424286)
--(axis cs:2.3372,1.35761082172394)
--(axis cs:2.0628,1.35761082172394)
--(axis cs:2.0628,0.301813304424286)
--cycle;
\draw[draw=white!29.8039215686275!black,fill=color1,line width=0.3pt] (axis cs:0,0) rectangle (axis cs:0,0);
\draw[draw=white!29.8039215686275!black,fill=color2,line width=0.3pt] (axis cs:0,0) rectangle (axis cs:0,0);
\addplot [semithick, white!29.8039215686275!black]
table {%
-0.2 0.366843938827515
-0.2 0.118753910064697
};
\addplot [semithick, white!29.8039215686275!black]
table {%
-0.2 1.04589033126831
-0.2 2.05925607681274
};
\addplot [semithick, white!29.8039215686275!black]
table {%
-0.298 0.118753910064697
-0.102 0.118753910064697
};
\addplot [semithick, white!29.8039215686275!black]
table {%
-0.298 2.05925607681274
-0.102 2.05925607681274
};
\addplot [semithick, white!29.8039215686275!black]
table {%
0.2 0.0383570194244385
0.2 0.0135374069213867
};
\addplot [semithick, white!29.8039215686275!black]
table {%
0.2 0.200863778591156
0.2 0.443918228149414
};
\addplot [semithick, white!29.8039215686275!black]
table {%
0.102 0.0135374069213867
0.298 0.0135374069213867
};
\addplot [semithick, white!29.8039215686275!black]
table {%
0.102 0.443918228149414
0.298 0.443918228149414
};
\addplot [semithick, white!29.8039215686275!black]
table {%
0.8 0.788287699222565
0.8 0.184430599212646
};
\addplot [semithick, white!29.8039215686275!black]
table {%
0.8 1.47399193048477
0.8 2.50151205062866
};
\addplot [semithick, white!29.8039215686275!black]
table {%
0.702 0.184430599212646
0.898 0.184430599212646
};
\addplot [semithick, white!29.8039215686275!black]
table {%
0.702 2.50151205062866
0.898 2.50151205062866
};
\addplot [semithick, white!29.8039215686275!black]
table {%
1.2 0.181256949901581
1.2 0.0198278427124023
};
\addplot [semithick, white!29.8039215686275!black]
table {%
1.2 0.860934019088745
1.2 1.88037586212158
};
\addplot [semithick, white!29.8039215686275!black]
table {%
1.102 0.0198278427124023
1.298 0.0198278427124023
};
\addplot [semithick, white!29.8039215686275!black]
table {%
1.102 1.88037586212158
1.298 1.88037586212158
};
\addplot [semithick, white!29.8039215686275!black]
table {%
1.8 0.747448325157166
1.8 0.147783994674683
};
\addplot [semithick, white!29.8039215686275!black]
table {%
1.8 1.35981070995331
1.8 2.27689981460571
};
\addplot [semithick, white!29.8039215686275!black]
table {%
1.702 0.147783994674683
1.898 0.147783994674683
};
\addplot [semithick, white!29.8039215686275!black]
table {%
1.702 2.27689981460571
1.898 2.27689981460571
};
\addplot [semithick, white!29.8039215686275!black]
table {%
2.2 0.301813304424286
2.2 0.0242786407470703
};
\addplot [semithick, white!29.8039215686275!black]
table {%
2.2 1.35761082172394
2.2 2.93828916549683
};
\addplot [semithick, white!29.8039215686275!black]
table {%
2.102 0.0242786407470703
2.298 0.0242786407470703
};
\addplot [semithick, white!29.8039215686275!black]
table {%
2.102 2.93828916549683
2.298 2.93828916549683
};
\addplot [semithick, white!29.8039215686275!black]
table {%
-0.3372 0.522754907608032
-0.0628 0.522754907608032
};
\addplot [semithick, white!29.8039215686275!black]
table {%
0.0628 0.0666451454162598
0.3372 0.0666451454162598
};
\addplot [semithick, white!29.8039215686275!black]
table {%
0.6628 1.0647314786911
0.9372 1.0647314786911
};
\addplot [semithick, white!29.8039215686275!black]
table {%
1.0628 0.395957231521606
1.3372 0.395957231521606
};
\addplot [semithick, white!29.8039215686275!black]
table {%
1.6628 1.01711559295654
1.9372 1.01711559295654
};
\addplot [semithick, white!29.8039215686275!black]
table {%
2.0628 0.532859921455383
2.3372 0.532859921455383
};
\end{groupplot}

\end{tikzpicture}

%% file: sections/conclusion.tex
\section{Conclusion}\label{sec:conclusion}
In this paper, we have presented a novel and elegant procedure for formulating many inverse kinematics problems in the language of distance geometry.
This distance-geometric perspective on IK allows us to leverage powerful low-rank matrix completion methods, resulting in an algorithm (\texttt{RTR-*}) that can efficiently compute IK solutions for a variety of robots using Riemannian optimization.
Our experiments show that that \texttt{RTR-*} outperforms competing algorithms in terms of success rate and number of iterations on IK problems for robots with multiple end-effectors, both with and without the inclusion of joint limits.
We have also demonstrated the feasibility of this approach to solve IK for commercial manipulators, achieving success rates competitive with both comparable~\cite{erleben_solving_2019} and conventional~\cite{lynch2017modern} angle-based methods.
Notably, we observe that our algorithm performs significantly better than a conventional method, both in terms of success rate and runtime, when the IK problem requires finding configurations that are not in collision with obstacles (modelled as spheres).
Overall, these experimental results indicate that a distance-geometric approach is particularly advantageous when a large number of workspace constraints are present.
We believe our algorithm provides a valuable benchmark for IK solvers, as well as a good starting point for future research into distance-geometric formulations of this problem.

\subsection{Limitations}
Avoiding reflections in the solution set of certain problems remains a core challenge when using a purely distance-based IK approach.
As noted in \Cref{prop:equivalence}, this spells out an important limitation of our formulation: its inability to handle revolute manipulators with non-coplanar consecutive axes of rotation.
This also restricts the capacity of our formulation to represent joint angle limits to those that are symmetrical about $\Vector{\Theta}_0 = \Vector{0}$ (i.e., of the form $[{-\theta_{\mathrm{lim}}}, \ \theta_{\mathrm{lim}}]$).
At the cost of some of the theoretical foundations laid out in \Cref{sec:background} and \Cref{sec:eqdgp}, it is possible to address these ambiguities by extending our formulation to include cross products (allowing ``handedness'' to be expressed).
For example, solutions that include reflections for non-coplanar joints $u$ and $v$ in \cref{fig:linkage_model} could be removed by taking ${\Vector{c} = (\Vector{p}_{\tilde{u}} - \Vector{p}_{u})\times(\Vector{p}_{\tilde{v}} - \Vector{p}_{v})}$ and constraining the sign of the dot product ${\Vector{c}\cdot(\Vector{p}_{v} - \Vector{p}_{u})}$, thereby only allowing one of the two possible relative orientations satisfying distance constraints.
Similarly, a non-symmetric joint limit on $v$ can be obtained by taking $\Vector{a} = (\Vector{p}_{v} - \Vector{p}_{u})\times(\Vector{p}_{\tilde{v}} - \Vector{p}_{v})$, $\Vector{b} = (\Vector{p}_{w} - \Vector{p}_{u})\times(\Vector{p}_{\tilde{v}} - \Vector{p}_{v})$ and constraining the sign of $\Vector{a} \cdot \Vector{b}$.
Since adding these constraints would require a lengthy and detailed deviation from the purely distance-geometric view of IK developed herein, we leave their characterization as future work. 


%% file: sections/future_work.tex
\subsection{Future Work}
We have identified several exciting research directions for the distance-geometric IK formulation presented in this paper.
Our algorithm utilizes a Riemannian optimization-based solution because it is fast, can easily be extended (e.g., to include obstacles or other cost terms), and avoids problems associated with redundant degrees of freedom.
However, we believe a thorough exploration of the vast body of literature on distance geometry may yield other effective approaches or insights into our particular problem structure.
Our hope is that further study of approaches such as semidefinite programming relaxations for EDM completion will help to yield deep theoretical insight into the geometric structure of IK and its relationship with other problems in distance geometry.
This will permit us to compare the runtime of our algorithm against a greater variety of IK solvers, including complex algorithms like \texttt{TRAC-IK} that utilize multiple methods in parallel~\cite{beeson2015trac}.
Additionally, we are eager to develop an optimized version of our algorithm in a fast compiled language such as \cpp.
We believe that the distance-geometric IK formulation described herein provides a strong mathematical basis for future research in kinematics, motion planning, and control.

%% file: paper.bbl
\begin{thebibliography}{10}
\providecommand{\url}[1]{#1}
\csname url@samestyle\endcsname
\providecommand{\newblock}{\relax}
\providecommand{\bibinfo}[2]{#2}
\providecommand{\BIBentrySTDinterwordspacing}{\spaceskip=0pt\relax}
\providecommand{\BIBentryALTinterwordstretchfactor}{4}
\providecommand{\BIBentryALTinterwordspacing}{\spaceskip=\fontdimen2\font plus
\BIBentryALTinterwordstretchfactor\fontdimen3\font minus
  \fontdimen4\font\relax}
\providecommand{\BIBforeignlanguage}[2]{{%
\expandafter\ifx\csname l@#1\endcsname\relax
\typeout{** WARNING: IEEEtran.bst: No hyphenation pattern has been}%
\typeout{** loaded for the language `#1'. Using the pattern for}%
\typeout{** the default language instead.}%
\else
\language=\csname l@#1\endcsname
\fi
#2}}
\providecommand{\BIBdecl}{\relax}
\BIBdecl

\bibitem{siciliano2010robotics}
B.~Siciliano, L.~Sciavicco, L.~Villani, and G.~Oriolo, \emph{Robotics:
  Modelling, Planning and Control}.\hskip 1em plus 0.5em minus 0.4em\relax
  Springer Science \& Business Media, 2010.

\bibitem{dokmanic_euclidean_2015}
I.~Dokmanic, R.~Parhizkar, J.~Ranieri, and M.~Vetterli,
  ``\BIBforeignlanguage{en}{Euclidean {{Distance Matrices}}: {{Essential
  Theory}}, {{Algorithms}} and {{Applications}}},''
  \emph{\BIBforeignlanguage{en}{IEEE Signal Process. Mag.}}, vol.~32, no.~6,
  pp. 12--30, Nov. 2015.

\bibitem{dejalonTwentyfiveYearsNatural2007}
J.~G. {de Jal{\'o}n}, ``\BIBforeignlanguage{en}{Twenty-five years of natural
  coordinates},'' \emph{\BIBforeignlanguage{en}{Multibody Sys. Dyn.}}, vol.~18,
  no.~1, pp. 15--33, Aug. 2007.

\bibitem{blanchini_convex_2017}
F.~Blanchini, G.~Fenu, G.~Giordano, and F.~A. Pellegrino,
  ``\BIBforeignlanguage{en}{A convex programming approach to the inverse
  kinematics problem for manipulators under constraints},''
  \emph{\BIBforeignlanguage{en}{Eur. J. Control}}, vol.~33, pp. 11--23, Jan.
  2017.

\bibitem{le2019kinematics}
T.~Le~Naour, N.~Courty, and S.~Gibet, ``Kinematics in the metric space,''
  \emph{Comput. Graph.}, vol.~84, pp. 13--23, 2019.

\bibitem{porta_inverse_2005}
J.~M. Porta, L.~Ros, and F.~Thomas, ``Inverse kinematics by distance matrix
  completion,'' in \emph{Proc. 12th Int. Workshop Computational
  Kinematics}.\hskip 1em plus 0.5em minus 0.4em\relax Elsevier, 2005.

\bibitem{han_inverse_2006}
L.~Han and L.~Rudolph, ``Inverse kinematics for a serial chain with joints
  under distance constraints.'' in \emph{Robotics: Science and Systems (RSS)},
  2006.

\bibitem{libertiEuclideanDistanceGeometry2014}
L.~Liberti, C.~Lavor, N.~Maculan, and A.~Mucherino,
  ``\BIBforeignlanguage{en}{Euclidean {{Distance Geometry}} and
  {{Applications}}},'' \emph{\BIBforeignlanguage{en}{SIAM Rev.}}, vol.~56,
  no.~1, pp. 3--69, Jan. 2014.

\bibitem{nguyenLowRankMatrixCompletion2019}
L.~T. Nguyen, J.~Kim, and B.~Shim, ``Low-rank matrix completion: A contemporary
  survey,'' \emph{IEEE Access}, vol.~7, pp. 94\,215--94\,237, 2019.

\bibitem{Mishra_2011}
B.~Mishra, G.~Meyer, and R.~Sepulchre, ``Low-rank optimization for distance
  matrix completion,'' in \emph{Proc. IEEE Conf. Decision and Control and Eur.
  Control Conf.}, Dec 2011, pp. 4455--4460.

\bibitem{Journ_e_2010}
M.~Journée, F.~Bach, P.-A. Absil, and R.~Sepulchre, ``Low-rank optimization on
  the cone of positive semidefinite matrices,'' \emph{SIAM J. Optim.}, vol.~20,
  no.~5, p. 2327–2351, Jan. 2010.

\bibitem{havelDistanceGeometryTheory2002}
T.~F. Havel, ``\BIBforeignlanguage{en}{Distance {{Geometry}}: {{Theory}},
  {{Algorithms}}, and {{Chemical Applications}}},'' in
  \emph{\BIBforeignlanguage{en}{Encyclopedia of {{Computational
  Chemistry}}}}.\hskip 1em plus 0.5em minus 0.4em\relax {John Wiley \& Sons,
  Ltd}, Apr. 2002, pp. 723--742.

\bibitem{angeles2013computational}
J.~Angeles, G.~Hommel, and P.~Kov{\'a}cs, \emph{Computational
  Kinematics}.\hskip 1em plus 0.5em minus 0.4em\relax Springer Science \&
  Business Media, 2013, vol.~28.

\bibitem{aristidouInverseKinematicsTechniques2018}
A.~Aristidou, J.~Lasenby, Y.~Chrysanthou, and A.~Shamir,
  ``\BIBforeignlanguage{en}{Inverse {{Kinematics Techniques}} in {{Computer
  Graphics}}: {{A Survey}}},'' \emph{\BIBforeignlanguage{en}{Comput. Graph.
  Forum.}}, vol.~37, no.~6, pp. 35--58, Sep. 2018.

\bibitem{lee1988new}
H.-Y. Lee and C.-G. Liang, ``A new vector theory for the analysis of spatial
  mechanisms,'' \emph{Mech. Mach. Theory}, vol.~23, no.~3, pp. 209--217, 1988.

\bibitem{manocha1994efficient}
D.~Manocha and J.~F. Canny, ``Efficient inverse kinematics for general 6r
  manipulators,'' \emph{IEEE Trans. Robot.}

\bibitem{husty2007new}
M.~L. Husty, M.~Pfurner, and H.-P. Schr{\"o}cker, ``A new and efficient
  algorithm for the inverse kinematics of a general serial 6r manipulator,''
  \emph{Mech. Mach. Theory}, vol.~42, no.~1, pp. 66--81, 2007.

\bibitem{diankov2010automated}
R.~Diankov, ``Automated construction of robotic manipulation programs,'' Ph.D.
  dissertation, Carnegie Mellon University, Pittsburgh, PA, September 2010.

\bibitem{kenwright2012inverse}
B.~Kenwright, ``Inverse kinematics--cyclic coordinate descent ({CCD}),''
  \emph{J. Graphics Tools}, vol.~16, no.~4, pp. 177--217, 2012.

\bibitem{muller2007triangualation}
R.~Muller-Cajar and R.~Mukundan, ``Triangulation - a new algorithm for inverse
  kinematics,'' in \emph{Int. Conf. Image and Vision Computing New Zealand},
  Dec 2007, pp. 181--186.

\bibitem{han_unified_2007}
``A unified geometric approach for inverse kinematics.''

\bibitem{Aristidou_2011}
A.~Aristidou and J.~Lasenby, ``{FABRIK}: A fast, iterative solver for the
  inverse kinematics problem,'' \emph{Graph. Models}, vol.~73, no.~5, p.
  243–260, Sep. 2011.

\bibitem{aristidou_extending_2016}
A.~Aristidou, Y.~Chrysanthou, and J.~Lasenby,
  ``\BIBforeignlanguage{en}{Extending {{FABRIK}} with model constraints},''
  \emph{\BIBforeignlanguage{en}{Comput. Animat. Virtual Worlds}}, vol.~27,
  no.~1, pp. 35--57, Jan. 2016.

\bibitem{zhu_algorithm_1997}
C.~Zhu, R.~H. Byrd, P.~Lu, and J.~Nocedal, ``\BIBforeignlanguage{en}{Algorithm
  778: {{L}}-{{BFGS}}-{{B}}: {{Fortran}} subroutines for large-scale
  bound-constrained optimization},'' \emph{\BIBforeignlanguage{en}{ACM Trans.
  Math. Softw.}}, vol.~23, no.~4, pp. 550--560, Dec. 1997.

\bibitem{schulman2014motion}
J.~Schulman, Y.~Duan, J.~Ho, A.~Lee, I.~Awwal, H.~Bradlow, J.~Pan, S.~Patil,
  K.~Goldberg, and P.~Abbeel, ``Motion planning with sequential convex
  optimization and convex collision checking,'' \emph{Int. J. Rob. Res.}, 2014.

\bibitem{boyd2004convex}
S.~Boyd and L.~Vandenberghe, \emph{Convex Optimization}.\hskip 1em plus 0.5em
  minus 0.4em\relax Cambridge University Press, 2004.

\bibitem{beeson2015trac}
P.~Beeson and B.~Ames, ``{TRAC-IK}: An open-source library for improved solving
  of generic inverse kinematics,'' in \emph{Proc. IEEE Int. Conf. Humanoid
  Robots (Humanoids)}, 2015, pp. 928--935.

\bibitem{sciavicco1986coordinate}
L.~Sciavicco and B.~Siciliano, ``Coordinate transformation: A solution
  algorithm for one class of robots,'' \emph{{IEEE} Trans. Syst., Man,
  Cybern.}, vol.~16, no.~4, pp. 550--559, 1986.

\bibitem{buss2005selectively}
S.~R. Buss and J.-S. Kim, ``Selectively damped least squares for inverse
  kinematics,'' \emph{J. Graphics Tools}, vol.~10, pp. 37--49, 2005.

\bibitem{nakamura1987task}
Y.~Nakamura, H.~Hanafusa, and T.~Yoshikawa, ``Task-priority based redundancy
  control of robot manipulators,'' \emph{Int. J. Rob. Res.}, vol.~6, no.~2, pp.
  3--15, 1987.

\bibitem{spong2020robot}
M.~W. Spong, S.~Hutchinson, and M.~Vidyasagar, \emph{Robot Modeling and
  Control}.\hskip 1em plus 0.5em minus 0.4em\relax John Wiley \& Sons, 2005.

\bibitem{erleben_solving_2019}
K.~Erleben and S.~Andrews, ``\BIBforeignlanguage{en}{Solving inverse kinematics
  using exact {{Hessian}} matrices},'' \emph{\BIBforeignlanguage{en}{Comput.
  Graph.}}, vol.~78, pp. 1--11, Feb. 2019.

\bibitem{dai_global_2019}
H.~Dai, G.~Izatt, and R.~Tedrake, ``\BIBforeignlanguage{en}{Global inverse
  kinematics via mixed-integer convex optimization},''
  \emph{\BIBforeignlanguage{en}{Int. J. Rob. Res.}}, vol.~38, no. 12-13, pp.
  1420--1441, Oct. 2019.

\bibitem{yenamandra_convex_2019}
T.~Yenamandra, F.~Bernard, J.~Wang, F.~Mueller, and C.~Theobalt,
  ``\BIBforeignlanguage{en}{Convex {{Optimisation}} for {{Inverse
  Kinematics}}},'' \emph{\BIBforeignlanguage{en}{Proc. Int. Conf. 3D Vision
  (3DV)}}, pp. 318--327, Sep. 2019.

\bibitem{blanchini_inverse_2015}
F.~Blanchini, G.~Fenu, G.~Giordano, and F.~A. Pellegrino, ``Inverse kinematics
  by means of convex programming: {{Some}} developments,'' in \emph{IEEE Int.
  Conf. Autom. Sci. Eng. (CASE)}, Aug. 2015, pp. 515--520.

\bibitem{dingSensorNetworkLocalization2010}
Y.~Ding, N.~Krislock, J.~Qian, and H.~Wolkowicz,
  ``\BIBforeignlanguage{en}{Sensor network localization, {Euclidean} distance
  matrix completions, and graph realization},''
  \emph{\BIBforeignlanguage{en}{Optim. Eng.}}, vol.~11, no.~1, pp. 45--66, Feb.
  2010.

\bibitem{cox2008multidimensional}
M.~A. Cox and T.~F. Cox, ``Multidimensional scaling,'' in \emph{Handbook of
  Data Visualization}.\hskip 1em plus 0.5em minus 0.4em\relax Springer, 2008,
  pp. 315--347.

\bibitem{liberti2008branch}
L.~Liberti, C.~Lavor, and N.~Maculan, ``A branch-and-prune algorithm for the
  molecular distance geometry problem,'' \emph{Int. Trans. Oper. Res.}

\bibitem{biswas2006semidefinite}
P.~Biswas, T.-C. Lian, T.-C. Wang, and Y.~Ye, ``Semidefinite programming based
  algorithms for sensor network localization,'' \emph{ACM Trans. Sensor
  Networks (TOSN)}, vol.~2, no.~2, pp. 188--220, 2006.

\bibitem{leung2010sdp}
N.-H.~Z. Leung and K.-C. Toh, ``An {SDP}-based divide-and-conquer algorithm for
  large-scale noisy anchor-free graph realization,'' \emph{SIAM J. Scientific
  Comput.}, vol.~31, no.~6, pp. 4351--4372, 2010.

\bibitem{vandereyckenLowrankMatrixCompletion2012}
B.~Vandereycken, ``\BIBforeignlanguage{en}{Low-rank matrix completion by
  {{Riemannian}} optimization\textemdash extended version},''
  \emph{\BIBforeignlanguage{en}{SIAM J. Optim.}}, vol.~23, no.~2, pp.
  1214--1236, Sep. 2012.

\bibitem{nguyenLocalizationIoTNetworks2019}
L.~T. Nguyen, J.~Kim, S.~Kim, and B.~Shim,
  ``\BIBforeignlanguage{en}{Localization of {{IoT Networks}} via {{Low}}-{{Rank
  Matrix Completion}}},'' \emph{\BIBforeignlanguage{en}{IEEE Trans. Commun.}},
  vol.~67, no.~8, pp. 5833--5847, Aug. 2019.

\bibitem{portaDistanceGeometryActive2018}
J.~M. Porta, N.~Rojas, and F.~Thomas, ``\BIBforeignlanguage{en}{Distance
  {{Geometry}} in {{Active Structures}}},'' in
  \emph{\BIBforeignlanguage{en}{Mechatronics for {{Cultural Heritage}} and
  {{Civil Engineering}}}}, 2018, vol.~92, pp. 115--136.

\bibitem{porta_branch-and-prune_2005}
J.~Porta, L.~Ros, F.~Thomas, and C.~Torras, ``A branch-and-prune solver for
  distance constraints,'' \emph{IEEE Trans. Robot.}, vol.~21, pp. 176--187,
  Apr. 2005.

\bibitem{maricInverseKinematicsSerial2020}
F.~Mari{\'c}, M.~Giamou, S.~Khoubyarian, I.~Petrovi{\'c}, and J.~Kelly,
  ``Inverse kinematics for serial kinematic chains via sum of squares
  optimization,'' in \emph{Proc. IEEE Int. Conf. Robot. Autom. (ICRA)}, Aug.
  2020, pp. 7101--7107.

\bibitem{parrilo_semidefinite_2003}
P.~A. Parrilo, ``\BIBforeignlanguage{en}{Semidefinite programming relaxations
  for semialgebraic problems},'' \emph{\BIBforeignlanguage{en}{Math.
  Program.}}, vol.~96, no.~2, pp. 293--320, May 2003.

\bibitem{lasserre_global_2001}
J.~B. Lasserre, ``\BIBforeignlanguage{en}{Global {{Optimization}} with
  {{Polynomials}} and the {{Problem}} of {{Moments}}},''
  \emph{\BIBforeignlanguage{en}{SIAM J. Optim.}}, vol.~11, no.~3, pp. 796--817,
  Jan. 2001.

\bibitem{Hendrickson_1992}
B.~Hendrickson, ``Conditions for unique graph realizations,'' \emph{SIAM J.
  Comput.}, vol.~21, no.~1, p. 65–84, Feb. 1992.

\bibitem{Sippl_1986}
M.~J. Sippl and H.~A. Scheraga, ``Cayley-menger coordinates.'' \emph{Proc.
  Natl. Acad. Sci.}, vol.~83, no.~8, p. 2283–2287, Apr. 1986.

\bibitem{alfakih1999solving}
A.~Y. Alfakih, A.~Khandani, and H.~Wolkowicz, ``Solving euclidean distance
  matrix completion problems via semidefinite programming,'' \emph{Comput.
  Optim. Appl.}, vol.~12, no. 1-3, pp. 13--30, 1999.

\bibitem{so_theory_2007}
A.~M.-C. So and Y.~Ye, ``\BIBforeignlanguage{en}{Theory of semidefinite
  programming for {{Sensor Network Localization}}},''
  \emph{\BIBforeignlanguage{en}{Math. Program.}}, vol. 109, no. 2-3, pp.
  367--384, Jan. 2007.

\bibitem{Burer_2004}
S.~Burer and R.~D. Monteiro, ``Local minima and convergence in low-rank
  semidefinite programming,'' \emph{Mathematical Programming}, vol. 103, no.~3,
  p. 427–444, Dec. 2004.

\bibitem{fangEuclideanDistanceMatrix2012}
H.~Fang and D.~P. O'Leary, ``\BIBforeignlanguage{en}{Euclidean distance matrix
  completion problems},'' \emph{\BIBforeignlanguage{en}{Optim. Methods
  Softw.}}, vol.~27, no. 4-5, pp. 695--717, Oct. 2012.

\bibitem{chuLEASTSQUARESEUCLIDEAN}
D.~I. Chu, H.~C. Brown, and M.~T. Chu, ``\BIBforeignlanguage{en}{{On Least
  Squares Euclidean Distance Matrix Approximation and Completion}},''
  Department of Mathematics, North Carolina State University, Tech. Rep., 2003.

\bibitem{Pearlmutter_1994}
B.~A. Pearlmutter, ``Fast exact multiplication by the {Hessian},'' \emph{Neural
  Comput.}, vol.~6, no.~1, p. 147–160, Jan. 1994.

\bibitem{absil2009optimization}
P.-A. Absil, R.~Mahony, and R.~Sepulchre, \emph{Optimization Algorithms on
  Matrix Manifolds}.\hskip 1em plus 0.5em minus 0.4em\relax Princeton
  University Press, 2009.

\bibitem{boumal2015lowrank}
N.~Boumal and P.-A. Absil, ``Low-rank matrix completion via preconditioned
  optimization on the {G}rassmann manifold,'' \emph{Linear Algebra Appl.}, vol.
  475, pp. 200--239, 2015.

\bibitem{weiGuaranteesRiemannianOptimization2016}
K.~Wei, J.-F. Cai, T.~F. Chan, and S.~Leung,
  ``\BIBforeignlanguage{en}{Guarantees of {{Riemannian Optimization}} for {{Low
  Rank Matrix Recovery}}},'' \emph{\BIBforeignlanguage{en}{{SIAM} J. Matrix
  Anal. Appl.}}, vol.~37, no.~3, Apr. 2016.

\bibitem{Le_Naour_2019}
T.~Le~Naour, N.~Courty, and S.~Gibet, ``Kinematics in the metric space,''
  \emph{Comput. Graph.}, vol.~84, p. 13–23, Nov. 2019.

\bibitem{hartenberg1955kinematic}
R.~S. Hartenberg and J.~Denavit, ``A kinematic notation for lower pair
  mechanisms based on matrices,'' \emph{J. Appl. Mech.}, vol.~77, no.~2, pp.
  215--221, 1955.

\bibitem{lynch2017modern}
K.~M. Lynch and F.~C. Park, \emph{Modern Robotics}.\hskip 1em plus 0.5em minus
  0.4em\relax Cambridge University Press, 2017.

\bibitem{ananthanarayanan_real-time_2015}
H.~Ananthanarayanan and R.~Ord{\'o}{\~n}ez, ``\BIBforeignlanguage{en}{Real-time
  inverse kinematics of (2n+1) {{DOF}} hyper-redundant manipulator arm via a
  combined numerical and analytical approach},''
  \emph{\BIBforeignlanguage{en}{Mech. Mach. Theory}}, vol.~91, pp. 209--226,
  Sep. 2015.

\bibitem{genrtr}
P.-A. Absil, C.~G. Baker, and K.~A. Gallivan, ``Trust-region methods on
  {Riemannian} manifolds,'' \emph{Found. Comput. Math.}, vol.~7, no.~3, pp.
  303--330, 2007.

\bibitem{tonycai_one-sided_2005}
T.~Tony~Cai, ``\BIBforeignlanguage{en}{One-sided confidence intervals in
  discrete distributions},'' \emph{\BIBforeignlanguage{en}{J. Stat. Plan.
  Inference}}, vol. 131, no.~1, pp. 63--88, Apr. 2005.

\bibitem{townsend2016pymanopt}
J.~Townsend, N.~Koep, and S.~Weichwald, ``Pymanopt: A {Python} toolbox for
  optimization on manifolds using automatic differentiation,'' \emph{J. Mach.
  Learn. Res.}, vol.~17, no.~1, pp. 4755--4759, 2016.

\bibitem{virtanen2020scipy}
P.~Virtanen, R.~Gommers, T.~E. Oliphant, M.~Haberland, T.~Reddy, D.~Cournapeau,
  E.~Burovski, P.~Peterson, W.~Weckesser, J.~Bright \emph{et~al.}, ``Scipy 1.0:
  fundamental algorithms for scientific computing in {Python},'' \emph{Nat.
  Methods}, vol.~17, no.~3, pp. 261--272, 2020.

\end{thebibliography}
